\documentclass[10pt]{article}
\usepackage{amsmath}
\usepackage{bm}
\usepackage{algorithm}
\usepackage{algpseudocode}
\usepackage{amsmath}
\usepackage{multirow}
\usepackage{indentfirst}
\usepackage{subfigure}
\usepackage{epsfig}
\usepackage{epstopdf}
\usepackage{xspace}
\usepackage{amsmath,amssymb}
\usepackage{setspace}
\usepackage{theorem}

\usepackage{amsfonts}
\usepackage{mathrsfs}
\usepackage{enumerate}
\usepackage{indentfirst}
\usepackage{verbatim}
\usepackage{graphicx}
\usepackage{geometry}
\usepackage{subeqnarray}
\usepackage{placeins}
\usepackage{url}

\usepackage{algorithm}
\usepackage{algpseudocode}
\usepackage{multirow}
\usepackage{natbib}
\usepackage{tcolorbox}
\usepackage{hyperref}

\newtheorem{theorem}{Theorem}

\newtheorem{mydef}{Definition}
\newtheorem{proof}{Proof}

\newcommand{\paratitle}[1]{\vspace{1.2ex}\noindent\textbf{\normalsize \color{blue} #1}}

\newcommand{\ie}{\emph{i.e., }}
\newcommand{\eg}{\emph{e.g., }}

\usepackage{booktabs}

\usepackage{colortbl}

\usepackage{geometry}

\usepackage{subeqnarray}
\newcommand{\mywritedate}{November 30, 2025}
\begin{document}

\title{\bf \large
A Tutorial on Regression Analysis: From Linear Models to \\ Deep Learning\\[2pt]
\it \normalsize --- Lecture Notes on Artificial Intelligence at Beihang University}

\author{ Jingyuan Wang, Jiahao Ji \\ \small BIGSCITY LAB, Beihang University \\ \small \url{https://www.bigscity.com}, jywang@buaa.edu.cn}
\date{\small \mywritedate}
\maketitle

\abstract{This article serves as the regression analysis lecture notes in the Intelligent Computing course cluster (including the courses of Artificial Intelligence, Data Mining, Machine Learning, and Pattern Recognition) at the School of Computer Science and Engineering, Beihang University. It aims to provide students -- who are assumed to possess only basic university-level mathematics (\ie  with prerequisite courses in calculus, linear algebra, and probability theory) -- with a comprehensive and self-contained understanding of regression analysis without requiring any additional references. The lecture notes systematically introduce the fundamental concepts, modeling components, and theoretical foundations of regression analysis, covering linear regression, logistic regression, multinomial logistic regression, polynomial regression, basis-function models, kernel-based methods, and neural-network-based nonlinear regression. Core methodological topics include loss-function design, parameter-estimation principles, ordinary least squares, gradient-based optimization algorithms and their variants, as well as regularization techniques such as Ridge and LASSO regression. Through detailed mathematical derivations, illustrative examples, and intuitive visual explanations, the materials help students understand not only how regression models are constructed and optimized, but also how they reveal the underlying relationships between features and response variables. By bridging classical statistical modeling and modern machine-learning practice, these lecture notes aim to equip students with a solid conceptual and technical foundation for further study in advanced artificial intelligence models.
}

\newpage

\tableofcontents

\newpage

\section{What is Regression Analysis}

\begin{tcolorbox}
\begin{quote}
``The Tao of heaven is like the bending of a bow. The high is lowered, and the low is raised. If the string is too long, it is shortened; If there is not enough, it is made longer.''

--- \textit{Laozi, Tao Te Ching, Chapter 77}
\end{quote}
\end{tcolorbox}

\subsection{Definition of Regression Analysis}

According to Wikipedia, the definition of Regression Analysis is: {\em ``In statistical modeling, regression analysis is a statistical method for estimating the relationship between a dependent variable (often called the outcome or response variable, or a label in machine learning parlance) and one or more independent variables (often called regressors, predictors, covariates, explanatory variables, or features).''}

\begin{table}[h]\small
\centering
\caption{Terms in regression analysis.}\label{tab:lr:terminology}
\begin{tabular}{ll}\toprule
  $X$ & $Y$ \\
  \midrule
      Feature & Label \\
      Independent Variable & Dependent Variable \\
      Explanatory Variable & Explained Variable \\
      Controlled Variable & Response Variable \\
      Predictor Variable & Predicted Variable \\
      Regressor & Regressand \\
      Income & Outcome \\
  \bottomrule
\end{tabular}
\end{table}

According to this definition, a variable $X$ that is relatively well understood and a variable $Y$ about which our knowledge is limited. Our objective is to investigate how $X$ influences $Y$. Regression analysis provides a statistical framework and methodological tools for examining the relationship between these two variables. Across different disciplines, the terms used to describe $X$ and $Y$ vary, as summarized in Table~\ref{tab:lr:terminology}.

In machine learning and data mining, the variables $X$ and $Y$ are typically referred to as \emph{features} and \emph{labels}, respectively. In this book, we adopt the same terminology to denote $X$ and $Y$. Then, the formal definition of regression analysis is given as follows.

\begin{mydef}[Regression Analysis]
Given a set of samples
$$\mathcal{D}=\{(\bm{x}_1, y_1), (\bm{x}_2, y_2), \ldots, (\bm{x}_M, y_M)\},$$
where $\bm{x}_m$ denotes the feature and $y_m$ denotes the corresponding label of the $m$-th sample, the objective of regression analysis is to learn a function that estimates or predicts $y_m$ from $\bm{x}_m$.
Formally,
\begin{equation}\label{}
      \hat{y}_m = f(\bm{x}_m; \bm{\theta}),
\end{equation}
where $f(\cdot; \bm{\theta})$ denotes the function that characterizes the relationship between the features and labels, referred to as the \textbf{regression function}, and $\bm{\theta}$ denotes its parameters.
The quantity $\hat{y}_m$ represents the predicted value of $y_m$.
In a regression function, the parameter $\bm{\theta}$ is unknown and must be estimated from the data.
\end{mydef}

Using a regression model for data analysis typically involves two stages.
The first stage is \textbf{Model Construction}, in which a regression function $f(\cdot; \bm{\theta})$ is selected based on the dataset $D$, and the parameter vector $\theta$ is estimated.
The second stage is \textbf{Model Utilization}, in which the regression function $f(\cdot; \bm{\theta})$ and the estimated parameters $\theta$ are used either to analyze the relationship between the feature vector $\bm{x}$ and the label $y$ within the dataset $D$, or to predict the value of a label $y'$ given an input feature $\bm{x}'$ corresponding to a sample not contained in $D$.

In some machine learning textbooks, the term ``regression'' is specifically used to describe predictive models and tasks in which the label $y$ is continuous, whereas predictive models and tasks with a discrete label are referred to as ``classification.'' As we will see in later chapters, regression analysis models are capable of handling both types of tasks. For example, logistic regression is specifically designed for classification tasks, despite being formulated within the general framework of regression analysis.  To reconcile these different uses of the term \emph{regression}, we may adopt the following perspective. From the viewpoint of predictive tasks, problems involving continuous labels and those involving discrete labels can be referred to as \emph{regression tasks} and \emph{classification tasks}, respectively. From the viewpoint of modeling, however, any analytical approach that employs a parametric model to study how a feature $x$ influences a label $y$ falls under the category of a \emph{regression analysis model}.

\subsection{The Elements of Regression Model Construction}~\label{sec:elements_regression}

The construction of a regression model consists of three essential elements:
the \textbf{regression function}, the \textbf{loss function}, and \textbf{parameter estimation} (we refer to these as the ``three steps for putting an elephant into a refrigerator,'' \ie open the refrigerator, put the elephant inside, and close the refrigerator).

\paratitle{Regression Functions.} Given the dataset $\mathcal{D}$, a regression analysis model needs to select a function $f(\cdot;\bm{\theta})$ to characterize the relationship between the feature vector $\bm{x}_m$ and the label $y_m$. This function is referred to as the \emph{regression function}. Different regression models employ different forms of regression functions. Table~\ref{tab:lr:reg_func} lists several commonly used regression functions adopted by various regression analysis models.

\begin{table}[t]\footnotesize
    \centering
    \caption{Regression functions used in several commonly adopted regression models.}\label{tab:lr:reg_func}
    \renewcommand{\arraystretch}{1.5}
    \begin{tabular}{rcc}
    \toprule
        Model & Regression Function $f(\cdot; \bm{\theta})$ & Parameters $\bm{\theta}$ \\
    \midrule
         Linear Regression &
         $y = \theta_0 + \theta_1 x_1 + \dots + \theta_N x_N$ &
         $(\theta_0, \theta_1, \dots, \theta_N)$ \\

         Logistic Regression &
         $y = \frac{1}{1+\exp( \theta_0 + \theta_1 x_1 + \dots + \theta_N x_N)}$ &
         $(\theta_0, \theta_1, \dots, \theta_N)$ \\

         Multinomial Logistic Regression &
         $y_k = \frac{\exp(\bm{\theta}_k^\top {\bm{x}})}{\sum_{j=1}^{K} \exp(\bm{\theta}_j^\top {\bm{x}})}$ &
         $(\bm{\theta}_1, \bm{\theta}_2, \dots, \bm{\theta}_K)$ \\

         Polynomial Regression &
         $y = \theta_0 + \theta_1 x + \theta_2 x^2 + \dots + \theta_N x^N$ &
         $(\theta_0, \theta_1, \dots, \theta_N)$ \\

         Autoregression (AR) &
         $x_{t+1} = \theta_0 + \theta_1 x_t + \theta_2 x_{t-1} + \dots + \theta_p x_{t-p+1}$ &
         $(\theta_0, \theta_1, \dots, \theta_p)$ \\
    \bottomrule
    \end{tabular}
\end{table}

Deep learning can be viewed as a regression analysis model in which the regression function is implemented by a multi-layer neural network. In this setting, the neural network serves as a highly flexible function approximator capable of modeling complex and nonlinear relationships between features and labels. However, in traditional regression analysis (listed in Table~\ref{tab:lr:reg_func}), the model parameters typically possess clear physical or statistical meaning and can therefore be used to interpret the data-generation mechanism (\ie how the feature $X$ influences or produces the label $Y$). Models with this property are usually referred to as \emph{parametric models}. In contrast, the parameters of a neural network in deep learning do not carry explicit physical meaning and cannot be directly used to explain the underlying data-generation process. For this reason, deep learning models are often categorized as \emph{non-parametric models}. This characteristic of neural networks prevents them from being incorporated into the probabilistic framework introduced in the later chapters of this book, and it also limits their ability to serve as tools for examining
or interpreting the relationship between $X$ and $Y$. Precisely due to these limitations, most taxonomies do not classify models that use neural networks as their regression function as part of the traditional family of regression analysis models.

\paratitle{Loss Functions.} Once the regression function is specified, the prediction for the label $y_m$ can be computed as
$\hat{y}_m = f(\bm{x}_m; \bm{\theta})$. To evaluate the quality of these predictions, a regression
analysis model introduces a \textbf{loss function}, also known as a \textbf{cost function}\footnote{
Some references distinguish between the two terms: the loss function $L(\hat{y}_m, y_m)$ measures
the error for a single sample, while the cost function refers to the average loss over the entire
dataset, as in Eq.~\eqref{eq:lr:loss_cost}. In this book, we do not differentiate between the two.}.
Given the labels $\{y_1, \dots, y_M\}$ and their corresponding predicted values
$\{\hat{y}_1, \dots, \hat{y}_M\}$, the loss function is defined as
\begin{equation}\label{eq:lr:loss_cost}
    \mathcal{L}(\bm{\theta})
    = \frac{1}{M} \sum_{m=1}^{M}
      L\Big(\hat{y}_m(\bm{\theta}, \bm{x}_m),\, y_m\Big),
\end{equation}
where $L(\cdot, \cdot)$ is an error measure between $\hat{y}_m$ and $y_m$.
Here, $\hat{y}_m(\bm{\theta}, \bm{x}_m) = f(\bm{x}_m; \bm{\theta})$ is a function of the parameter
vector $\bm{\theta}$ for a given sample $\bm{x}_m$ in the dataset $D$.

For continuous dependent variables (ratio or interval scales), common error measures include the Mean Square Error (MSE) and the Mean Absolute Error (MAE). For discrete dependent variables (\eg nominal or ordinal types), common measures include Cross Entropy, Accuracy, and Precision.  Table~\ref{tab:lr:error_metric_function} summarizes several widely used error metrics and their applicable data types.

\begin{table}[t]\footnotesize
    \centering
    \caption{Commonly Used Error Metric Functions.}\label{tab:lr:error_metric_function}
    \renewcommand{\arraystretch}{1.5}
    \begin{tabular}{rcc}
    \toprule
        Error Metric & $L(\hat{y}_m(\bm{\theta}, \bm{x}_m), y_m)$ & Applicable Label Type \\
    \midrule
         MSE, Mean Squared Error &
         $\frac{1}{M}\sum_{m=1}^{M} (\hat{y}_m - y_m)^2$ &
         Continuous \\

         RMSE, Root Mean Squared Error &
         $\sqrt{\frac{1}{M}\sum_{m=1}^{M} (\hat{y}_m - y_m)^2}$ &
         Continuous \\

         MAE, Mean Absolute Error &
         $\frac{1}{M}\sum_{m=1}^{M} |\hat{y}_m - y_m|$ &
         Continuous \\

         MAPE, Mean Absolute Percentage Error &
         $\frac{1}{M}\sum_{m=1}^{M} \left|\frac{\hat{y}_m - y_m}{y_m}\right|$ &
         Continuous \\

         Cross Entropy &
         $-\frac{1}{M}\sum_{m=1}^{M}\sum_{k=1}^{K} y_{k,m} \log p_{k,m}$ &
         Discrete \\

         Accuracy &
         $1 - \frac{\mathrm{\#\ error\ items}}{\mathrm{\#\ all\ items}}$ &
         Discrete \\

         Precision &
         $\frac{\mathrm{True\ Positives}}{\mathrm{True\ Positives + False\ Positives}}$ &
         Discrete \\

         Recall &
         $\frac{\mathrm{True\ Positives}}{\mathrm{True\ Positives + False\ Negatives}}$ &
         Discrete \\

         F1 Score &
         $\frac{2 \cdot \mathrm{Precision} \cdot \mathrm{Recall}}{\mathrm{Precision + Recall}}$ &
         Discrete \\

         AUC Value &
         Area under the ROC curve &
         Discrete \\
    \bottomrule
    \end{tabular}
\end{table}

In regression modeling, the purpose of an error metric is to compute the optimal
parameters $\bm{\theta}$. Therefore, the choice of error measure must also take into account the tractability of parameter estimation (\eg differentiability), and hence not all error metrics can be directly used to construct a loss function.

\paratitle{Parameter Estimation.} The purpose of defining a loss function is to determine the parameter $\bm{\theta}$ that best fits the dataset $\mathcal{D}$. According to Eq.~\eqref{eq:lr:loss_cost}, the loss function is a function of the regression model parameters $\bm{\theta}$. Clearly, among all possible choices of $\bm{\theta}$, the optimal parameter vector is the one that minimizes the prediction error (\ie the loss function). Given the dataset $\mathcal{D}$, regression analysis typically computes the optimal parameters $\bm{\theta}^*$ by minimizing the loss function, namely
\begin{equation}\label{eq:minimize_loss_4}
    \bm{\theta}^* = \mathop{\arg\min}\limits_{\bm{\theta}}\, \mathcal{L}(\bm{\theta}).
\end{equation}
There are many approaches to computing the optimal parameters $\bm{\theta}^*$.
In general, these methods can be categorized into {\em closed-form solutions}
and {\em iterative optimization methods}.
\begin{itemize}
  \item {\bf Closed-form solutions} refer to parameter estimates that can be obtained analytically by directly solving an exact mathematical expression. Such solutions typically exist when the loss function is convex, differentiable, and admits an explicit minimizer. A classical example is the ordinary least squares (OLS) solution for linear regression. Closed-form methods are computationally efficient and provide exact solutions, but they are applicable only to relatively simple models or specific loss functions.
  \item {\bf Iterative optimization methods}, in contrast, compute parameter estimates by gradually refining an initial guess through repeated updates. Examples include gradient descent, stochastic gradient descent, Newton's method, and various other numerical optimization algorithms. These methods are highly flexible and can handle complex models, non-linear regression functions, regularization terms, and non-convex objectives. Although iterative methods do not always guarantee a global optimum, they can still yield practical and useful solutions for models whose loss functions cannot be minimized via closed-form expressions -- after all, half a loaf is better than none. For example, in regression models where neural networks serve as the regression function (\ie deep learning), iterative optimization methods are essentially the only feasible option for parameter estimation, since neural networks lack closed-form solutions and typically involve highly non-linear and non-convex objective functions.
\end{itemize}

\subsection{Model Utilization of Regression Models}

Once a regression model has been constructed and its optimal parameters have been estimated, the model can be used in two primary ways: \emph{prediction} and
\emph{interpretation}.
\begin{itemize}
  \item \textbf{Prediction.} A regression model can be used to make predictions for new samples not contained in the original dataset. Given a new feature vector $\bm{x}'$, the regression function produces a corresponding prediction $\hat{y}' = f(\bm{x}'; \bm{\theta}^*)$. This predictive capability is essential in a wide range of machine learning applications, including forecasting, estimation, and decision-making.
  \item \textbf{Interpretation.} For many parametric regression models, the learned parameters carry explicit statistical or physical meaning and can thus be used to interpret how the feature variables $X$ influence the response variable $Y$. This interpretability enables tasks such as identifying important predictors, understanding causal mechanisms, and analyzing the sensitivity of $Y$ with respect to changes in $X$. Beyond directly examining parameter values, interpretation often involves applying various forms of statistical inference or hypothesis testing.
      These procedures allow us to determine whether a particular feature has a statistically significant effect on the response variable, and enable researchers to assess whether the observed relationships between $X$ and $Y$ are meaningful or may have arisen merely from random noise.
\end{itemize}

Classical linear regression, generalized linear models, and many probabilistic
regression models fall into the category of parametric regression models, whose
parameters possess explicit statistical or physical meaning. As a result, these
models support rich interpretative analyses. However, the strong structural assumptions underlying these models---such as linear relationships or specific distributional forms---often oversimplify the complexity of real-world data. Consequently, although highly interpretable, such models frequently underperform in predictive accuracy, especially in the modern era of AI where non-parametric deep learning models dominate by capturing rich nonlinear patterns. In contrast, more complex models such as neural networks primarily emphasize predictive accuracy. Their parameters lack clear statistical interpretation, making them unsuitable for traditional inferential tasks and limiting their role in the interpretative component of regression analysis. In practice, choosing an appropriate regression function often requires a trade-off between modeling capacity and interpretability. Predictive accuracy and interpretability are often mutually exclusive -- you can't have your cake and eat it too.

\begin{figure}[ht]
    \centering
    \subfigure[Model construction and estimation]{\includegraphics[width=0.45\linewidth]{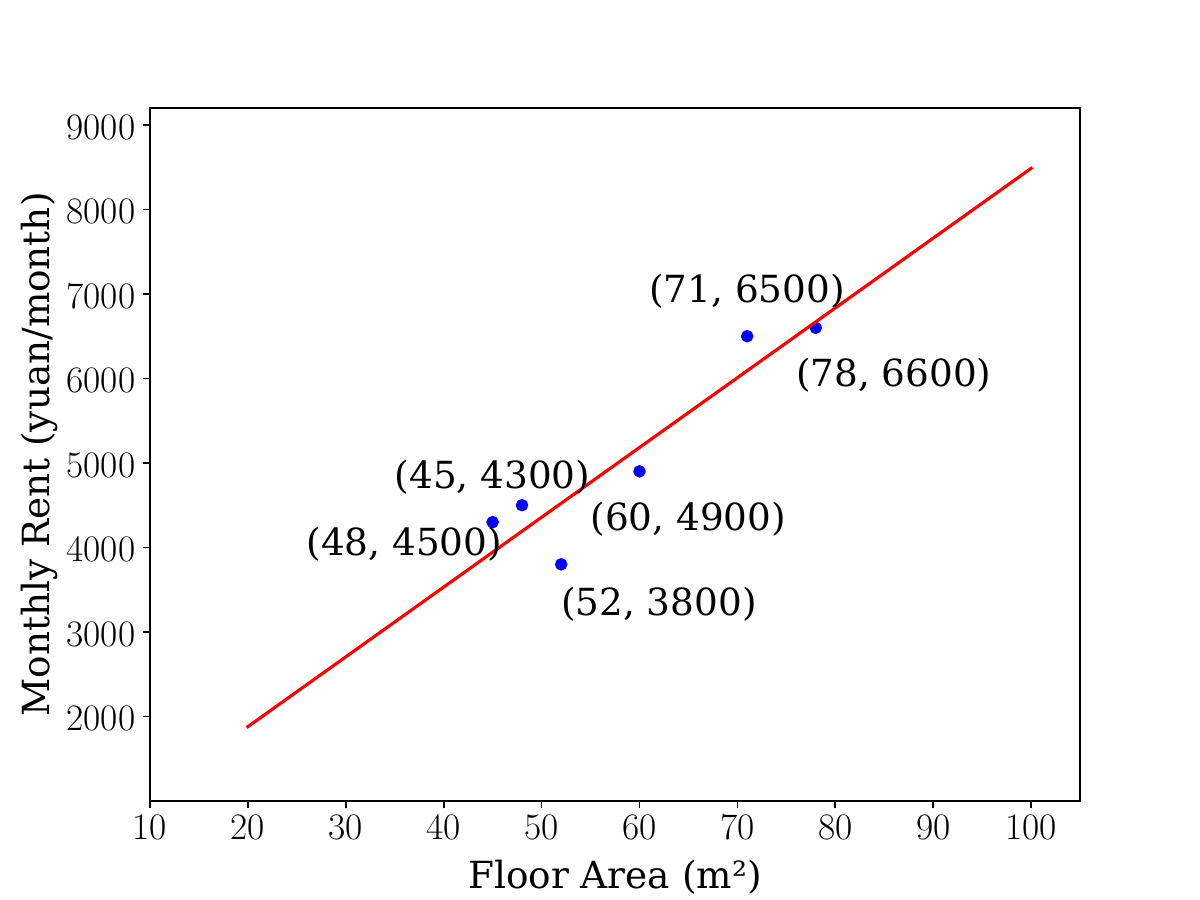}\label{fig:lr:model_train}}~~
    \subfigure[Model utilization]{\includegraphics[width=0.45\linewidth]{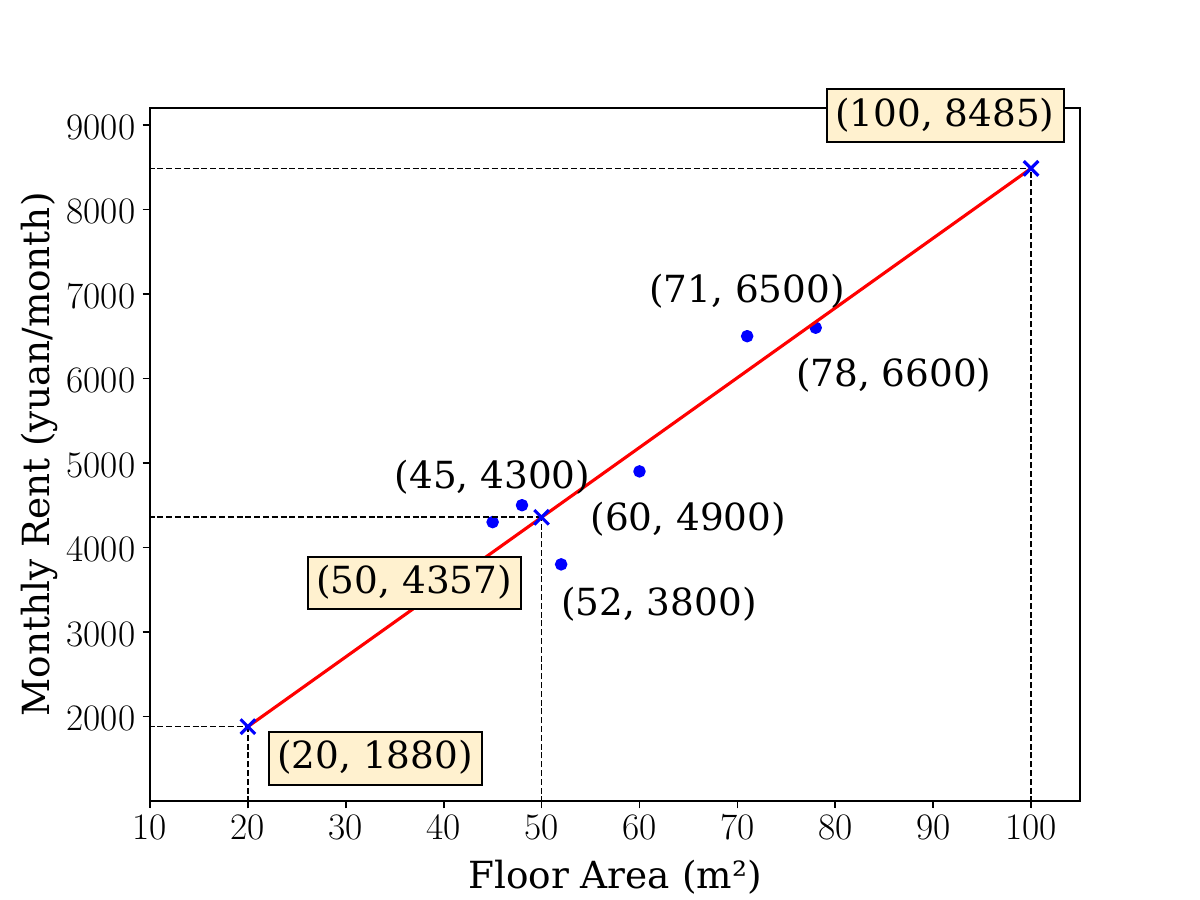}\label{fig:lr:model_use}}
    \caption{\small An example of regression analysis for rental prices near a university.}
    \label{fig:lr:reg_model}
\end{figure}

\paratitle{Example: Regression for Rental Price near Beihang.} Next, we use rental prices near Beihang University as an example to illustrate how a regression model can be used to analyze data. The regression analysis process consists of four steps: data collection, model construction, parameter estimation, and model utilization.
\begin{itemize}
    \item \textbf{Data Collection.}
    In this example, the data consist of rental prices for apartments near Beihang University, collected from an online rental platform\footnote{\url{https://www.58.com}}. Each data point is represented as a pair $(x, y)$, where $x$ denotes the floor area (in square meters) and $y$ denotes the monthly rent (in yuan), \ie,
        \begin{equation*}
            \{(78, 6600), (71, 6500), (60, 4900), (48, 4500), (52, 3800), (45, 4300)\},
        \end{equation*}
    where, for example, the pair $(78, 6600)$ indicates that an apartment of $78$ square meters rents for 6{,}600 yuan per month.

    \item \textbf{Model Construction.}
    We select the simplest linear regression model,
    $$ y = wx + b, $$
    where $y$ denotes the monthly rent, $x$ denotes the floor area, and $w$ and $b$ are the parameters to be estimated.

    \item \textbf{Parameter Estimation.}
    Various methods exist for estimating the parameters of a regression model, which will be discussed in detail in Section~\ref{sec:lr:linreg}. Here, we directly present the estimation result:
    $$ w = 82.6,\quad b = 228.4, $$
    giving the fitted regression model
    $$ y = 82.6x + 228.4, $$
    as shown in Fig.~\ref{fig:lr:model_train}.

    \item \textbf{Model Utilization.}
    From the estimated slope $w$, we can interpret that each additional square meter of floor area increases the monthly rent by approximately 82.6 yuan. The model can also be used for prediction. For example, suppose there is a 20-square-meter apartment near the university and we wish to estimate a reasonable rental price. The model produces
    $$  y = 82.6 \times 20 + 228.4 = 1880.4. $$
    Similarly, the model can predict the rent for apartments of 50 or even 100 square meters, as illustrated in Fig.~\ref{fig:lr:model_use}.
\end{itemize}

\subsection{The Origin of the Term ``Regression''}

Readers may wonder why the term \emph{regression} is used to describe the process of constructing a predictive/relation analysis model. The word ``regression'' in fact originates from a natural phenomenon known as \emph{regression toward the mean}, which refers to the
tendency whereby, for a given random variable, extremely large or small observations are likely to be followed by observations closer to the average.

This phenomenon was first discovered by the British scientist and explorer Francis Galton (1822--1911). In his 1886 paper titled ``Regression towards mediocrity in hereditary stature''~\citep{galton1886regression}, Galton analyzed the relationship between the heights of parents and their adult children. He collected height data for 1,078 families and used the \emph{mid-parent height} (the average of the father's and mother's heights) as the explanatory variable $x$, and the son's height as the response variable $y$. To quantitatively capture the relationship between the two generations, the paper reported the following regression equation:
\[
    y = 0.516\,x + 33.73,
\]
where $x$ is the mid-parent height and $y$ is the adult child's height.
This equation displays a striking property: when $x$ exceeds about $68.25$ inches --- the ``level of mediocrity'' identified by Galton---the predicted child height $y$ falls below $x$; conversely, when $x$ is below $68.25$ inches, the predicted child height exceeds $x$. In other words, children of unusually tall parents tend to be shorter, and children of unusually short parents tend to be taller. This is the empirical pattern Galton famously termed \emph{regression toward mediocrity}, now known as \emph{regression toward the mean}.

\begin{figure}[t] \centering \includegraphics[width=0.5\columnwidth]{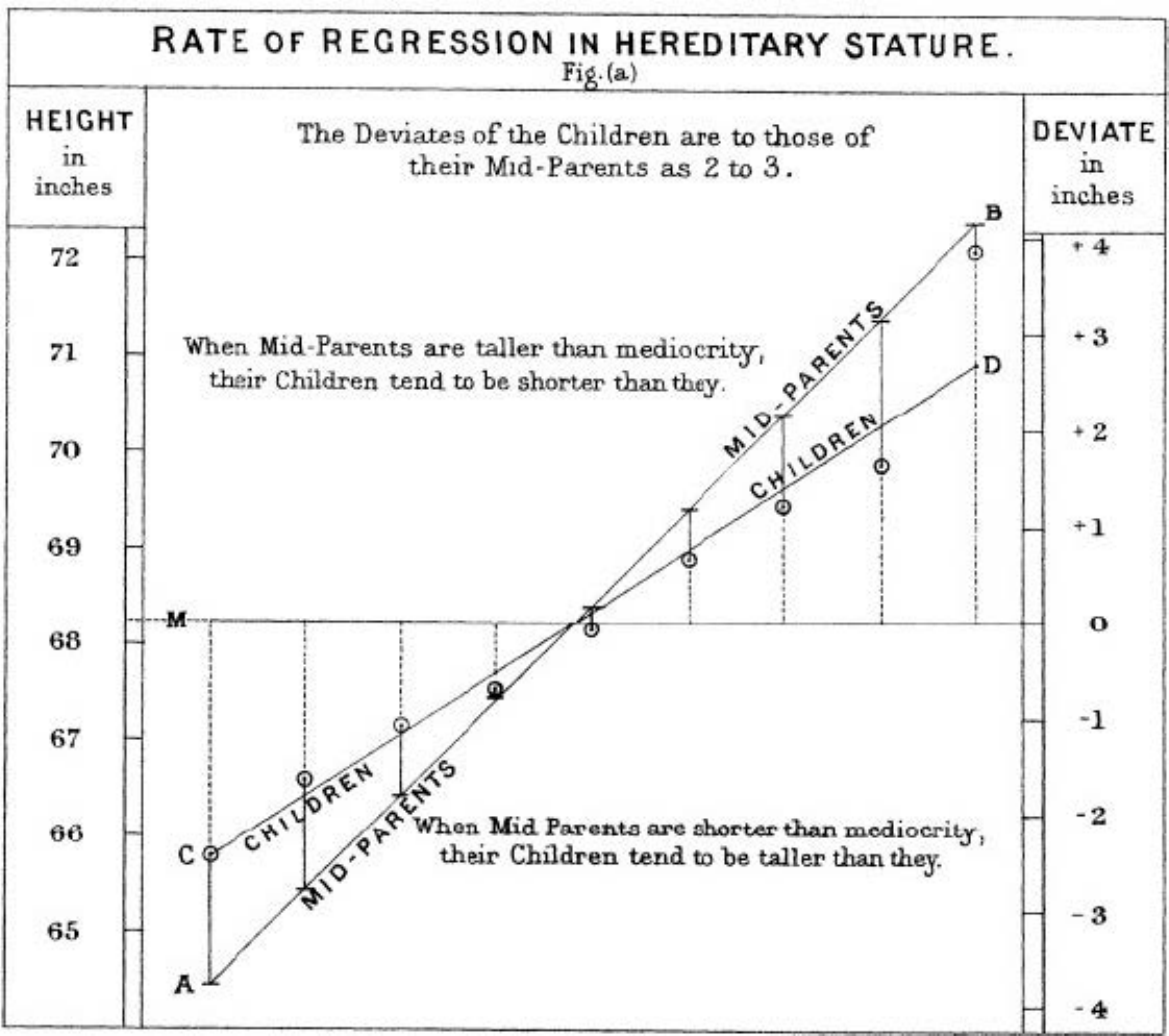}
    \caption{\small Plate IX from Galton (1886): Rate of regression in hereditary stature. Mid-parent height (in inches) is plotted on the horizontal axis, and the median child height is plotted for each parental height band. The identity line (AB) represents no change between generations, while the empirical line (CD) shows that children of extremely tall parents tend to be shorter, and children of extremely short parents tend to be taller--illustrating regression toward the mean.}
\label{fig:img_reading_4_1}
\end{figure}

This regression-to-the-mean phenomenon has since been recognized as widespread in natural processes, such as fluctuations in biological traits, measurement errors in physical experiments, seasonal variations in climate variables, and performance changes in repeated competitions or athletic events. By counteracting extreme fluctuations, regression toward the mean helps maintain the stability of natural systems. Indeed, if such regression did not occur, human height would diverge dramatically over generations, leading to unrealistically extreme groups---a phenomenon never observed in reality.

Its opposite, the \emph{Matthew effect}~\footnote{The term ``Matthew effect'' originates from a passage in the Gospel of Matthew: ``For to everyone who has, more will be given \ldots but from the one who has not, even what he has will be taken away.''}, is more commonly observed in social systems. In contrast to regression toward the mean---which pulls extreme outcomes back toward the average and thereby stabilizes natural processes---the Matthew effect reinforces initial differences and leads to increasing divergence over time. Individuals or groups who start with advantages tend to accumulate even more resources, opportunities, and recognition, while those who begin with disadvantages often fall further behind.

\begin{figure}[ht] \centering \includegraphics[width=0.5\columnwidth]{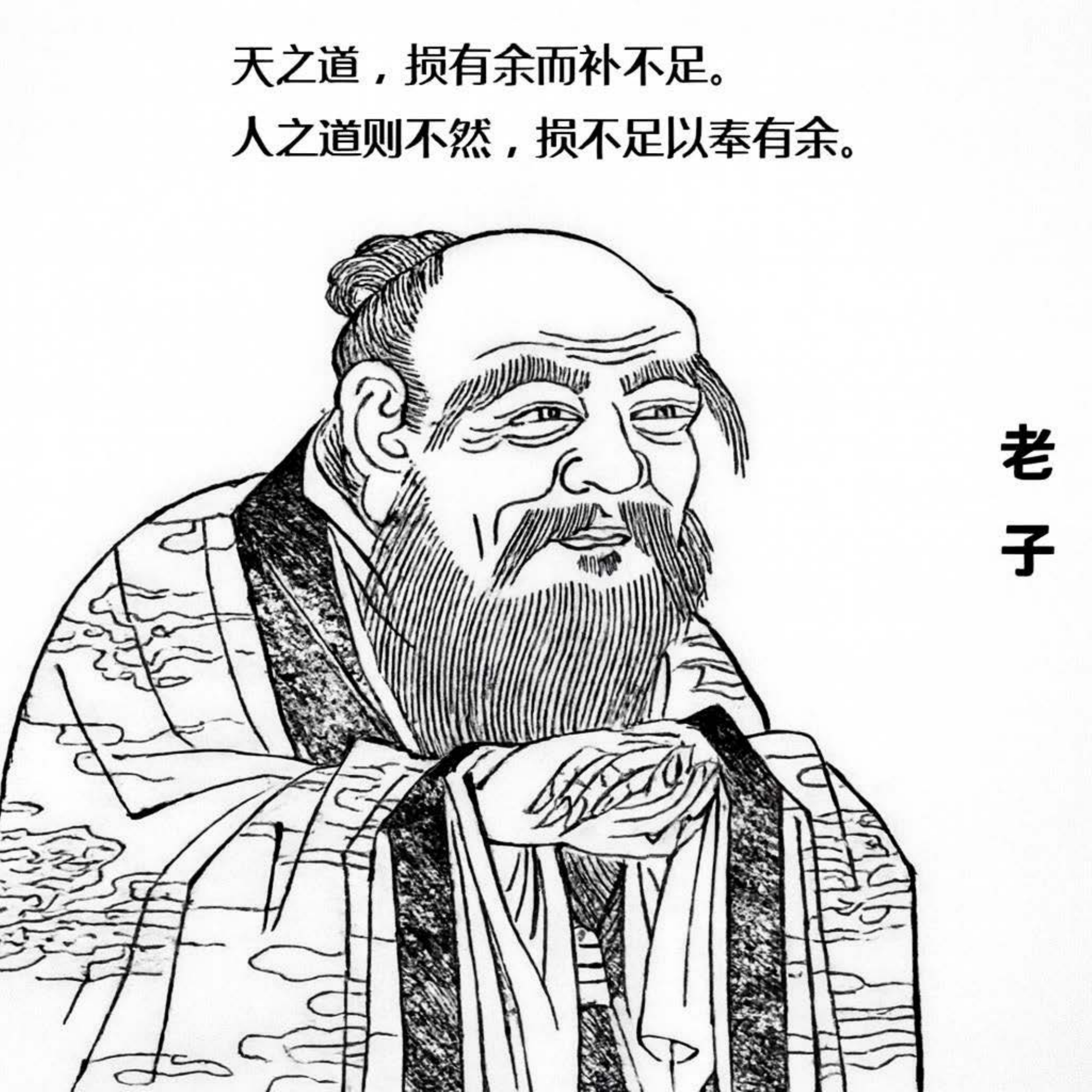}
    \caption{\small Laozi and the passage from the \emph{Tao Te Ching}: ``The Way of Heaven reduces excess and replenishes deficiency; the way of humankind is the opposite -- it takes from the poor to serve the rich.''}
\label{fig:laozi}
\end{figure}

The Matthew effect is more commonly observed in social phenomena, whereas regression toward the mean appears more frequently in natural processes. Remarkably, more than 2,500 years ago, the Chinese philosopher Laozi had already perceived this contrast with great insight (Figure~\ref{fig:laozi}). In his work \emph{Tao Te Ching}, he articulated this idea as follows: ``The Way of Heaven reduces excess and replenishes deficiency; the way of humankind is the opposite---it takes from the poor to serve the rich.''

The phenomenon of ``regression toward the mean'' was originally discovered through the construction of a statistical model describing the relationship between a variable $x$ and another variable $y$. Since then, the term ``regression'' has been retained to refer specifically to modeling methods that use explanatory variables to predict response variables, even though modern regression analysis is no longer limited to studying regression-to-the-mean behavior.

\newpage

\section{Linear Regression}\label{sec:lr:linreg}

Suppose the final examination of the course is approaching, and the instructor wishes to predict students' final scores in order to identify those who may need additional support. In essence, the instructor's goal is to determine a function $f$ whose input is a student's performance during the semester and whose output is the predicted final score:
$$f(\text{student performance}) = \text{final score}.$$
This is a typical regression problem. Let $\bm{x}$ denote a feature vector
representing a student's course performance. The vector $\bm{x}$ may contain multiple components such as $x_{qa}$ (the number of questions answered in class), $x_{at}$ (attendance rate), $x_{hw}$ (homework score), and $x_{te}$ (midterm exam score). We use $\hat{y}$ to denote the predicted final score. By collecting past data and fitting a regression model, the instructor may obtain a predictive function of the form
$$\hat{y} = 7.0 + 3.0\, x_{qa} + 1.2\, x_{at} + 0.4\, x_{hw} + 0.6\, x_{te}.$$
This is a \emph{linear regression model}. It is called ``linear'' because the relationship between each feature (\eg $x_{qa}, x_{at}$) and the predicted outcome $\hat{y}$ is linear. In the following subsections, we provide a detailed introduction to the linear regression model.

\subsection{Elements of Linear Regression}

Here, we introduce the linear regression model through the three elements of regression modeling: the \textbf{regression function}, the
\textbf{loss function}, and the \textbf{parameter estimation} (see Section~\ref{sec:elements_regression}).

\paratitle{Regression Function.} {Linear regression} is a regression analysis model that uses a linear equation as its regression function. It models the relationship between the feature $\bm{x}$ and the label $y$ using a linear function, expressing the label as a linear combination of one or more
features.

\begin{mydef}[Linear Regression]
Given a dataset consisting of $M$ samples,
\begin{equation}
    \mathcal{D}=\{(\bm{x}_1, y_1),(\bm{x}_2, y_2), \ldots, (\bm{x}_M, y_M)\},
\end{equation}
where $\bm{x}_m \in \mathbb{R}^{N}$ is an $N$-dimensional feature vector and
$y_m \in \mathbb{R}$ is its corresponding label for $m=1,\ldots,M$.
\textbf{Linear regression} models the relationship between features and the label by predicting $y$ as a linear combination of the features. Its functional form is
\begin{equation}\label{eq:lr:linear}
    \hat{y}_m = f(\bm{x}_m; \bm{\theta})
    = \theta_0 + \theta_1 x_{m,1} + \cdots + \theta_N x_{m,N}
    = \bm{\theta}^{\top} {\bm{x}}_m,
\end{equation}
where $\top$ denotes the transpose operator. Specifically,

$\bullet$ $\bm{\theta} = (\theta_0, \theta_1, \ldots, \theta_N)^{\top}
    \in \mathbb{R}^{N+1}$ is the parameter vector, where $\theta_0$ is the bias term;

$\bullet$ ${\bm{x}}_m = (1, x_{m,1}, \ldots, x_{m,N})^{\top}$ is the augmented
    feature vector formed by prepending a $1$;

$\bullet$ $\hat{y}_m$ denotes the model's prediction or estimate of the true value $y_m$.
\end{mydef}

\paratitle{Loss Function.} Linear regression adopts the \textbf{mean squared error (MSE)} as its loss function to measure the prediction error between $\hat{y}_m$ and $y_m$. It is defined as
\begin{equation}
    \label{eq:lr:mse_loss}
    \mathcal{L}(\bm{\theta})
    = \frac{1}{M}\sum_{m=1}^{M}(\hat{y}_m - y_m)^2
    = \frac{1}{M}\sum_{m=1}^{M}\left(\bm{\theta}^{\top}{\bm{x}}_m - y_m\right)^2.
\end{equation}
It can be seen that the MSE loss computes the squared prediction error
$(y_m - \hat{y}_m)^2$ for each sample in the dataset, and uses their average
as an overall measure of the regression function's performance on the dataset.
The squaring of the term $y_m - \hat{y}_m$ prevents positive and negative
errors from canceling each other out. Using the absolute error
$\lvert y_m - \hat{y}_m \rvert$ would achieve the same purpose, but the squared
form has two advantages: it makes the loss function differentiable, which is
important for parameter estimation, and it penalizes larger errors more
strongly.

The mean squared error is also the most commonly used loss function for regression models whose response variable is continuous (that is, for regression tasks in machine learning or data mining).

\paratitle{Parameter Estimation.} Given the dataset $\mathcal{D}$, the objective of parameter estimation in linear regression is to minimize the mean squared error loss function, that is,
\begin{equation}\label{eq:mse_loss_2}
    \bm{\theta}^*
    = \mathop{\arg\min}\limits_{\bm{\theta}}
    \frac{1}{M} \sum_{m=1}^{M}
    \left(\bm{\theta}^{\top}{\bm{x}}_m - y_m\right)^2,
\end{equation}
where $\bm{\theta}^*$ is the estimated optimal parameter. There are multiple approaches for solving this optimization problem. For linear
regression, one can obtain a closed-form solution using the method of ordinary
least squares, or apply iterative optimization algorithms such as gradient
descent. We introduce both types of methods in the next subsection.

\subsection{Parameter Estimation for Linear Regression: Least Squares}

The method of least squares provides a closed-form solution for estimating the
parameters of a linear regression model. The core idea is to find the parameter
vector $\bm{\theta}$ that minimizes the loss function $\mathcal{L}(\bm{\theta})$
in Eq.~\eqref{eq:lr:mse_loss} by identifying its minimum point. Because the
optimization problem in Eq.~\eqref{eq:mse_loss_2} seeks to minimize a sum of
squared errors, this approach is known as the ``least squares'' method. In linear regression, the parameter estimation method based on minimizing the
squared error is known as the \emph{Ordinary Least Squares} (OLS) method.

In the Ordinary Least Squares metohd, we first rewrites the linear regression loss function in Eq.~\eqref{eq:lr:mse_loss} in matrix form. For the dataset $\mathcal{D}$, we define the design matrix of the features as
\begin{equation}
    {\bf{X}} =
    \begin{pmatrix}
        1      & x_{1,1} & \cdots & x_{1,N} \\
        1      & x_{2,1} & \cdots & x_{2,N} \\
        \vdots & \vdots  & \ddots & \vdots  \\
        1      & x_{M,1} & \cdots & x_{M,N}
    \end{pmatrix}.
\end{equation}
Each row corresponds to the transposed augmented feature vector
${\bm{x}}_m^{\top}$ for sample $m = 1,\ldots,M$. The first column of
${\bf{X}}$ consists of ones and is introduced to incorporate the bias term
into the matrix formulation of the regression model. Similarly, we construct a
vector containing all response values:
\begin{equation}
    \bm{y} = (y_1, \ldots, y_M)^{\top}.
\end{equation}

With these definitions, the loss function of linear regression in
Eq.~\eqref{eq:lr:mse_loss} can be rewritten in matrix form as
\begin{equation}\label{eq:lr:lss}
    \begin{split}
        \mathcal{L}(\bm{\theta})
        &= \sum_{m=1}^{M} \left(\bm{\theta}^{\top}{\bm{x}}_m - y_m\right)^2 \\
        &= ({\bf{X}}\bm{\theta} - \bm{y})^{\top}
           ({\bf{X}}\bm{\theta} - \bm{y}).
    \end{split}
\end{equation}
Careful readers may notice that the loss function used here omits the factor
$\frac{1}{M}$ that appears in Eq.~\eqref{eq:lr:mse_loss}. This omission is
intentional: the constant factor does not affect the value of the parameter
vector that minimizes the loss. For simplicity of notation, we leave it out in
the derivation.

\begin{figure}[t]
    \centering\label{fig:max_min_saddle}
    \subfigure[Local Minimum]{\includegraphics[width=0.3\linewidth]{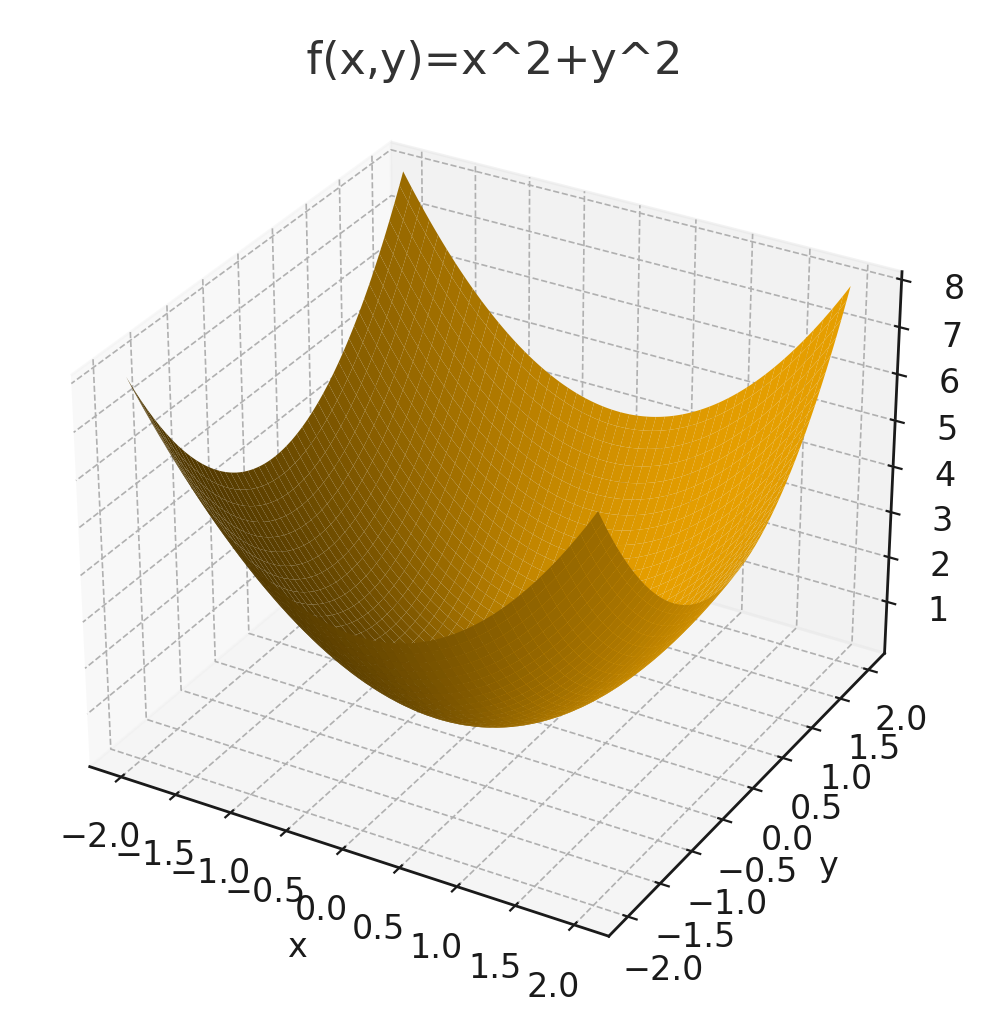}\label{fig:lr:local_mini}}~~
    \subfigure[Local Maximum]{\includegraphics[width=0.3\linewidth]{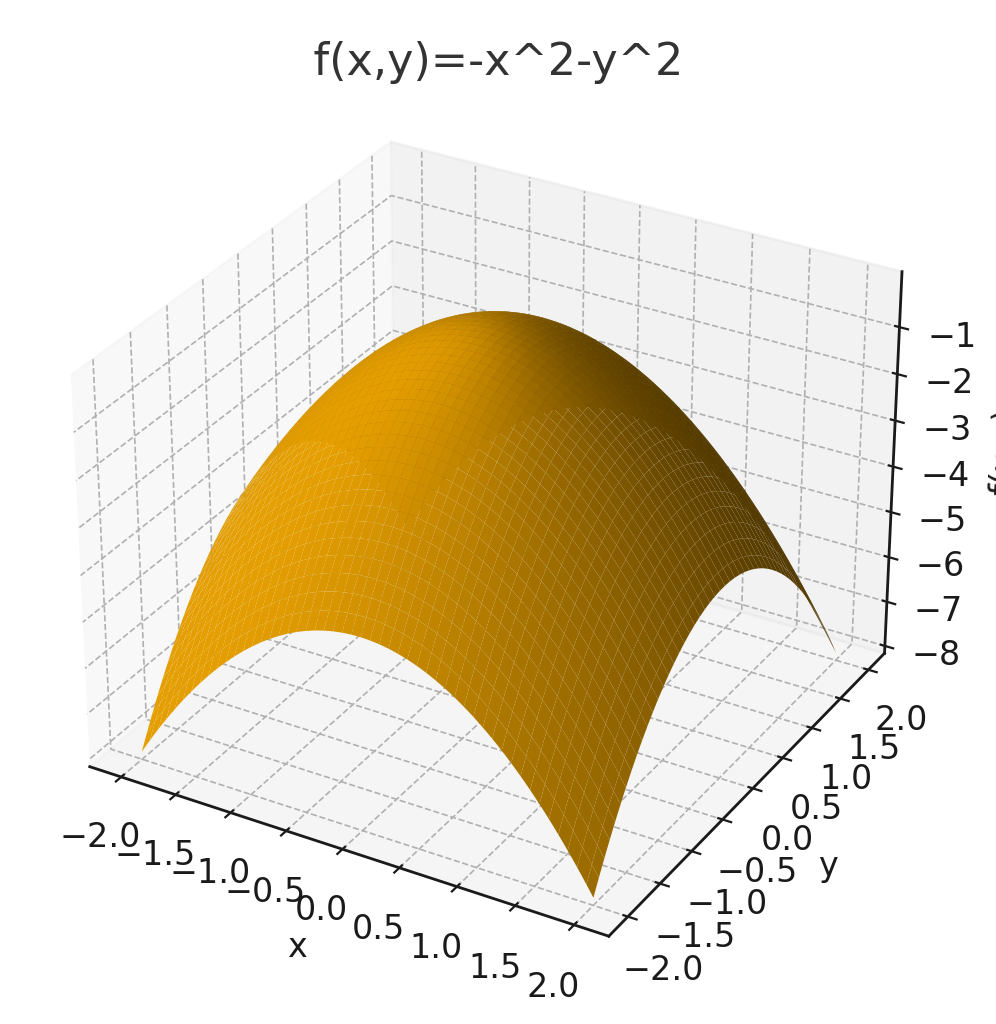}\label{fig:lr:local_maxi}}~~
    \subfigure[Saddle Point]{\includegraphics[width=0.3\linewidth]{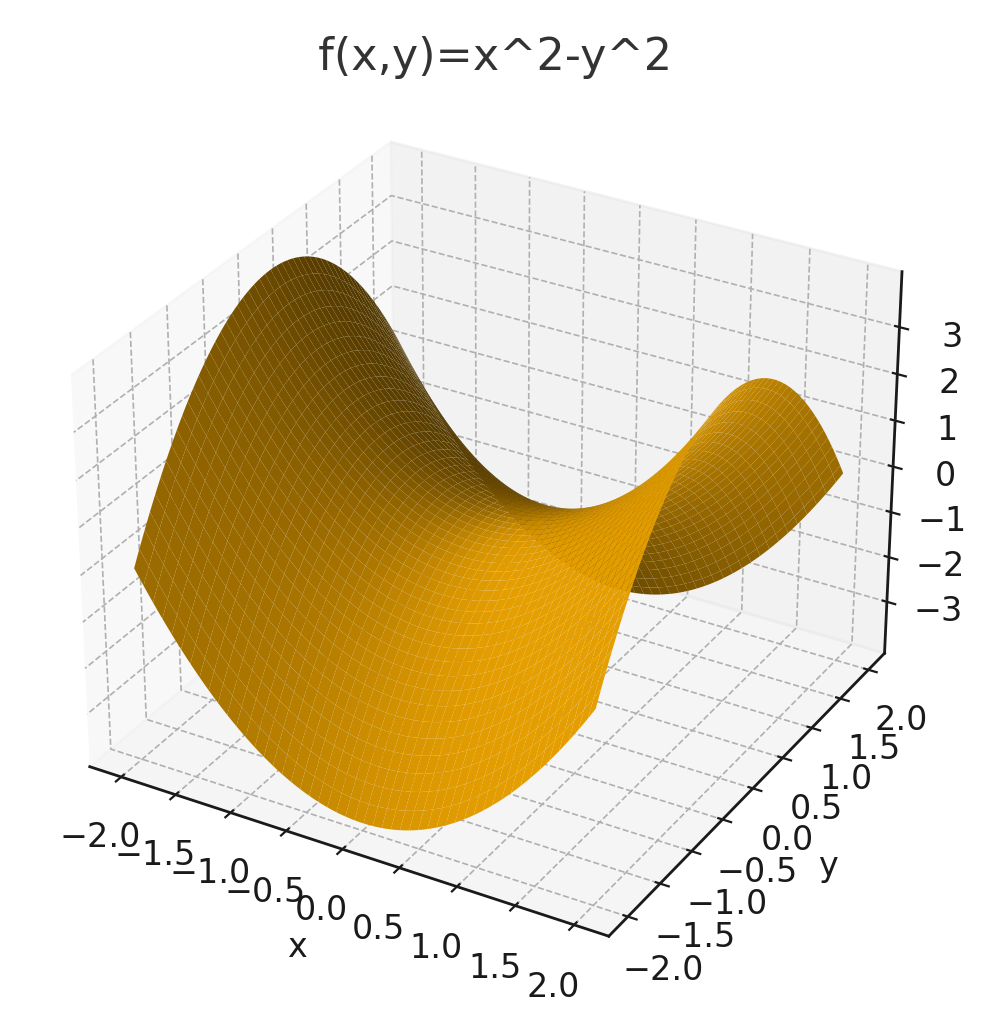}\label{fig:lr:saddle_point}}
    \caption{\small Geometric interpretation of extrema for multivariate functions: a positive definite Hessian corresponds to a bowl-shaped surface (local minimum), a negative definite Hessian corresponds to a dome-shaped surface (local maximum), and an indefinite Hessian produces a saddle-shaped surface (saddle point).}
\end{figure}

\paratitle{Preliminary Knowledge of Calculus.} Before calculating the minimum value of Eq.~\eqref{eq:lr:lss}, we need to review some basic knowledge of calculus first. From basic knowledge of calculus, given a real-valued function $y = f(x)$, the function attains an extremum at any point $x$ where its first derivative $\frac{\mathrm{d}f(x)}{\mathrm{d}x} = 0$. Moreover, if the second derivative $\frac{\mathrm{d}^2 f(x)}{\mathrm{d} x^2} > 0$, the extremum is a minimum; conversely,if the second derivative is strictly negative, the point is a maximum. For a multivariate real-valued function $y = f(x_1, x_2, \ldots, x_n)$, its first derivative becomes a vector, $\frac{\partial f(\bm{x})}{\partial \bm{x}}
= \left( \frac{\partial f(\bm{x})}{\partial x_1}, \, \frac{\partial f(\bm{x})}{\partial x_2}, \, \ldots, \, \frac{\partial f(\bm{x})}{\partial x_n}
\right)^{\top}$. A point $\bm{x}$ satisfying $\frac{\partial f(\bm{x})}{\partial \bm{x}} = \bm{0}$ is called a \emph{stationary point}. Such a point may correspond to one of three possibilities: a local minimum, a local maximum, or a saddle point. Determining which case occurs requires examining the second-order derivatives of $f(\bm{x})$. The second-order derivatives include terms such as
$\frac{\partial^2 f}{\partial x_i^2}$, $\frac{\partial^2 f}{\partial x_i \partial x_j}$, which construct a matrix known as the \emph{Hessian matrix}, defined as
\[
{\bf H}(f) =
\begin{pmatrix}
\frac{\partial^2 f}{\partial x_1^2} &
\frac{\partial^2 f}{\partial x_1 \partial x_2} &
\cdots &
\frac{\partial^2 f}{\partial x_1 \partial x_n} \\
\frac{\partial^2 f}{\partial x_2 \partial x_1} &
\frac{\partial^2 f}{\partial x_2^2} &
\cdots &
\frac{\partial^2 f}{\partial x_2 \partial x_n} \\
\vdots & \vdots & \ddots & \vdots \\
\frac{\partial^2 f}{\partial x_n \partial x_1} &
\frac{\partial^2 f}{\partial x_n \partial x_2} &
\cdots &
\frac{\partial^2 f}{\partial x_n^2}
\end{pmatrix}.
\]
If the Hessian is \emph{positive definite} at a stationary point, the function attains a local minimum there. If the Hessian is \emph{negative definite}, the stationary point is a local maximum. When the Hessian is neither positive nor negative definite, the stationary point is a saddle point, indicating that the function curves upward in some directions and downward in others (See Figure~\ref{fig:max_min_saddle}).

\paratitle{Calculating the Optimal Parameters.} Following the principle given in preliminary knowledge of calculus, we set the first derivative of $\mathcal{L}(\bm{\theta})$ to zero to obtain the candidate stationary point $\bm{\theta}$, and then examine the second derivative to verify whether $\mathcal{L}(\bm{\theta})$ indeed achieves a minimum at that point. To compute the first- and second-order derivatives of $\mathcal{L}(\bm{\theta})$ in Eq.~\eqref{eq:lr:lss}, we first expand the quadratic form:
\begin{equation}\label{eq:lr:extend}
    \begin{split}
    \mathcal{L}(\bm{\theta})& = \big({\bf{X}}\bm{\theta} - \bm{y}\big)^{\top}
           \big({\bf{X}}\bm{\theta} - \bm{y}\big) \\
    & = \bm{\theta}^{\top}{\bf{X}}^{\top}{\bf{X}}\bm{\theta}
      - \bm{\theta}^{\top}{\bf{X}}^{\top}\bm{y}
      - \bm{y}^{\top}{\bf{X}}\bm{\theta}
      + \bm{y}^{\top}\bm{y}.
    \end{split}
\end{equation}
Taking the derivative of $\mathcal{L}(\bm{\theta})$ with respect to $\bm{\theta}$ yields the first derivative of Eq.~\eqref{eq:lr:extend} as\footnote{
The derivations in Eqs.~\eqref{eq:lr:1order_derivation} and
\eqref{eq:lr:2order_derivation} make use of several basic rules from matrix
calculus, including:
\begin{equation*}
    \begin{cases}
          \dfrac{\partial (\bm{x}^{\top}\bm{a})}{\partial \bm{x}}
          = \dfrac{\partial (\bm{a}^{\top}\bm{x})}{\partial \bm{x}}
          = \bm{a},\\[0.8em]
          \dfrac{\partial (\bm{x}^{\top}\mathbf{A}\bm{x})}{\partial \bm{x}}
          = \mathbf{A}\bm{x} + \mathbf{A}^{\top}\bm{x},
    \end{cases}
\end{equation*}
where $\bm{a}$ and $\bm{x}$ are vectors, and $\mathbf{A}$ is a matrix.
}
\begin{equation}\label{eq:lr:1order_derivation}
    \begin{split}
        \frac{\partial \mathcal{L}(\bm{\theta})}{\partial \bm{\theta}}
        &= \frac{\partial}{\partial\bm{\theta}}
           \Big(\bm{\theta}^{\top}{\bf{X}}^{\top}{\bf{X}}\bm{\theta} \Big)
           - \frac{\partial}{\partial\bm{\theta}}
           \Big(\bm{\theta}^{\top}{\bf{X}}^{\top}\bm{y}\Big)
           - \frac{\partial}{\partial\bm{\theta}}
           \Big(\bm{y}^{\top}{\bf{X}}\bm{\theta}\Big)
           + \frac{\partial}{\partial\bm{\theta}}
           \Big(\bm{y}^{\top}\bm{y}\Big) \\
        &= \Big({\bf{X}}^{\top}{\bf{X}}\bm{\theta}
           + ({\bf{X}}^{\top}{\bf{X}})^{\top}\bm{\theta}\Big)
           - {\bf{X}}^{\top}\bm{y}
           - {\bf{X}}^{\top}\bm{y}
           + 0 \\
        &= 2{\bf{X}}^{\top}{\bf{X}}\bm{\theta}
           - 2{\bf{X}}^{\top}\bm{y}.
    \end{split}
\end{equation}

Setting the first-order derivative to zero yields the stationary point:
\begin{equation}~\label{eq:LR_close}
    \begin{split}
        \frac{\partial \mathcal{L}(\bm{\theta})}{\partial \bm{\theta}} = \bm{0}
        &\Rightarrow\quad
        2{\bf{X}}^{\top}{\bf{X}}\bm{\theta}
        - 2{\bf{X}}^{\top}\bm{y} = \bm{0} \\[0.4em]
        &\Rightarrow\quad
        {\bf{X}}^{\top}{\bf{X}}\bm{\theta}
        = {\bf{X}}^{\top}\bm{y} \\[0.4em]
        &\Rightarrow\quad
        \bm{\theta}^{*}
        = \left({\bf{X}}^{\top}{\bf{X}}\right)^{-1}
        {\bf{X}}^{\top}\bm{y},
    \end{split}
\end{equation}
where $\bm{0}$ denotes the zero vector having the same dimensionality as
${\partial \mathcal{L}(\bm{\theta})}/{\partial \bm{\theta}}$. At this point, we have obtained the stationary point of the linear regression
model, $\bm{\theta}^{*} = \left({\bf{X}}^{\top}{\bf{X}}\right)^{-1}{\bf{X}}^{\top}\bm{y}$. To determine whether this stationary point corresponds to a minimum of the loss
function, we examine the second-order derivative of $\mathcal{L}(\bm{\theta})$,
that is, the Hessian matrix:
\begin{equation}\label{eq:lr:2order_derivation}
    \frac{\partial^{2} \mathcal{L}(\bm{\theta})}{\partial \bm{\theta}^{2}}
    = \frac{\partial}{\partial \bm{\theta}}
      \big( 2{\bf{X}}^{\top}{\bf{X}}\bm{\theta} \big)
      - \frac{\partial}{\partial \bm{\theta}}
      \big( 2{\bf{X}}^{\top}\bm{y} \big)
    = 2{\bf{X}}^{\top}{\bf{X}}.
\end{equation}
The Hessian matrix of the loss function is therefore ${\bf H} = 2{\bf X}^{\top}{\bf X}$, which is always \emph{positive semidefinite}, since for any vector $\bm{v}$, we have
\[
\bm{v}^{\top}({\bf X}^{\top}{\bf X})\bm{v}
    = ({\bf X}\bm{v})^{\top}({\bf X}\bm{v})
    = \|{\bf X}\bm{v}\|^{2} \ge 0.
\]
Moreover, if the columns of $\bm{X}$ are linearly independent, then
${\bf X}^{\top}{\bf X}$ is \emph{positive definite}, because
\[
\bm{v}^{\top}({\bf X}^{\top}{\bf X})\bm{v}
    = \|{\bf X}\bm{v}\|^{2} = 0
    \quad\Longrightarrow\quad
    {\bf X}\bm{v} = \bm{0}
    \quad\Longrightarrow\quad
    \bm{v} = \bm{0}.
\]
In this case, the Hessian is positive definite, implying that the stationary
point $\bm{\theta}^{*}$ is indeed a \emph{global} minimum of the loss function.

Therefore, when $\bm{\theta} = \bm{\theta}^*$, the loss function
$\mathcal{L}(\bm{\theta})$ attains its minimum value. The resulting solution
obtained by the method of least squares is
\begin{equation}\label{eq:lr:theta_star}
    \bm{\theta}^* = \left({{\bf X}}^{\top}{{\bf X}}\right)^{-1}
      {{\bf X}}^{\top}\bm{y}.
\end{equation}
With the optimal parameter vector $\boldsymbol{\theta}^*$ obtained in
Eq.~\eqref{eq:lr:theta_star}, the prediction of linear regression for a new
input sample $\boldsymbol{x}_i$ is given by
$\boldsymbol{\theta}^* = (\mathbf{X}^{\top}\mathbf{X})^{-1}\mathbf{X}^{\top}\boldsymbol{y}$
into the prediction function yields
\begin{equation}
    \hat{y}_i
    = {\boldsymbol{\theta}^*}^{\top}\boldsymbol{x}_i
    = \left[
        (\mathbf{X}^{\top}\mathbf{X})^{-1}\mathbf{X}^{\top}\boldsymbol{y}
      \right]^{\top}
      \boldsymbol{x}_i.
\end{equation}
Since $(AB)^{\top} = B^{\top}A^{\top}$, this can be further written as
\begin{equation}
    \hat{y}_i
    = \boldsymbol{y}^{\top}\mathbf{X}
      (\mathbf{X}^{\top}\mathbf{X})^{-1}
      \boldsymbol{x}_i.
\end{equation}

\paratitle{Limitations of the Least Squares Method.}
The closed-form solution in Eq.~\eqref{eq:lr:theta_star} requires the data
matrix $\mathbf{X}$ to be of full column rank. This condition imposes two
important requirements on the dataset.
First, the columns of $\mathbf{X}$ must be linearly independent, meaning that
the features must not exhibit multicollinearity. Each feature must contribute
unique information rather than being expressible as a linear combination of
others.
Second, full column rank implies that $M \ge N+1$, \ie the number of samples
must be at least as large as the number of parameters to be estimated.
Otherwise, the model cannot uniquely determine the contribution of each
feature, and $\mathbf{X}^{\top}\mathbf{X}$ becomes singular.

Even when these conditions are met, the least squares method still suffers from
several practical limitations. Computing the analytic solution requires forming $\mathbf{X}^{\top}\mathbf{X}$ and inverting it, an operation with time complexity $O(N^{3})$ and memory complexity $O(N^{2})$. This becomes prohibitive in high-dimensional applications such as natural language processing or bioinformatics. When the number of samples is also large, even constructing
$\mathbf{X}^{\top}\mathbf{X}$ is computationally expensive. Moreover, if the dataset violates either of the two rank conditions -- insufficient
sample size or multicollinearity -- the matrix $\mathbf{X}^{\top}\mathbf{X}$ becomes non-invertible, and the closed-form least squares solution does not exist.

These limitations motivate the use of more scalable and robust alternatives,
such as gradient-based optimization or regularization methods, which avoid
explicit matrix inversion and remain effective in high-dimensional or
ill-conditioned settings.

Due to these limitations, the least squares method is often impractical for
large-scale or high-dimensional data mining problems. In such scenarios,
iterative optimization methods -- most notably gradient descent -- are typically used for parameter estimation. These methods will be introduced in the next section.

\subsection{Parameter Estimation for Linear Regression: Gradient Descent}~\label{know:lr:gd}

{\em Gradient Descent} is a kind of iterative optimization method. Given a loss function $\mathcal{L}(\bm{\theta})$, gradient descent starts from a randomly initialized parameter vector $\bm{\theta}^{(0)}$ and iteratively updates $\bm{\theta}$ in the direction of the \emph{negative gradient} of
$\mathcal{L}(\bm{\theta})$, \ie $-\frac{\partial \mathcal{L}(\bm{\theta})}{\partial \bm{\theta}}$. The update process continues until the value of $\mathcal{L}(\bm{\theta})$ no longer decreases significantly, indicating that the algorithm has converged. Since the negative gradient points in the direction of the steepest descent of the objective function, gradient descent is efficient methods for solving unconstrained optimization problems.

\paratitle{Principle of Gradient Descent.} We now prove the basic principle underlying gradient descent: \emph{the negative gradient direction is the direction of steepest descent} for a differentiable function.

\begin{theorem}[Steepest Descent Direction]\label{thm:steepest_descent}
Let $f:\mathbb{R}^n \to \mathbb{R}$ be differentiable at a point
$\bm{x}\in\mathbb{R}^n$, and assume that $\nabla f(\bm{x}) \neq \bm{0}$.
Among all directions $\bm{d}$ with unit norm $\|\bm{d}\|_2 = 1$, the directional
derivative of $f$ at $\bm{x}$ in the direction $\bm{d}$, is defined as
\begin{equation}\label{eq:directional_derivative}
  D f(\bm{x}; \bm{d}) = \nabla f(\bm{x})^{\top}\bm{d},
\end{equation}
which is rate of change of the function $f$ at the point $\bm{x}$ in the direction $\bm{d}$. The directional derivative $D f(\bm{x}; \bm{d})$ is minimized when
\[
\bm{d}^* = -\frac{\nabla f(\bm{x})}{\|\nabla f(\bm{x})\|_2},
\]
and the minimum value of the directional derivative is
\[
D f(\bm{x}; \bm{d}^*) = -\|\nabla f(\bm{x})\|_2.
\]
Therefore, the negative gradient direction $-\nabla f(\bm{x})$ is the direction
of steepest local decrease of $f$ at $\bm{x}$.
\end{theorem}

\begin{proof}
We restrict attention to directions with unit norm, i.e., $\|\bm{d}\|_2 = 1$.
By the Cauchy--Schwarz inequality, we have
\[
\nabla f(\bm{x})^{\top}\bm{d}
\ge -\|\nabla f(\bm{x})\|_2 \,\|\bm{d}\|_2
= -\|\nabla f(\bm{x})\|_2,
\]
and
\[
\nabla f(\bm{x})^{\top}\bm{d}
\le \|\nabla f(\bm{x})\|_2 \,\|\bm{d}\|_2
= \|\nabla f(\bm{x})\|_2.
\]
Thus, among all unit vectors $\bm{d}$, the smallest value of the
directional derivative is $-\|\nabla f(\bm{x})\|_2$.

Moreover, equality in the Cauchy--Schwarz inequality holds if and only if
$\bm{d}$ is parallel to $\nabla f(\bm{x})$ but with opposite direction, i.e.,
\[
\bm{d}^* = -\frac{\nabla f(\bm{x})}{\|\nabla f(\bm{x})\|_2}.
\]
Substituting this into Eq.~\eqref{eq:directional_derivative} for the directional derivative gives
\[
D f(\bm{x}; \bm{d}^*)
= \nabla f(\bm{x})^{\top}\left(-\frac{\nabla f(\bm{x})}{\|\nabla f(\bm{x})\|_2}\right)
= -\frac{\|\nabla f(\bm{x})\|_2^2}{\|\nabla f(\bm{x})\|_2}
= -\|\nabla f(\bm{x})\|_2.
\]
Hence, the negative gradient direction $\bm{d}^*$ yields the steepest local
descent of the function, and no other unit direction can produce a larger
instantaneous decrease in $f$.
\end{proof}

The steepest descent property can be made more explicit via a first-order
Taylor expansion. For any unit direction $\bm{d}$ and small step size
$\alpha > 0$, we have
\[
f(\bm{x} + \alpha \bm{d})
= f(\bm{x}) + \alpha \nabla f(\bm{x})^{\top}\bm{d} + o(\alpha),
\]
where $o(\alpha)$ denotes higher-order terms that vanish faster than $\alpha$
as $\alpha \to 0$. For sufficiently small $\alpha$, the dominant term in the
change of $f$ is
\[
f(\bm{x} + \alpha \bm{d}) - f(\bm{x})
\approx \alpha \nabla f(\bm{x})^{\top}\bm{d}.
\]
Thus, to decrease $f(\bm{x})$ as much as possible for a small step size
$\alpha$, we should choose $\bm{d}$ to minimize
$\nabla f(\bm{x})^{\top}\bm{d}$ under the constraint $\|\bm{d}\|_2 = 1$.
By Theorem~\ref{thm:steepest_descent}, this is achieved by
\[
\bm{d}^* = -\frac{\nabla f(\bm{x})}{\|\nabla f(\bm{x})\|_2},
\]
which explains why gradient descent updates the parameters in the direction of
the negative gradient.

\paratitle{Implementation of Gradient Descent.} Figure~\ref{fig:lr:gd} illustrates a simple example of how the gradient descent algorithm operates. Consider a quadratic loss function $L(\theta) = (\theta - 1)^2$.
When $\theta < 1$, the derivative is $L'(\theta) = 2(\theta - 1) < 0$, meaning that $L(\theta)$ decreases as $\theta$ moves in the positive direction---that is, along the direction of the negative gradient. Similarly, when $\theta > 1$, the derivative $L'(\theta) = 2(\theta - 1) > 0$, so decreasing $L(\theta)$ requires moving $\theta$ in the negative direction, which is again the negative gradient direction. When $L'(\theta) = 0$, the derivative provides no information about where to move next, and the algorithm converges to a stationary point. Gradient descent leverages precisely this principle: at each iteration, it updates the parameter in the direction of the negative gradient so as to progressively reduce $L(\theta)$ and eventually reach the value of $\theta$ that minimizes the loss.

\begin{figure}[t]
    \centering
    \includegraphics[width=0.5\columnwidth]{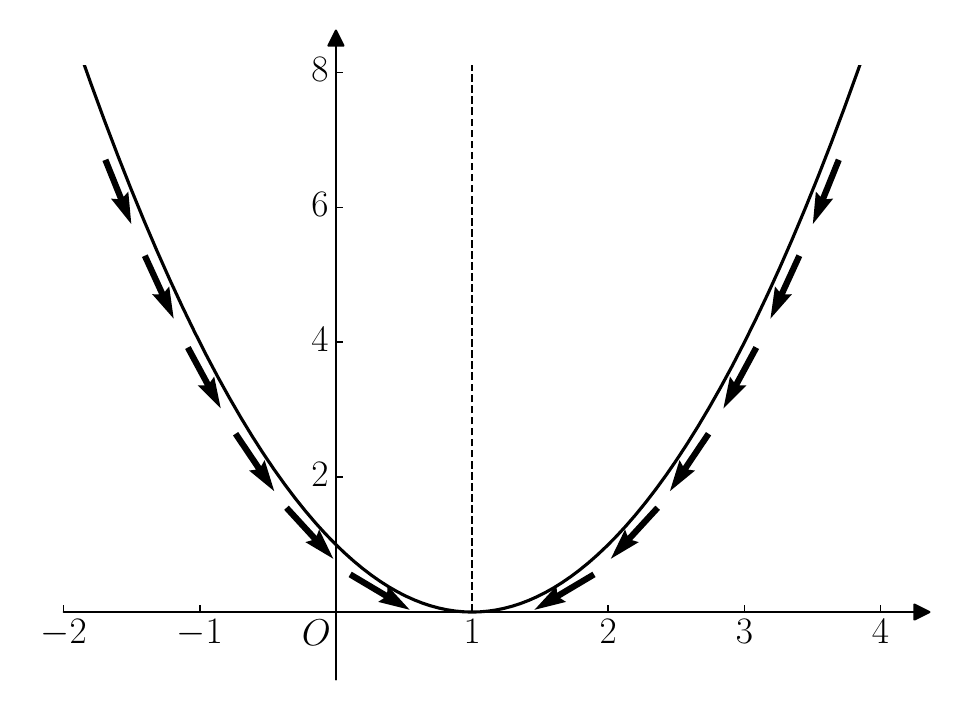}
    \caption{\small An illustrative example of gradient descent.}
    \label{fig:lr:gd}
\end{figure}

We now present a formal description of the gradient descent algorithm. Consider
a differentiable (loss) function $\mathcal{L} : \mathbb{R}^{N+1} \rightarrow \mathbb{R}$ and the unconstrained optimization problem
\begin{equation*}
    \mathop{\min}_{\bm{\theta}} \mathcal{L}(\bm{\theta}).
\end{equation*}
Gradient descent computes an approximate minimizer $\bm{\theta}^*$ through the
following three steps:

\begin{itemize}
    \item \textbf{Step 1 (Initialization).} Choose an initial parameter vector $\bm{\theta}^{(0)}$ randomly. Clearly,     $\mathcal{L}(\bm{\theta}^{(0)}) \geq \mathcal{L}(\bm{\theta}^*)$.

    \item \textbf{Step 2 (Iterative update).} Update the parameter vector in the direction of the \emph{negative gradient} of the loss. At iteration $t$, the update rule is     \begin{equation}\label{eq:lr:gd_updt}
        \bm{\theta}^{(t+1)}
        = \bm{\theta}^{(t)}
        - \eta \, \frac{\partial \mathcal{L}(\bm{\theta}^{(t)})}{\partial \bm{\theta}}.
        \end{equation}
        Here, $\eta > 0$ is the \emph{learning rate}, which is a preset hyperparameter.

    \item \textbf{Step 3 (Stopping criterion).} Choose a small constant $\delta > 0$. The iteration stops when
        \[
            \mathcal{L}(\bm{\theta}^{(t)}) -
            \mathcal{L}(\bm{\theta}^{(t+1)}) < \delta,
        \]
        in which case we set $\bm{\theta}^* \leftarrow \bm{\theta}^{(t)}$.         Alternatively, one may stop after a predefined number of iterations, or use a combination of both criteria.
\end{itemize}
The basic procedure of gradient descent is summarized in
Algorithm~\ref{alg:lr:gd}.

\begin{algorithm}[t]
    \caption{Gradient Descent Algorithm}
    \label{alg:lr:gd}
    \begin{algorithmic}[1]
        \Require
            Objective function $\mathcal{L}(\bm{\theta})$, learning rate $\eta$
        \Ensure
            Estimated optimal parameter $\bm{\theta}^*$
        \State Randomly initialize $\bm{\theta}^{(0)}$
        \State $t \leftarrow 0$
        \While{stopping criterion is not satisfied}
            \State Compute the gradient
            {\small $\partial {\mathcal{L}}(\bm{\theta}^{(t)}) / \partial \bm{\theta}$}
            \State Update parameter: {\small $\bm{\theta}^{(t+1)} \leftarrow
                \bm{\theta}^{(t)} - \eta \,
                \frac{\partial \mathcal{L}(\bm{\theta}^{(t)})}{\partial \bm{\theta}}$}

            \State $t \leftarrow t + 1$
        \EndWhile
        \State $\bm{\theta}^* \leftarrow \bm{\theta}^{(t)}$
    \end{algorithmic}
\end{algorithm}

\begin{figure}[t]
    \centering
    \includegraphics[width=0.85\columnwidth]{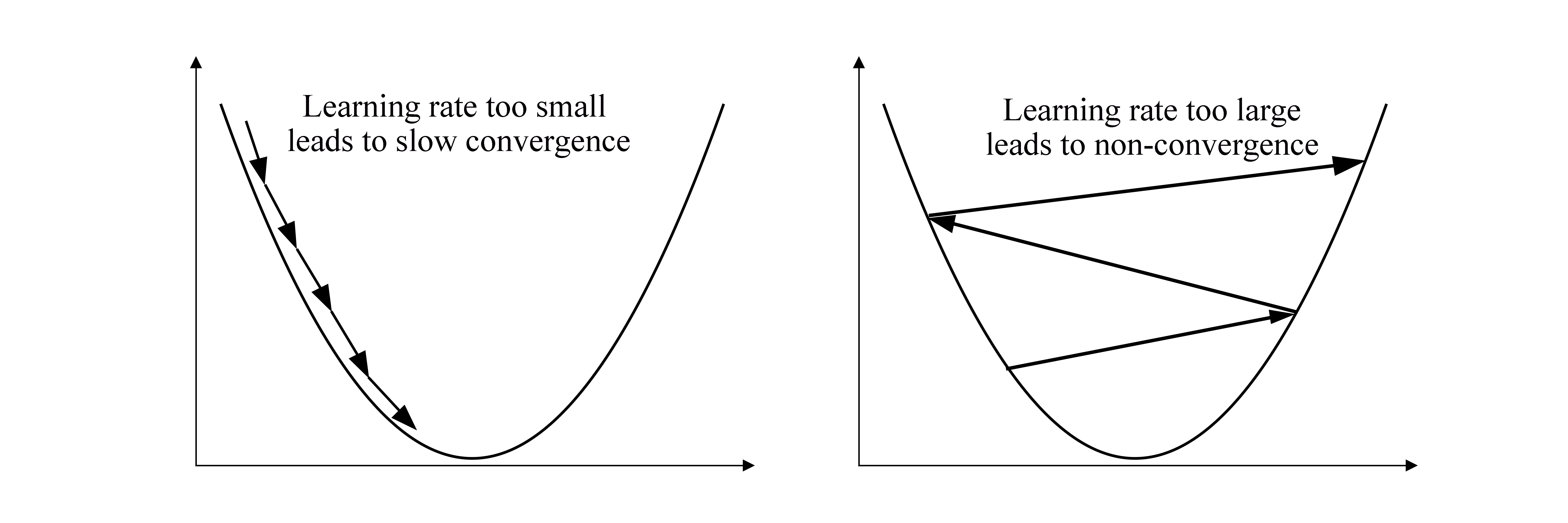}
    \caption{\small Illustration of the effects of overly small or overly large learning rates.}
    \label{fig:lr:learning_rate}
\end{figure}

In gradient descent, the learning rate $\eta$ is a crucial hyperparameter, as it
directly determines both the efficiency and the quality of convergence. As
illustrated in Fig.~\ref{fig:lr:learning_rate}, if $\eta$ is set too small, the
updates of $\bm{\theta}^{(t)}$ will move slowly, resulting in extremely slow
convergence. Conversely, if $\eta$ is set too large, the updates may overshoot
the optimum and oscillate around the minimum, preventing the algorithm from
converging at all. In practical applications, adaptive strategies are often designed to adjust $\eta$ dynamically during training. A common approach is to begin with a relatively large learning rate and reduce it once the loss stops decreasing. Training may terminate after several reductions in the learning rate. Classical scheduling methods include step decay~\citep{Ge2019StepDecay}, exponential decay~\citep{Bengio2012Practical}, and cosine annealing~\citep{Loshchilov2017SGDR}, among others.

\paratitle{Gradient Computation for Linear Regression.} To compute the gradient with respect to $\bm{\theta}$, we expand the loss for a single sample:
\[
\mathcal{L}_m(\bm{\theta})=\frac{1}{2}(\hat{y}_m - y_m)^2
= \frac{1}{2}(\bm{\theta}^\top \bm{x}_m - y_m)^2.
\]
Differentiating with respect to $\bm{\theta}$ and applying the chain rule gives
\[
\frac{\partial \mathcal{L}_m}{\partial \bm{\theta}}
= (\bm{\theta}^\top \bm{x}_m - y_m)\,\frac{\partial}{\partial \bm{\theta}}(\bm{\theta}^\top \bm{x}_m).
\]
Since the derivative of a linear form is simply the input vector, we have
\[
\frac{\partial}{\partial \bm{\theta}}(\bm{\theta}^\top \bm{x}_m)=\bm{x}_m.
\]
Thus,
\[
\frac{\partial \mathcal{L}_m}{\partial \bm{\theta}}
= (\hat{y}_m - y_m)\bm{x}_m.
\]
This result shows that each training sample contributes to the gradient through the product of its prediction error $(\hat{y}_m - y_m)$ and its input features $\bm{x}_m$. Samples with larger errors or larger feature values exert a stronger influence on the parameter update.

\paratitle{Different Gradient Computation Strategies.} Given a dataset $\mathcal{D}=\{(\bm{x}_1, y_1),$ $(\bm{x}_2, y_2),$ $\ldots,$ $(\bm{x}_M, y_M)\}$ containing $M$ samples, different strategies can be used to compute the gradient depending on the application scenario. There are three representative computational methods.
\begin{itemize}
  \item \textbf{Batch Gradient Descent (BGD).} In batch gradient descent, each parameter update uses the gradients computed from \emph{all} samples in the dataset. For the linear regression loss function in Eq.~\eqref{eq:lr:mse_loss}, the gradient of $\mathcal{L}(\bm{\theta}^{(t)})$ at iteration $t$ is
    \begin{equation}\label{eq:lr:bgd}
        \frac{\partial \mathcal{L}(\bm{\theta}^{(t)})}{\partial \bm{\theta}^{(t)}}
        = \frac{1}{M}\sum_{m=1}^{M}
        \left( \bm{\theta}^{(t)\top} {\bm{x}}_m - y_m \right)
        {\bm{x}}_m.
    \end{equation}
    The advantage of BGD is that the gradient estimate is highly accurate; indeed, \eqref{eq:lr:bgd} provides an unbiased estimate of the true gradient, with its variance decreasing as the number of samples increases. The drawback is that each iteration requires a full pass over all samples, which becomes expensive for large datasets (millions or even billions of samples).
  \item \textbf{Stochastic Gradient Descent (SGD).} In many online learning scenarios, data arrive one sample at a time. In this case, each parameter update uses only a single sample:
    \begin{equation}
    \frac{\partial \mathcal{L}(\bm{\theta}^{(t)})}{\partial \bm{\theta}}
    =
    \left(\bm{\theta}^{(t)\top} {\bm{x}}_m - y_m \right) {\bm{x}}_m,
    \quad m \in \{1,\ldots, M\}.
    \end{equation}
    SGD is computationally efficient and well suited for large-scale or online settings. However, gradients computed from single samples are noisy and may fluctuate heavily, especially near convergence, making it difficult for the algorithm to settle at the minimum.
  \item \textbf{Mini-Batch Gradient Descent (MBGD).} Mini-batch gradient descent strikes a balance between BGD and SGD. At each iteration, a small batch of samples (after shuffling the dataset) is randomly selected to compute the gradient. Given a dataset $\mathcal{D}$, in the $t$-th iteration we randomly sample a mini-batch of size $B$, \ie $\mathcal{B}^{(t)}=\{(\bm{x}_{m_1},y_{m_1}),\ldots,(\bm{x}_{m_B},y_{m_B})\}$. The parameter update rule of \emph{mini-batch gradient descent} is
        \[ \bm{\theta}^{(t+1)} = \bm{\theta}^{(t)} - \eta\cdot \frac{1}{B} \sum_{m\in\mathcal{B}^{(t)}} \left(\bm{\theta}^{(t)\top}{\bm{x}}_m - y_m\right){\bm{x}}_m.
        \]
      This approach reduces gradient variance while maintaining computational efficiency, making it the most widely used strategy in deep learning and large-scale machine learning.
\end{itemize}
Figure~\ref{fig:lr:three_gd_method} shows the iterative behaviors of the three gradient-descent algorithms.

\begin{figure}[t]
    \centering
    \includegraphics[width=0.6\columnwidth]{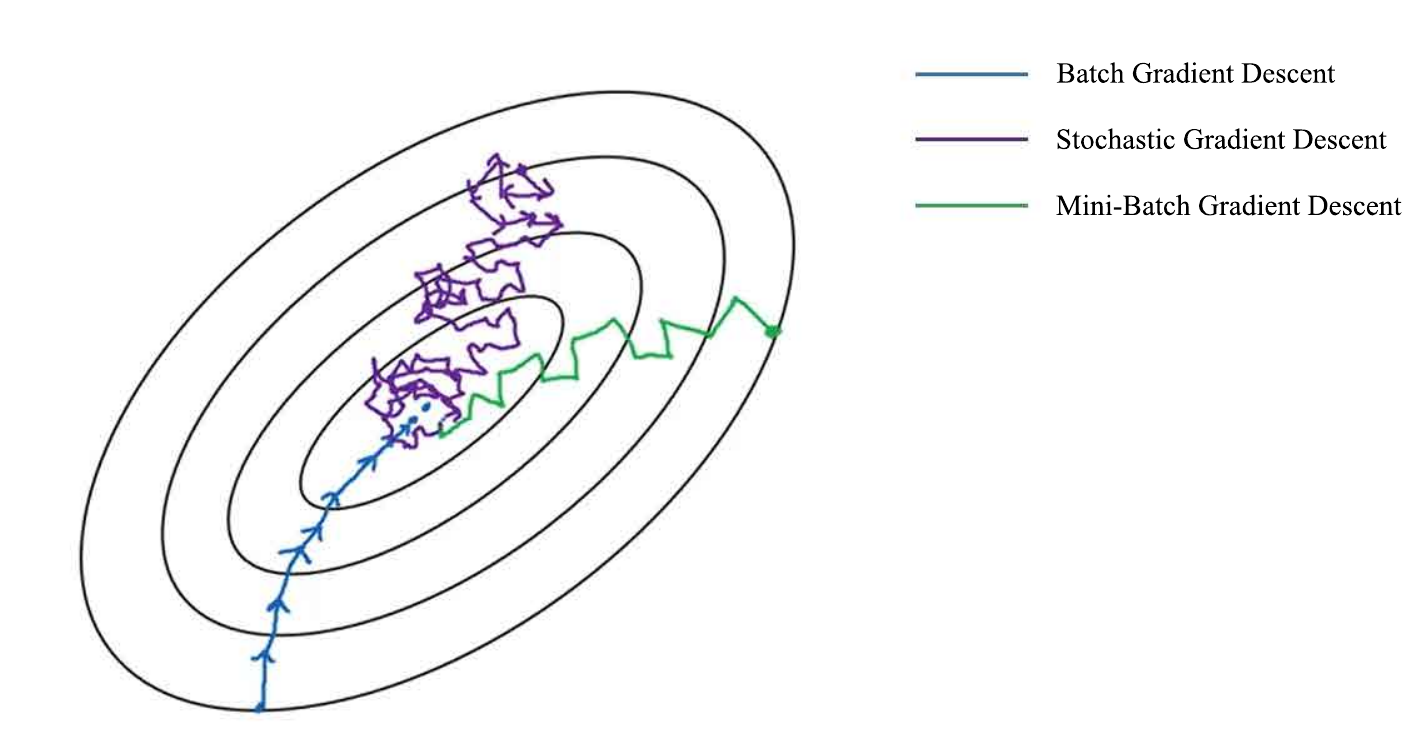}
    \caption{\small Comparison of convergence paths for batch gradient descent, stochastic gradient descent, and mini-batch gradient descent. Figure source:~\cite{NgMiniBatchGD}}.
    \label{fig:lr:three_gd_method}
\end{figure}

\subsection{Extended Discussion on Gradient Descent}

Gradient descent is one of the simplest and most classical algorithms for solving unconstrained optimization problems. Its applicability extends far beyond linear regression. For example, the back-propagation (BP) algorithm used to train deep learning models is essentially a variant of gradient descent. In this subsection, we step beyond the context of linear regression and provide several extended discussions related to gradient descent and its broader applications.

\paratitle{Local Optimality Issues.} Gradient descent can be applied to a wide variety of optimization problems, providing a practical solution even when no analytical or closed-form solution exists. However, it is important to note that gradient descent is guaranteed to converge to the global optimum \emph{only when} the objective function $\mathcal{L}(\bm{\theta})$ is convex. When $\mathcal{L}(\bm{\theta})$ is a highly non-convex function with an irregular landscape, gradient descent may converge to a \emph{local} optimum rather than the global one. For linear regression, the loss function in  Eq.~\eqref{eq:lr:mse_loss} is a convex quadratic function. Therefore, gradient descent is guaranteed to find the global minimizer in this case.

\begin{figure}[t]
    \centering
    \includegraphics[width=0.75\columnwidth]{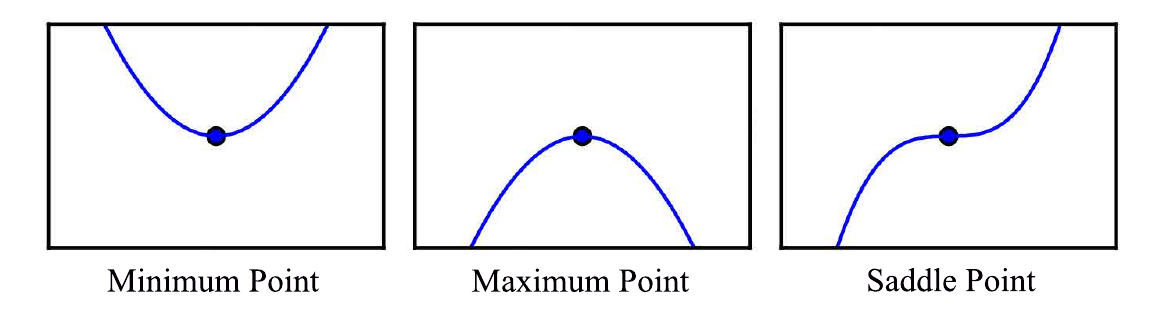}
    \caption{\small A point with zero gradient may correspond to a minimum, a maximum, or a saddle point.}
    \label{fig:lr:saddle}
\end{figure}

In gradient descent, the algorithm converges to points where the gradient is
zero, known as \textbf{critical points}. As illustrated in Fig.~\ref{fig:lr:saddle}, a critical point may correspond to a local minimum or a local maximum. However, not every critical point is an extremum; some are \textbf{saddle points}, where the function decreases in some directions and
increases in others.

In many data mining and machine learning applications, the objective function
often contains numerous local minima, as well as a large number of saddle
points located in flat or nearly flat regions, as shown in Fig.~\ref{fig:lr:1d}. This phenomenon becomes even more pronounced in high-dimensional parameter spaces, as illustrated in Fig.~\ref{fig:lr:2d}. To prevent gradient-based optimization from being trapped in poor-quality local minima or flat saddle regions, a wide range of improved optimization algorithms have been proposed. Representative methods include Momentum~\cite{rumelhart1986learning}, AdaGrad~\cite{duchi2011adaptive}, RMSProp~\cite{tieleman2012lecture}, and
Adam~\cite{kingma2015adam}.

\begin{figure}[t]
    \centering
    \subfigure[\small Different types of local minima.]{\includegraphics[width=0.45\linewidth]{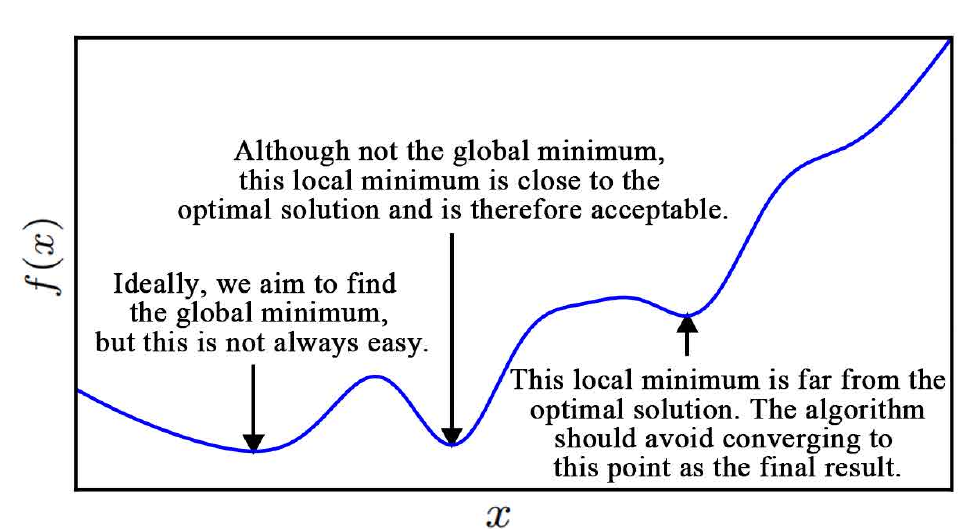}\label{fig:lr:1d}}~~
    \subfigure[\small Gradient descent trapped in a local minimum in a 2D landscape. Figure source:~\cite{NgGradientDescent}.]{\includegraphics[width=0.52\linewidth]{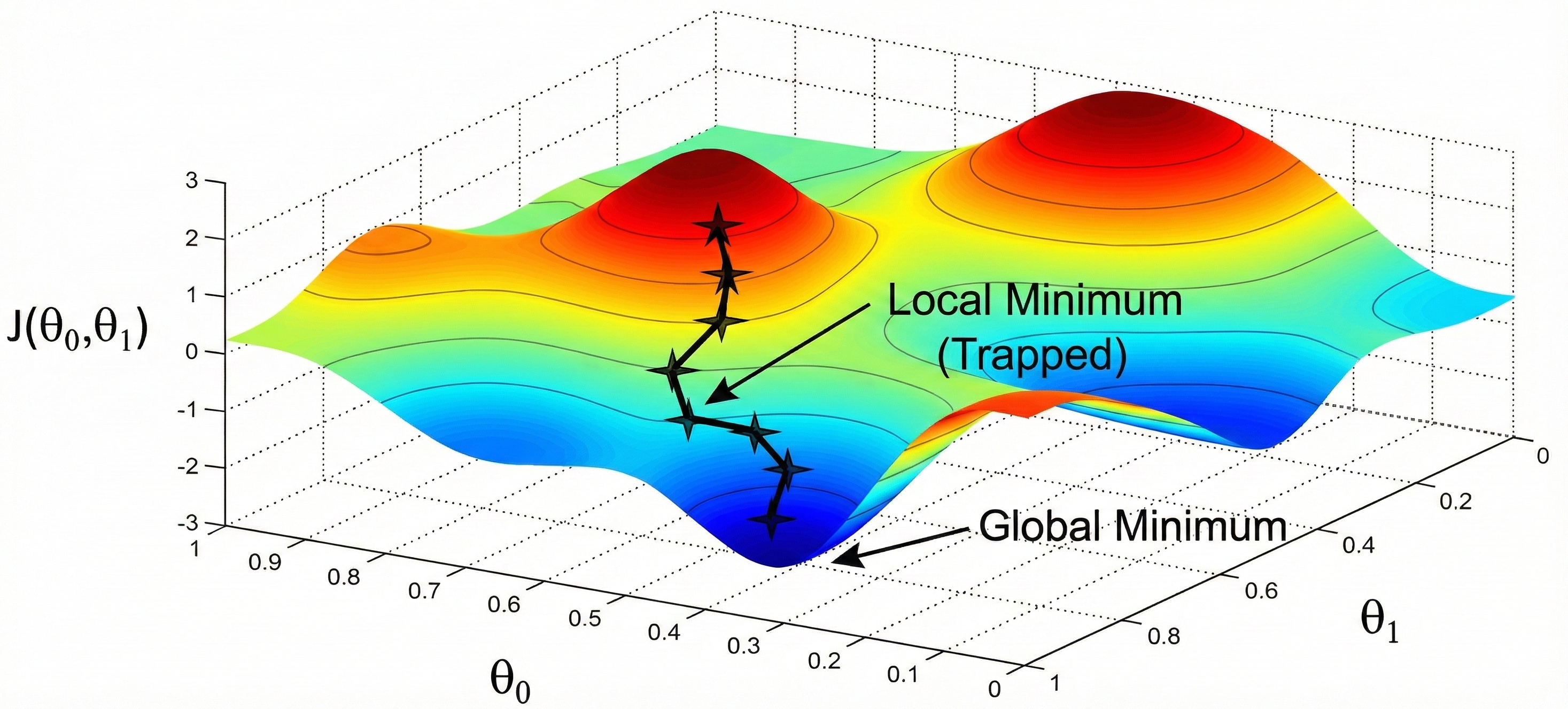}\label{fig:lr:2d}}
    \caption{\small  Illustration of the local minimum problem}
\end{figure}

\paratitle{Nondifferentiable Loss Functions.} Gradient descent requires the loss function to be differentiable; otherwise, the gradient $\frac{\partial \mathcal{L}}{\partial \bm{\theta}}$ cannot be computed. However, in many applications the loss function may contain nondifferentiable points, such as those arising from piecewise-defined functions. In practice, these nondifferentiable points can usually be ignored during training, because with a large number of samples the probability that the gradient evaluation falls exactly on such a point is extremely small. If the function is one-sided differentiable (left- or right-differentiable), the corresponding one-sided derivative may also be used to perform the gradient update.

For loss functions that contain nondifferentiable components, the \textbf{coordinate descent method} is an effective alternative optimization
technique. Unlike gradient descent, which updates all components of the
parameter vector simultaneously, coordinate descent updates only one component
of $\bm{\theta} = (\theta_1, \theta_2, \ldots, \theta_n)$ at each iteration
while keeping all other components fixed. Given a loss function $\mathcal{L}(\bm{\theta})$, during the $t$-th iteration,
coordinate descent updates $\theta_1$ through $\theta_n$ sequentially:
\begin{equation*}
    \begin{aligned}
        \theta_1^{(t+1)} &= \theta_1^{(t)} - \eta\, \frac{\partial \mathcal{L}(\theta_1^{(t)})}{\partial \theta_1},\\
        \theta_2^{(t+1)} &= \theta_2^{(t)} - \eta\, \frac{\partial \mathcal{L}(\theta_2^{(t)})}{\partial \theta_2},\\
        &\;\;\vdots\\
        \theta_n^{(t+1)} &= \theta_n^{(t)} - \eta\, \frac{\partial \mathcal{L}(\theta_n^{(t)})}{\partial \theta_n}.
    \end{aligned}
\end{equation*}
Here, $\mathcal{L}(\theta_k^{(t)})$ denotes the loss treated as a function of
$\theta_k$ only, with all other coordinates regarded as constants. Thus,
$\mathcal{L}(\bm{\theta})$ effectively reduces to a single-variable function in
$\theta_k$ during the update.

When the partial derivative with respect to a coordinate does not exist (for
example, when the loss function is piecewise-defined or contains nondifferentiable points), the update direction can be determined by comparing function values in a small neighborhood. For instance, for coordinate $\theta_n$, the update rule may be defined as:
\begin{equation*}
\theta_n^{(t+1)} =
\begin{cases}
\theta_n^{(t)}+\eta, &
\text{if }\min\big(\mathcal{L}(\theta_n^{(t)}),\;\mathcal{L}(\theta_n^{(t)}-\eta)\big)
> \mathcal{L}(\theta_n^{(t)}+\eta),\\[6pt]
\theta_n^{(t)}-\eta, &
\text{if }\min\big(\mathcal{L}(\theta_n^{(t)}),\;\mathcal{L}(\theta_n^{(t)}+\eta)\big)
> \mathcal{L}(\theta_n^{(t)}-\eta).
\end{cases}
\end{equation*}
This local comparison strategy avoids the need for derivative information,
making coordinate descent particularly suitable for nondifferentiable losses or
functions with sharp corners.

\newpage

\section{Logistic Regression and Softmax Regression}

Linear regression is suitable for handling continuous label variables,
\ie $y \in \mathbb{R}$, or discrete variables that can be reasonably
approximated as continuous (for example, age, which---although a natural
number---can be treated as continuous and remains meaningful even when
extended to non-integer values). However, in many real-world applications,
the label is often \textbf{nominal}, meaning that different values merely
indicate categorical distinctions without implying any order or magnitude
(supporting only equality and inequality operations, such as $=$ or $\neq$).
Examples include gender and product category. For such cases, alternative
regression models are required. In this section, we introduce two regression models specifically designed for discrete label variables:
\begin{enumerate}
  \item \textbf{Logistic Regression} for binary labels, also called {\em Logit Regression} (Here, ``Logit'' is not an abbreviation of ``Logistic.'').
  \item \textbf{Multinomial Logistic Regression} for multi-class labels, also known as {\em Softmax Regression}.
\end{enumerate}

\subsection{Elements of Logistic Regression}

\paratitle{Regression Function.}
Suppose we have collected a binary classification dataset
\[
    \mathcal{D}=\{(\bm{x}_1, y_1), (\bm{x}_2, y_2), \ldots, (\bm{x}_M, y_M)\},
\]
where $\bm{x}_m \in \mathbb{R}^{N}$ and $y_m \in \{0,1\}$ for
$m=1,2,\ldots,M$. For binary classification, two common label encodings are used: $y_m \in \{0,1\}$ or $y_m \in \{-1,1\}$.
Logistic regression adopts the former, \ie $y_m \in \{0,1\}$. Ideally, logistic regression would like to model the relationship between
$\bm{x}_m$ and $y_m$ using the following form:
\begin{equation}\label{eq:lr:logistic_1}
    \hat{y}_m = g(\bm{w}^\top \bm{x}_m + b),
\end{equation}
where $g(\cdot)$ is the \emph{unit step function} defined as
\begin{equation}\label{eq:unit_step_function}
      g(z)=
    \begin{cases}
        0, & z<0,\\[4pt]
        \frac{1}{2}, & z=0,\\[4pt]
        1, & z>0.
    \end{cases}
\end{equation}
In other words, if $\bm{w}^\top \bm{x}_m + b < 0$, then $\hat{y}_m = 0$;
if $\bm{w}^\top \bm{x}_m + b > 0$, then $\hat{y}_m = 1$;
and $\hat{y}_m = 0.5$ when the expression equals zero.
This formulation essentially takes the output of a linear regression model and
binarizes it: positive values map to $1$ and negative values map to $0$.

However, the unit step function is piecewise-defined: it is non-differentiable
at $z=0$ and has zero derivative everywhere else. This makes parameter optimization extremely difficult, as gradient-based methods cannot be applied.
To address this issue, logistic regression adopts a smooth approximation of the
step function -- the \textbf{standard logistic function}~\footnote{This choice is the origin of the term ``logistic''
in logistic regression.}, also known as the
\emph{sigmoid function}. The standard logistic function is defined as
\begin{equation}\label{eq:lr:logistic_2}
    \sigma(z)=\frac{1}{1+\exp(-z)}.
\end{equation}
The logistic function (or sigmoid function) is an $S$-shaped curve.
Figure~\ref{fig:lr:sigmoid} illustrates its shape and compares it with the
unit step function. As shown in the figure, logistic regression replaces the discontinuous step function with a smooth approximation, allowing continuous outputs that can still be interpreted as predictions for a binary label. Moreover, the logistic function converges to the unit step function in
Eq.~\eqref{eq:unit_step_function} in the limit:
\[
g(x)
= \lim_{k \rightarrow \infty}
  \frac{1}{1+\exp(-2k x)},
\]
where $g(x)$ denotes the unit step function.

Combining Eq.~\eqref{eq:lr:logistic_1} with the logistic function
Eq.~\eqref{eq:lr:logistic_2}, the final regression function of logistic
regression becomes
\begin{equation}\label{eq:lr:logistic_3}
    \hat{y}_m
    = \frac{1}{1+\exp\!\big(-\bm{\theta}^\top \bm{x}_m\big)}
    \in (0,1).
\end{equation}
As in linear regression, the parameter vector
$\bm{\theta} = (\theta_0, \theta_1, \ldots, {\theta}_N)^\top$ is to be
estimated, and the linear term $\bm{\theta}^\top \bm{x}_m$ is explicitly given by
\[
    \bm{\theta}^\top \bm{x}_m
    = \theta_0 + \theta_1 x_{m,1} + \cdots + \theta_N x_{m,N}.
\]

\begin{figure}[t]
    \centering
    \includegraphics[width=0.5\columnwidth]{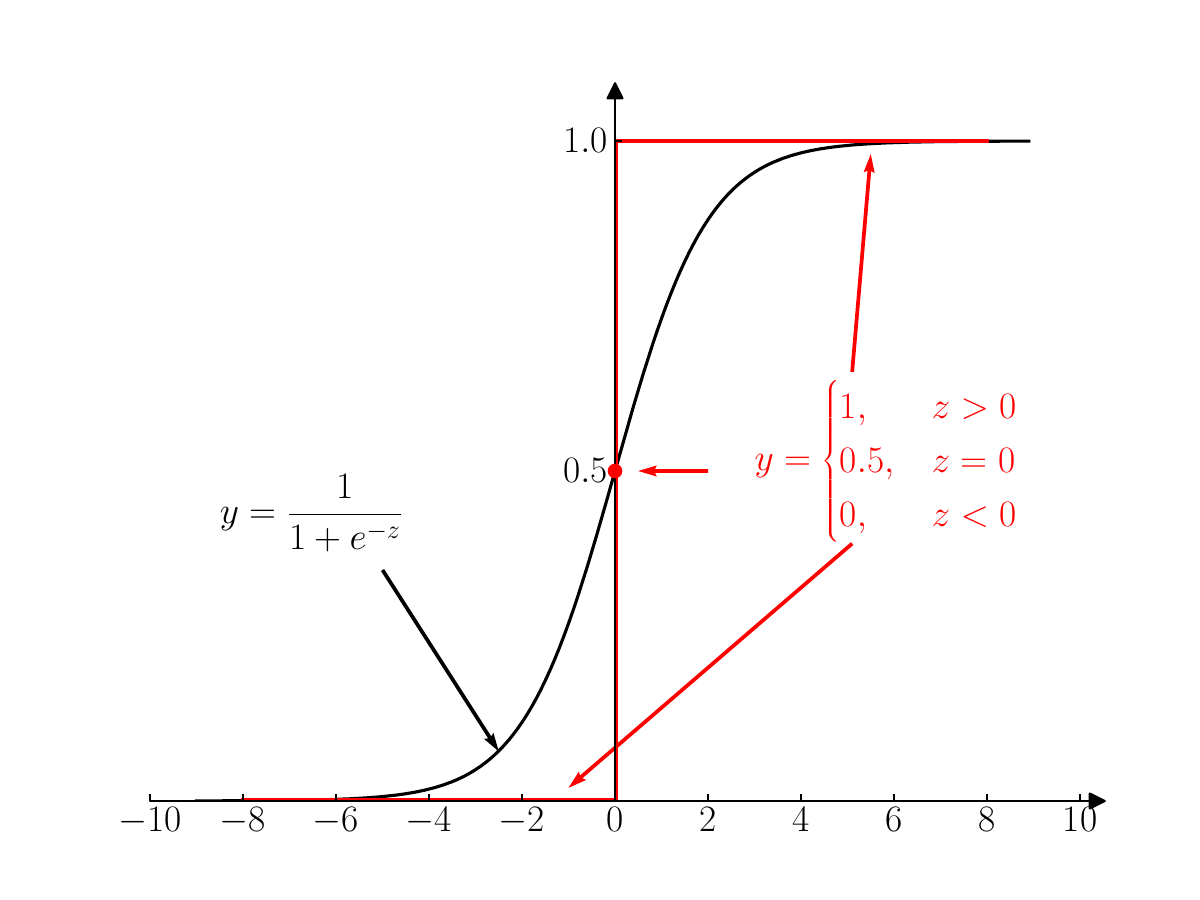}
    \caption{\small Comparison between the Sigmoid function and the unit step function. Figure source:~\cite{zhou2021machine}.
    }
    \label{fig:lr:sigmoid}
\end{figure}

In logistic regression, the output $\hat{y}$ always lies in the interval $(0,1)$, \ie $\hat{y}\in(0,1)$. Thus, $\hat{y}$ can be interpreted as the model's predicted probability that the label takes the value $y=1$. Given the model parameter vector $\bm{\theta}$ and an input feature vector $\bm{x}_m$, the probability that $y_m = 1$ is predicted as
\begin{equation}\label{eq:lr:logistic_4}
    \Pr(y_m = 1)
    = \hat{y}_m(\bm{x}_m, \bm{\theta})
    = \frac{1}{1 + \exp(-\bm{\theta}^{\top}{\bm{x}}_m)}.
\end{equation}
Correspondingly, the predicted probability that $y_m = 0$ is
\begin{equation}\label{eq:lr:logistic_5}
    \Pr(y_m = 0)
    = 1 - \hat{y}_m(\bm{x}_m, \bm{\theta})
    = \frac{\exp(-\bm{\theta}^{\top}{\bm{x}}_m)}
           {1 + \exp(-\bm{\theta}^{\top}{\bm{x}}_m)}.
\end{equation}
In practical prediction, a decision threshold is applied to $\hat{y}$. If $\hat{y}$ exceeds a chosen threshold (commonly set to $0.5$, but adjustable depending on the application), the model predicts $y=1$; otherwise, it predicts $y=0$.

\paratitle{Loss Function.} Logistic regression adopts the \textbf{cross-entropy loss} as its objective function~\footnote{In many references, this cross-entropy loss is also known as the \textbf{log loss}.}. Given a binary label $y_m$ and the corresponding logistic-regression prediction $\hat{y}_m(\bm{x}_m, \bm{\theta})$, the cross-entropy loss is defined as
\begin{equation}\label{eq:lr:logistic_6}
\mathcal{L}(\bm{\theta})
= -\frac{1}{M}\sum_{m=1}^M \Big(
    y_m \log \hat{y}_m(\bm{\theta}, \bm{x}_m)
    + (1-y_m)\log\!\big(1-\hat{y}_m(\bm{\theta}, \bm{x}_m)\big)
\Big).
\end{equation}

In Eq.~\eqref{eq:lr:logistic_6}, the loss term
$-\big(y_m \log \hat{y}_m + (1-y_m)\log(1-\hat{y}_m)\big)$
achieves its minimum value of $0$ when the prediction is correct
(\ie when $\hat{y}_m = y_m = 1$ or $\hat{y}_m = y_m = 0$).
When the model makes a completely wrong prediction (\eg $\hat{y}_m \to 1, y_m=0$ or $\hat{y}_m \to 0, y_m=1$), the loss term tends to $\infty$.
However, because the sigmoid function satisfies $0 < \hat{y}_m < 1$, the loss never actually diverges.

Substituting Eq.~\eqref{eq:lr:logistic_4} and Eq.~\eqref{eq:lr:logistic_5} into Eq.~\eqref{eq:lr:logistic_6}, we obtain the explicit form of the logistic regression cross-entropy loss:
\begin{equation}\label{eq:lr:logistic_7}
\mathcal{L}(\bm{\theta})
= -\frac{1}{M}\sum_{m=1}^M \left(
    y_m \log \frac{1}{1+\exp(-\bm{\theta}^{\top}{\bm{x}}_m)}
    + (1-y_m)\log \frac{\exp(-\bm{\theta}^{\top}{\bm{x}}_m)}
           {1+\exp(-\bm{\theta}^{\top}\tilde{\bm{x}}_m)}
\right).
\end{equation}

\paratitle{Parameter Estimation.} The optimal model parameters are obtained by minimizing the cross-entropy loss:
\[
\bm{\theta}^{*} = \arg\min_{\bm{\theta}} \mathcal{L}(\bm{\theta}).
\]
The parameters of logistic regression are typically estimated using iterative
optimization methods, most commonly gradient descent. The detailed procedures
and variants of gradient-based optimization will be introduced in the next
section.

\subsection{Parameter Estimation in Logistic Regression}

\paratitle{Gradient Descent Solution.}
The standard way to estimate the parameters of logistic regression is to compute the gradient of the cross-entropy loss in Eq.~\eqref{eq:lr:logistic_6} and then apply gradient descent. Since Section~\ref{know:lr:gd} has already introduced the mechanics of gradient descent, we focus here on the detailed derivation of the gradient of the logistic regression loss.

To begin, we simplify the per-sample cross-entropy loss term. Let $\mathcal{L}_m(\bm{\theta})$ be the contribution from the $m$-th sample. Expanding the sigmoid function and performing step-by-step algebra gives:
\begin{equation}
\begin{split}
\mathcal{L}_m(\bm{\theta})
&= -\Big( y_m \log \sigma(\bm{\theta}^{\top}{\bm{x}}_m)
    + (1-y_m)\log\!\big(1-\sigma(\bm{\theta}^{\top}{\bm{x}}_m)\big) \Big) \\[5pt]
&= -y_m \log\!\left( \frac{1}{1+\exp(-\bm{\theta}^{\top}\bm{x}_m)} \right)
   -(1-y_m)\log\!\left( 1 - \frac{1}{1+\exp(-\bm{\theta}^{\top}\bm{x}_m)} \right) \\[5pt]
&= -y_m \log\!\left( \frac{\exp(\bm{\theta}^{\top}\bm{x}_m)}
                          {1+\exp(\bm{\theta}^{\top}\bm{x}_m)} \right)
   -(1-y_m)\log\!\left( \frac{1}{1+\exp(\bm{\theta}^{\top}\bm{x}_m)} \right) \\[5pt]
&= -y_m\,\bm{\theta}^{\top}\bm{x}_m
    + \log\!\big( 1 + \exp(\bm{\theta}^{\top}\bm{x}_m) \big).
\end{split}
\end{equation}
The derivative with respect to $\bm{\theta}$ can now be computed directly:
\begin{equation}
\frac{\partial \mathcal{L}_m(\bm{\theta})}{\partial \bm{\theta}}
= - y_m \bm{x}_m
  + \frac{\exp(\bm{\theta}^{\top}\bm{x}_m)}
         {1+\exp(\bm{\theta}^{\top}\bm{x}_m)}\, \bm{x}_m
= \big( \sigma(\bm{\theta}^{\top}\bm{x}_m) - y_m \big)\bm{x}_m.
\end{equation}

Averaging over all samples yields the gradient of the full loss:
\begin{equation}\label{eq:gradient_Logistic}
\frac{\partial \mathcal{L}(\bm{\theta})}{\partial \bm{\theta}}
= \frac{1}{M} \sum_{m=1}^{M}
  \big( \sigma(\bm{\theta}^{\top}\bm{x}_m) - y_m \big)\bm{x}_m.
\end{equation}
Substituting this result into the gradient descent update rule gives the iterative parameter solution.

\paratitle{Closed-form Solution.}
In general, logistic regression does not admit a closed-form solution when parameters are learned by minimizing the cross-entropy loss. However, if we reinterpret logistic regression from a generative modeling viewpoint, assuming that the class-conditional densities follow multivariate Gaussian distributions with a shared covariance matrix, then a closed-form analytical solution can be obtained.

Let $\Pr(y=1\mid\bm{x})$ and $\Pr(y=0\mid\bm{x})$ denote the two class probabilities. Logistic regression assumes that the log-odds takes a linear form:
\begin{equation}\label{eq:lr:log_wx_b}
\log\frac{\Pr(y=1\mid\bm{x})}{\Pr(y=0\mid\bm{x})}
= \bm{\theta}^{\top}\bm{x}
= \bm{w}^{\top}\bm{x} + b.
\end{equation}
Applying Bayes' rule gives:
\begin{equation}\label{eq:lr:log_logreg}
\log\frac{\Pr(y=1\mid\bm{x})}{\Pr(y=0\mid\bm{x})}
= \log\frac{\Pr(\bm{x}\mid y=1)\Pr(y=1)}
         {\Pr(\bm{x}\mid y=0)\Pr(y=0)}.
\end{equation}
We now assume the following Gaussian generative model:
\[
\Pr(\bm{x}\mid y=1)=\mathcal{N}(\bm{\mu}_1,\bm{\Sigma}), \qquad
\Pr(\bm{x}\mid y=0)=\mathcal{N}(\bm{\mu}_0,\bm{\Sigma}),
\]
where $\bm{\mu}_1$ and $\bm{\mu}_0$ are the class means, and $\bm{\Sigma}$ is the shared covariance matrix. The density of $\mathcal{N}(\bm{\mu},\bm{\Sigma})$ is:
\[
\Pr(\bm{x})
= \frac{1}{\sqrt{(2\pi)^n|\bm{\Sigma}|}}
  \exp\!\left( -\frac{1}{2}
               (\bm{x}-\bm{\mu})^{\top}\bm{\Sigma}^{-1}(\bm{x}-\bm{\mu}) \right).
\]

Substituting the Gaussian densities into Eq.~\eqref{eq:lr:log_logreg} and expanding quadratic terms yields\footnote{
The derivation of Eq.~\eqref{eq:lr:logreg_ana} is as follows:
\begin{equation}
\begin{split}
&\log\frac{\Pr(y=1\mid\bm{x})}{\Pr(y=0\mid\bm{x})} \\[2pt]
&= \log\frac{
 \frac{1}{\sqrt{(2\pi)^n|\bm{\Sigma}|}}
 \exp\!\left( -\frac{1}{2}(\bm{x}-\bm{\mu}_1)^{\top}\bm{\Sigma}^{-1}(\bm{x}-\bm{\mu}_1) \right) \Pr(y=1)
}{
 \frac{1}{\sqrt{(2\pi)^n|\bm{\Sigma}|}}
 \exp\!\left( -\frac{1}{2}(\bm{x}-\bm{\mu}_0)^{\top}\bm{\Sigma}^{-1}(\bm{x}-\bm{\mu}_0) \right) \Pr(y=0)
} \\[5pt]
&= \log\frac{
 \exp\!\left( -\frac{1}{2}(\bm{x}-\bm{\mu}_1)^{\top}\bm{\Sigma}^{-1}(\bm{x}-\bm{\mu}_1) \right)
}{
 \exp\!\left( -\frac{1}{2}(\bm{x}-\bm{\mu}_0)^{\top}\bm{\Sigma}^{-1}(\bm{x}-\bm{\mu}_0) \right)
}
+ \log\frac{\Pr(y=1)}{\Pr(y=0)} \\[5pt]
&=
\bm{x}^{\top}\bm{\Sigma}^{-1}(\bm{\mu}_1 - \bm{\mu}_0)
+ \frac{1}{2}\!\left( \bm{\mu}_0^{\top}\bm{\Sigma}^{-1}\bm{\mu}_0
                   - \bm{\mu}_1^{\top}\bm{\Sigma}^{-1}\bm{\mu}_1 \right)
+ \log\frac{\Pr(y=1)}{\Pr(y=0)} .
\end{split}
\end{equation}
}
\begin{equation}\label{eq:lr:logreg_ana}
\begin{split}
\log \frac{\Pr(y=1\mid \bm{x})}{\Pr(y=0\mid \bm{x})}
&= \bm{x}^{\top}\bm{\Sigma}^{-1}(\bm{\mu}_1 - \bm{\mu}_0) \\[3pt]
&\quad + \frac{1}{2}\!\left(
    \bm{\mu}_0^\top \bm{\Sigma}^{-1}\bm{\mu}_0
  - \bm{\mu}_1^\top \bm{\Sigma}^{-1}\bm{\mu}_1
  \right)
  + \log\frac{\Pr(y=1)}{\Pr(y=0)}.
\end{split}
\end{equation}
Comparing Eq.~\eqref{eq:lr:logreg_ana} with the logistic form in Eq.~\eqref{eq:lr:log_wx_b}, and using the equality $\bm{w}^{\top}\bm{x}=\bm{x}^{\top}\bm{w}$, we obtain:
\begin{equation}
\begin{aligned}
\bm{w} &= \bm{\Sigma}^{-1}(\bm{\mu}_1 - \bm{\mu}_0), \\[3pt]
b &= \frac{1}{2}\!\left(
      \bm{\mu}_0^\top \bm{\Sigma}^{-1}\bm{\mu}_0
    - \bm{\mu}_1^\top \bm{\Sigma}^{-1}\bm{\mu}_1 \right)
    + \log\frac{\Pr(y=1)}{\Pr(y=0)} .
\end{aligned}
\end{equation}

Thus, under the Gaussian assumption, logistic regression admits a closed-form analytical solution. The class priors $\Pr(y=1)$ and $\Pr(y=0)$ can be estimated by empirical frequencies, and the shared covariance matrix $\bm{\Sigma}$ is typically obtained by pooling the class-specific covariance estimates.

\subsection{Extended Discussion on Cross-entropy Loss}~\label{know:lr:cross_entropy}

Perhaps readers may wonder why logistic regression does not simply adopt the
sum of squared residuals between $y_m$ and $\hat{y}_m$ as its loss function,
namely $\mathcal{L}(\boldsymbol{\theta}) = \frac{1}{M}\sum_{m=1}^{M}(\hat{y}_m - y_m)^2$. In the following, we provide several extended discussions to explain why the mean squared error is not an appropriate choice for logistic regression.

\paratitle{Comparison Between Cross-Entropy and Mean Squared Error.} In practice, it is widely observed that MSE is less sensitive than cross-entropy for logistic regression. We illustrate this with a simple example. Suppose a binary classification dataset contains a sample whose true label is $1$, but the logistic regression model predicts $\hat{y}=0$. The two losses are:
\begin{equation}
    \begin{cases}
        \mathrm{MSE}: & (1-0)^2 = 1,\\[4pt]
        \mathrm{Xent}: & -(1\cdot\log(0)+0\cdot\log(1)) \to \infty,
    \end{cases}
\end{equation}
where Xent denotes the cross-entropy loss. We see that when the prediction and
the true label differ completely, MSE assigns a relatively small penalty
(loss = 1), while cross-entropy assigns an extremely large penalty
(loss $\to \infty$). In linear regression, the prediction can take any real
value; in logistic regression, $\hat{y}$ is restricted to the interval
$(0,1)$. Under this constraint, cross-entropy effectively amplifies the penalty
for incorrect predictions, making the model more sensitive to classification
errors.

Moreover, for logistic regression, the MSE loss is \emph{non-convex}, which
creates significant difficulties for gradient-based optimization. A strictly
positive second derivative ensures convexity. The second derivative of the MSE
loss for logistic regression is
\begin{equation}\label{eq:lr:hw1}
    h_{\mathrm{MSE}}(\hat{y})
    = -2\!\left(3\hat{y}^2 - 2\hat{y}(y+1) + y\right)x^2,
\end{equation}
where $\hat{y} = \frac{1}{1+\exp(-wx-b)}$. The expression
$h_{\mathrm{MSE}}(\hat{y})$ can be negative under certain conditions, implying
non-convexity.

In contrast, the second derivative of the cross-entropy loss is
\begin{equation}\label{eq:lr:hw2}
    h_{\mathrm{Xent}}(\hat{y})
    = \frac{x^2 \exp(-wx-b)}
           {\left(1+\exp(-wx-b)\right)^2} > 0,
\end{equation}
which guarantees that the loss function is convex. As a result, gradient
descent can always converge to the global optimum.

\paratitle{Information-theoretic Interpretation of Cross-Entropy.} The cross-entropy loss is closely related to the concept of entropy in information theory. Here we explain this intrinsic connection.

Given a binary label $y_m$ and its prediction $\hat{y}_m$, we construct two
Bernoulli distributions $Q$ and $Q'$ using $y_m$ and $\hat{y}_m$ as their
parameters:
\begin{itemize}
    \item $Q$ is the Bernoulli distribution $\mathrm{Bernoulli}(y_m)$, where a
    random variable $z \sim Q$ satisfies
    $\Pr_Q(z=1)=y_m$ and $\Pr_Q(z=0)=1-y_m$;
    \item $Q'$ is the Bernoulli distribution $\mathrm{Bernoulli}(\hat{y}_m)$,
    where a random variable $z \sim Q'$ satisfies
    $\Pr_{Q'}(z=1)=\hat{y}_m$ and $\Pr_{Q'}(z=0)=1-\hat{y}_m$.
\end{itemize}

To measure how different the two distributions $Q$ and $Q'$ are, we use the
\textbf{relative entropy}, also known as the
\textbf{Kullback-Leibler (KL) divergence}. The KL divergence between $Q$ and
$Q'$ is:

\begin{equation}\label{eq:lr:k14_6}
\begin{aligned}
D_{\mathrm{KL}}(Q \| Q')
&= \sum_{z_i=0}^{1}
    \Pr_Q(z_i)
    \log\!\left(
        \frac{\Pr_Q(z_i)}{\Pr_{Q'}(z_i)}
    \right) \\
&= \sum_{z_i=0}^{1} \Pr_Q(z_i)\log \Pr_Q(z_i)
 - \sum_{z_i=0}^{1} \Pr_Q(z_i)\log \Pr_{Q'}(z_i) \\
&= -H(Q)
   -\Big(
        (1-y_m)\log(1-\hat{y}_m)
        + y_m \log(\hat{y}_m)
     \Big),
\end{aligned}
\end{equation}
where $H(Q)$ is the entropy of distribution $Q$, which is a constant independent
of $\hat{y}_m$. Importantly, the bracketed term in the last line of
Eq.~\eqref{eq:lr:k14_6} is exactly the per-sample cross-entropy loss used in
logistic regression.

Since KL divergence quantifies the discrepancy between two probability
distributions\footnote{
KL divergence is sometimes called ``KL distance,'' but strictly speaking it is
not a metric because $D_{\mathrm{KL}}(P\|Q)\neq D_{\mathrm{KL}}(Q\|P)$ in
general, so it does not satisfy the symmetry property required of a distance.
It measures only dissimilarity.
}, minimizing the cross-entropy loss is equivalent to making the predictive
distribution $Q'$ as close as possible to the true distribution $Q$.
Thus, from an information-theoretic perspective, logistic regression learns
parameters that minimize the divergence between the true label distribution and
the model's predicted distribution.

\subsection{The Origin and Naming of Logistic Regression}

\paratitle{Discrete Choice Models and the Logit Model.}
In addition to the name \emph{Logistic Regression}, the model is also known as
\emph{Logit Regression}. This terminology originates from an important class of
models in economics known as \textbf{Discrete Choice Models (DCMs)}.
Historically, the logit model was first proposed as a specific form of DCM,
rather than as a statistical classification model.

As the name suggests, discrete choice models are designed to capture decision
making in a discrete set of alternatives. Such decision problems arise widely
in everyday economic activities, including shopping, voting, and general
decision-making scenarios (e.g., choosing one or several products from a set of
available items).

To illustrate the idea, consider a shopping example. Suppose that whether a
consumer purchases a product depends on the product's perceived usefulness or
value. This value is characterized by a real-valued quantity $U$, called
\textbf{utility}. Consumers form their assessment of utility based on certain
observable product attributes. Assume the consumer has collected $N$ such
attributes, denoted by $x_1, x_2, \ldots, x_N$, and let
$\bm{x} = (1, x_1, x_2, \ldots, x_N)$ represent the feature vector. A DCM models
the relationship between utility and consumer information using a linear
functional form:
\begin{equation}
    U = \bm{\theta}^\top \bm{x} + \epsilon,
\end{equation}
where $\bm{\theta} = (\theta_0, \ldots, \theta_N)$ is the parameter vector and $\epsilon$ is a random disturbance term. Since utility is influenced by many unobserved factors, $\epsilon$ captures these unknown components.

Under the DCM framework, a consumer purchases the product if and only if the
utility is positive; otherwise, the product is not purchased. Thus,
\begin{equation}\label{eq:lr:r2_1}
\begin{aligned}
    \Pr(y = 1 \mid \bm{x})
    &= \Pr(U > 0 \mid \bm{x}) \\
    &= \Pr(\bm{\theta}^\top {\bm{x}} > -\epsilon \mid \bm{x}),
\end{aligned}
\end{equation}
where the purchase probability $\Pr(y=1 \mid \bm{x})$ is determined by the
cumulative distribution function (CDF) of the random variable $-\epsilon$.
Therefore, the choice probability depends on the distribution of $U$, which in
turn is governed by the distribution of $\epsilon$.

A natural assumption is to let $\epsilon$ follow a normal distribution. This
leads to the well-known \textbf{Probit Model}, an early and widely used DCM that
predates logistic regression. The probit model uses the Gaussian CDF to model
the choice probability and represents one of the earliest formal approaches to
binary choice modeling.

\begin{figure}[ht]
    \centering
    \includegraphics[width=0.5\columnwidth]{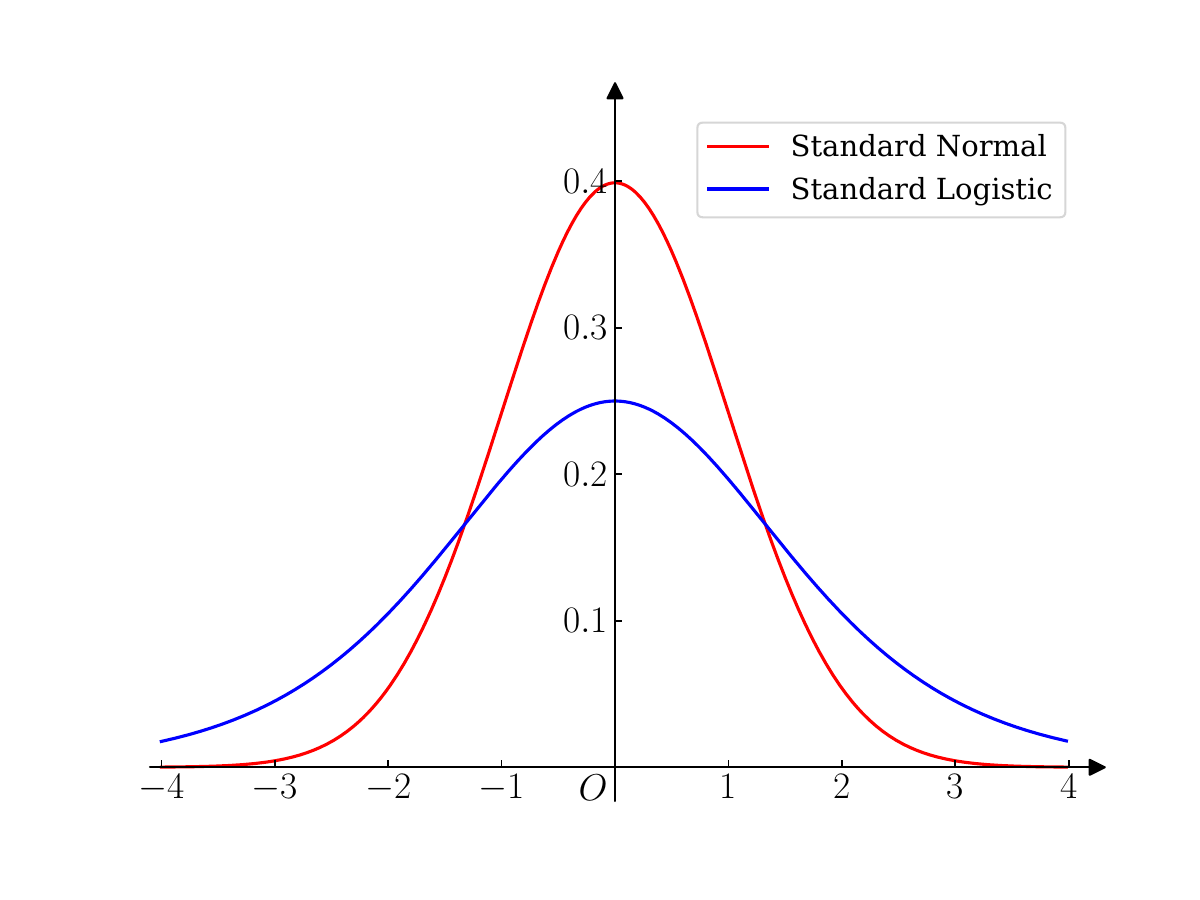}
    \caption{\small Probability density functions of the standard logistic distribution and the standard normal distribution.}
    \label{fig:lr:logit_dist}
\end{figure}

Because the cumulative distribution function (CDF) of the Gaussian distribution does not have a closed-form expression, its use in analytical derivations is inconvenient. To address this issue, researchers proposed replacing the normal CDF with the \textbf{standard logistic distribution}. The CDF of the standard logistic distribution is
\begin{equation}\label{eq:lr:r2_2}
    F(t) = \Pr(x > t) = \frac{1}{1+\exp(-t)}.
\end{equation}
As illustrated in Figure~\ref{fig:lr:logit_dist}, the logistic distribution closely resembles the standard normal distribution - both exhibiting the well-known ``bell-shaped'' form.

Substituting Eq.~\eqref{eq:lr:r2_2} into Eq.~\eqref{eq:lr:r2_1} yields the standard form of the logistic regression model:
\begin{equation}
\begin{aligned}
    \Pr(y=1 \mid \bm{x})
        &= \frac{\exp(\bm{\theta}^{\top} {\bm{x}})}
                {1 + \exp(\bm{\theta}^{\top} {\bm{x}})},\\[4pt]
    \Pr(y=0 \mid \bm{x})
        &= \frac{1}
                {1 + \exp(\bm{\theta}^{\top} {\bm{x}})}.
\end{aligned}
\end{equation}
Correspondingly, the log-odds of choosing the product are given by
\begin{equation}
    \log \frac{\Pr(y=1 \mid \bm{x})}{\Pr(y=0 \mid \bm{x})}
    = \bm{\theta}^{\top} {\bm{x}}.
\end{equation}

Thus, logistic regression can be viewed as a specific instance of a discrete choice model: it models utility as a linear function and models the randomness in utility using the logistic distribution.

Before the development of the logistic model, the Probit model was already widely accepted in economics as a standard DCM. As a result, the inventors of logistic regression adopted the name \emph{Logit} in analogy with \emph{Probit}. At the same time, the term ``Logit'' is also commonly interpreted as an abbreviation of ``the log of an odds.'' Therefore, the term Logit is a clever double entendre: it echoes the structure of the Probit model while also highlighting the connection to the logistic function that underlies the model.

\paratitle{Origin of the Term \emph{Logistic}.}
We know that the name ``logistic regression'' comes from the \emph{logistic
function}. But where does the word ``logistic'' itself originate? Although the
word resembles ``logistics'' (meaning transportation, supply, or military
operations), the two terms have no historical or semantic connection.

\begin{figure}[t]
    \centering
    \includegraphics[width=0.7\columnwidth]{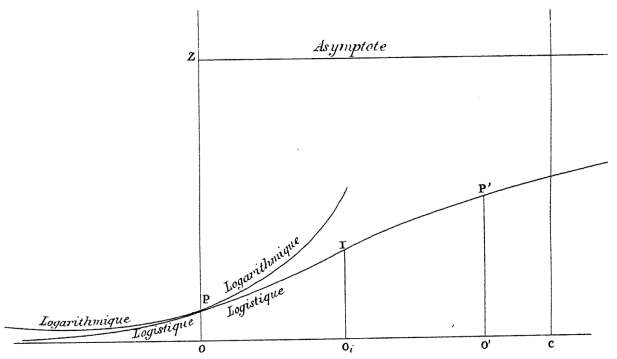}
    \caption{\small The ``Logarithm'' curve and ``Logarithmique'' curve in Verhulst's work~\cite{verhulst1845resherches}.}
    \label{fig:logarithm_and_logarithmique}
\end{figure}

The term \emph{logistic} was first introduced by the Belgian mathematician
Pierre Francois Verhulst (1804--1849). In his 1845 paper on population growth,
Verhulst plotted two curves, shown in Figure~\ref{fig:logarithm_and_logarithmique}~\cite{verhulst1845resherches}. He
argued that if left unchecked, population would grow exponentially -- corresponding to the curve labeled \emph{Logarithmique} in French, which at that time referred to an exponential curve (historically, the word ``logarithm'' was used in a way related to exponentiation). However, real-world population growth is constrained by environmental factors and more closely follows the curve labeled \emph{Logistique}, from which the English word \emph{logistic} is derived. Thus, the word \emph{logistic} originally meant ``having an exponential-like shape.'' Indeed, the logistic function can be rewritten as
\begin{equation}
    \sigma(z)
    = \frac{1}{1+\exp(-z)}
    = \frac{\exp(z)}{1+\exp(z)},
\end{equation}
which makes clear that it is essentially a transformation of the exponential
function.

Over time, the meaning of ``logarithm'' in English shifted from ``exponent'' to
its modern meaning of ``log function,'' but the meaning of \emph{logistic}
remained unchanged. Consequently, the two English names for the model,
\emph{logistic regression} and \emph{logit regression}, carry entirely
different meanings: \emph{logistic} refers to the exponential-like shape of the
logistic function, whereas \emph{logit} refers to the \emph{log-odds}
transformation, \ie the logarithm of the odds ratio.

\subsection{Elements of Softmax Regression}

Logistic regression provides an effective solution for binary classification
problems. However, many practical tasks involve more than two categories. To
handle such multi-class settings, we introduce a regression model specifically
designed for multi-class classification: \textbf{Multinomial Logistic
Regression}, more commonly known as \textbf{Softmax Regression}. Because the
term ``Softmax'' is far more widely used, we will adopt it throughout the
following discussion.

\paratitle{Data Preparation.}
Consider a multi-class classification dataset:
\[
    \mathcal{D}
    = \{(\bm{x}_1, y'_1), (\bm{x}_2, y'_2), \ldots, (\bm{x}_M, y'_M)\},
\]
where $\bm{x}_m \in \mathbb{R}^{N}$ and ${y'_m} \in \{1,2,\ldots,K\}$ for
$m=1,2,\ldots,M$. We assume a \emph{mutually exclusive} multi-class setting,
meaning that each sample $\bm{x}_m$ belongs to exactly one of the $K$ classes.

For convenience in later derivations, we transform the label $y'_m$ using
\textbf{one-hot encoding}. Given $y'_m = k$, one-hot encoding constructs a vector $\bm{y}_m \in \mathbb{R}^{K}$ whose $k$-th entry is $1$ and all others are $0$. For example, if $y'_m = 3$, then $\bm{y}_m = (0, 0, 1, 0, \ldots, 0)$. After this transformation, the dataset becomes:
\[
    \mathcal{D}
    = \{(\bm{x}_1, \bm{y}_1), (\bm{x}_2, \bm{y}_2), \ldots, (\bm{x}_M, \bm{y}_M)\},
\]
where each $\bm{y}_m$ is a $K$-dimensional one-hot vector representing the
class membership of sample $\bm{x}_m$.

\paratitle{Regression Function.} In the SoftMax regression model, each input sample $\bm{x}_m$ is transformed through $K$ linear combinations, producing $K$ real-valued scores:
\begin{equation}\label{eq:lr:k18_1}
\begin{aligned}
\left(
\begin{array}{c}
z_{m,1} \\
\vdots \\
z_{m,k} \\
\vdots \\
z_{m,K}
\end{array}
\right)
=
\left(
\begin{array}{c}
\bm{\theta}_1^{\top}{\bm{x}}_m \\
\vdots \\
\bm{\theta}_k^{\top}{\bm{x}}_m \\
\vdots \\
\bm{\theta}_K^{\top}{\bm{x}}_m
\end{array}
\right),
\end{aligned}
\end{equation}
where $\bm{\theta}_k$ is the parameter vector for class $k$, and $z_{m,k}$ is the
corresponding linear score. In compact form, Eq.~\eqref{eq:lr:k18_1} can be
written as
\[
    \bm{z}_m = \Theta^{\top}{\bm{x}}_m,
\]
where
\[
    \Theta = (\bm{\theta}_1, \ldots, \bm{\theta}_K),\qquad
    \bm{z}_m = (z_{m,1}, \ldots, z_{m,K})^{\top}.
\]
Each $z_{m,k}$ represents the score associated with class $k$. A natural way to
predict the class label for $\bm{x}_m$ is simply to select the class with the
largest score:
\begin{equation}\label{eq:lr:k18_2}
    \hat{y}_m'
    = \arg\max_{k}
      \big\{
        z_{m,1}, \ldots, z_{m,K}
      \big\}.
\end{equation}
Here, $\hat{y}_m'$ denotes the predicted class label.

Eq.~\eqref{eq:lr:k18_2} involves the $\arg\max$ operator, which is
non-differentiable and therefore cannot be optimized using gradient-based
methods. To address this issue, we apply the \textbf{SoftMax function} to the
score vector $\bm{z}_m$, producing a normalized, differentiable, and
``soft'' approximation to the one-hot encoded label vector $\bm{y}_m$.
Under the SoftMax function, the predicted value of the $k$-th component of
$\bm{y}_m$ is computed as:
\begin{equation}\label{eq:lr:k18_3}
    \hat{y}_{m,k}
    = \frac{\exp(z_{m,k})}
           {\sum_{j=1}^K \exp(z_{m,j})}
    = \frac{\exp(\bm{\theta}_k^{\top}{\bm{x}}_m)}
           {\sum_{j=1}^K \exp(\bm{\theta}_j^{\top}{\bm{x}}_m)}
    \in (0,1).
\end{equation}

Here, the exponential term $\exp(z_{m,k})=\exp(\bm{\theta}_k^{\top}{\bm{x}}_m)$ performs a monotonic scaling of the raw score $z_{m,k}$, while the denominator
$\sum_{j=1}^K \exp(z_{m,j})$ normalizes the transformed scores across all
classes so that the resulting probabilities sum to one. The vector
\[
    \hat{\bm{y}}_m
    = (\hat{y}_{m,1}, \ldots, \hat{y}_{m,K})^{\top}
\]
thus provides a smooth and differentiable prediction of the one-hot encoded
label vector $\bm{y}_m$.

Combining Eqs.~\eqref{eq:lr:k18_1} and \eqref{eq:lr:k18_3}, the full prediction
function of the SoftMax regression model can be written as
\begin{equation}
    \label{eq:lr:k18_4}
    \hat{\boldsymbol{y}}_m
    =
    \left(
    \begin{array}{c}
        \hat{y}_{m,1} \\
        \vdots \\
        \hat{y}_{m,k} \\
        \vdots \\
        \hat{y}_{m,K}
    \end{array}
    \right)
    =
    \left(
    \begin{array}{c}
        \dfrac{\exp(\bm{\theta}_{1}^{\top}{\bm{x}}_m)}
              {\sum_{j=1}^{K}\exp(\bm{\theta}_{j}^{\top}{\bm{x}}_m)}
        \\
        \vdots \\
        \dfrac{\exp(\bm{\theta}_{k}^{\top}{\bm{x}}_m)}
              {\sum_{j=1}^{K}\exp(\bm{\theta}_{j}^{\top}{\bm{x}}_m)}
        \\
        \vdots \\
        \dfrac{\exp(\bm{\theta}_{K}^{\top}{\bm{x}}_m)}
              {\sum_{j=1}^{K}\exp(\bm{\theta}_{j}^{\top}{\bm{x}}_m)}
    \end{array}
    \right).
\end{equation}
Since each $\hat{y}_{m,k}$ lies in $(0,1)$ and the vector satisfies $\sum_{k=1}^{K}\hat{y}_{m,k}=1$, each $\hat{y}_{m,k}$ can be naturally interpreted as the predicted probability that the $m$-th sample belongs to class $k$:
\[
\begin{aligned}
p(y_m' = 1 \mid \bm{x}_m ; \bm{\Theta}) &= \hat{y}_{m,1},\\
p(y_m' = 2 \mid \bm{x}_m ; \bm{\Theta}) &= \hat{y}_{m,2},\\
&\;\vdots\\
p(y_m' = K \mid \bm{x}_m ; \bm{\Theta}) &= \hat{y}_{m,K}.
\end{aligned}
\]

\begin{figure}[t]
    \centering
    \includegraphics[width=0.75\columnwidth]{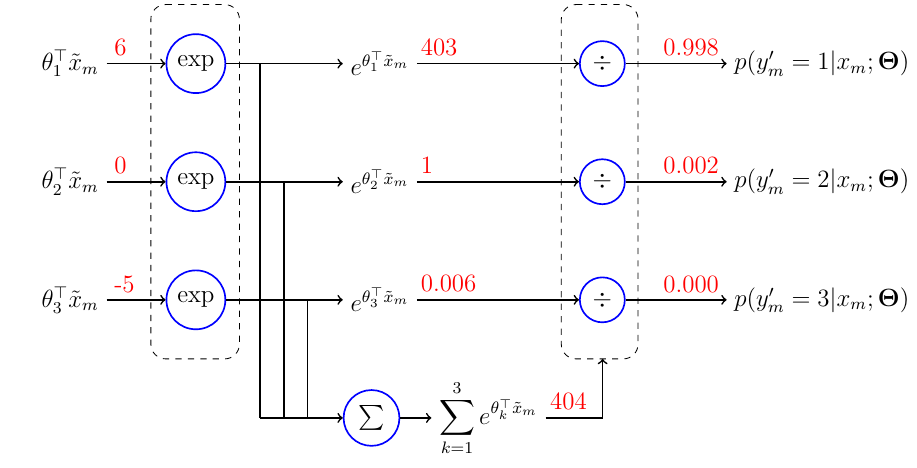}
    \caption{\small Illustration of the Softmax regression model.}
    \label{fig:lr:Softmax}
\end{figure}

Figure~\ref{fig:lr:Softmax} provides an illustrative example of how the Softmax regression model processes input data. As shown in Figure~\ref{fig:lr:Softmax}, the Softmax function applies two transformations to the linear prediction $\boldsymbol{\theta}_k^{\top}\boldsymbol{x}_m$. First, the exponential operation maps the original values to the interval $(0,\infty)$ and amplifies their differences, turning variations within the same order of magnitude into differences across multiple orders of magnitude. Second, the normalization step converts the transformed values into a valid probability distribution. Compared with directly normalizing the linear outputs, the exponential operation further enlarges the gap between the maximum value and the others, making the probability of the most likely class approach 1 after normalization.

In addition, Softmax regression can be regarded as a general form of logistic regression, while logistic regression is a special case of Softmax regression when $K = 2$. This is also why Softmax regression is often referred to as multinomial logistic regression. Interested readers may attempt to derive the Softmax formulation for $K = 2$ to observe how it degenerates into the standard logistic regression model.

\paratitle{Loss Function}. In multi-class classification tasks, we still use the cross-entropy loss to measure the discrepancy between $\bm{y}_m$ and $\hat{\bm{y}}_m$. The loss is defined as
\begin{equation}
    \label{eq:lr:k18_5}
    \mathcal{L}(\boldsymbol{\Theta})
    = -\frac{1}{M} \sum_{m=1}^M
    \left( \sum_{k=1}^K y_{m,k} \log \hat{y}_{m,k}\!\left(\boldsymbol{\theta}_k, \boldsymbol{x}_m\right) \right),
\end{equation}
where $\hat{y}_{m,k}\!\left(\boldsymbol{\theta}_k, \boldsymbol{x}_m\right)$ indicates that $\hat{y}_{m,k}$ is a function of $\boldsymbol{\theta}_k$ and $\bm{x}_m$.

In Eq.~\eqref{eq:lr:k18_5}, each ${y}_{m,k}$ takes only the value 0 or 1. For entries where ${y}_{m,k}=0$, the Softmax prediction does not affect the loss. In other words, the cross-entropy loss cares only about whether the model correctly predicts the single component of the one-hot vector that equals~1. This property is crucial for \textbf{Asymmetric Binary Features}, where only the value~1 carries meaningful information, and correctly predicting zeros provides no practical benefit. For example, for the vectors (0,0,0,0,1) and (1,0,0,0,0), if we focus solely on the position of the value~1, the prediction is entirely incorrect; yet if we also count the zeros, one might mistakenly conclude that 60\% of the elements are predicted correctly.

\paratitle{Origin of the Term ``Softmax''}. The name ``Softmax'' comes from the fact that the function can be viewed as a ``soft'' version of the $\arg\max$ function whose output is represented in one-hot form. If we require the $\arg\max$ function to output a one-hot vector, then for a set of inputs $z_1,\ldots,z_n$ with a unique maximum, its output is
\begin{equation}
    \label{eq:lr:k18_6}
    \arg\max(z_1,\ldots,z_n)
    = (y_1,\ldots,y_n)
    = (0,\ldots,0,1,0,\ldots,0),
\end{equation}
where $y_i = 1$ only when $z_i$ is the largest value, and $y_i = 0$ otherwise. If we regard Eq.~\eqref{eq:lr:k18_6} as the ``hard'' $\arg\max$ function, then the Softmax function can be viewed as its ``soft'' counterpart.

If we further interpret $\bm{y}_{m} = \{ y_{m,1},\ldots,y_{m,K} \}$ as a probability distribution - where exactly one class has probability 1 and all others have probability 0 - then Eq.~\eqref{eq:lr:k18_4} becomes equivalent to computing the KL divergence between the distributions represented by $\bm{y}_{m}$ and $\hat{\bm{y}}_{m}$. Interested readers may consult Section~\ref{know:lr:cross_entropy} and try deriving this equivalence relationship on their own.

\subsection{Parameter Estimation in Softmax Regression}

In this subsection, we present two approaches for solving Softmax regression, namely a gradient-based solution and a closed-form solution.

\paratitle{Gradient-based Solution.} Similar to logistic regression, Softmax regression is typically optimized by computing the gradient of the multi-class cross-entropy loss (Eq.~\eqref{eq:lr:k18_5}) and applying gradient descent. Below, we derive the gradient with respect to each parameter vector~$\bm{\theta}_i$ in a clear and detailed manner.

The loss function of the Softmax regression is
\begin{equation}
    \mathcal{L}(\bm{\Theta}) = -\frac{1}{M}\sum_{m=1}^{M}\sum_{k=1}^{K}
    y_{m,k}\log\frac{\exp(\bm{\theta}_k^{\top}{\bm{x}}_m)}{\sum_{j=1}^{K}\exp(\bm{\theta}_j^{\top}{\bm{x}}_m)}.
\end{equation}
To compute the gradient with respect to $\bm{\theta}_i$, observe that only the terms involving class $i$ depend on $\bm{\theta}_i$. Thus,
\begin{equation}
    \frac{\partial \mathcal{L}(\bm{\Theta})}{\partial \bm{\theta}_i}
    = -\frac{1}{M} \sum_{m=1}^{M}
    \frac{\partial}{\partial \bm{\theta}_i}
    \left(
        y_{m,i}
        \left(
            \bm{\theta}_i^{\top}{\bm{x}}_m
            -
            \log \sum_{j=1}^{K}\exp(\bm{\theta}_j^{\top}{\bm{x}}_m)
        \right)
    \right).
\end{equation}
We now differentiate each part carefully for
\begin{itemize}
  \item {\em Derivative of the linear term:}
\[
\frac{\partial}{\partial \bm{\theta}_i}
\left( y_{m,i}\,\bm{\theta}_i^{\top}{\bm{x}}_m \right)
= y_{m,i}{\bm{x}}_m.
\]
  \item {\em Derivative of the log-sum-exp term:} Let
\[
Z_m = \sum_{j=1}^{K}\exp(\bm{\theta}_j^{\top}{\bm{x}}_m).
\]
Then
\[
\frac{\partial}{\partial \bm{\theta}_i} \log Z_m
= \frac{1}{Z_m}
\cdot
\exp(\bm{\theta}_i^{\top}{\bm{x}}_m) {\bm{x}}_m
= \hat{y}_{m,i}{\bm{x}}_m,
\]
because Softmax predicts
\[
\hat{y}_{m,i}=\frac{\exp(\bm{\theta}_i^{\top}\tilde{\bm{x}}_m)}{\sum_{j=1}^{K}\exp(\bm{\theta}_j^{\top}\tilde{\bm{x}}_m)}.
\]
\end{itemize}
Putting the pieces together:
\begin{equation}
\frac{\partial}{\partial \bm{\theta}_i}
\Big(
    y_{m,i}
    \left(
        \bm{\theta}_i^{\top}{\bm{x}}_m
        -
        \log Z_m
    \right)
\Big)
= y_{m,i}{\bm{x}}_m - y_{m,i}\hat{y}_{m,i}{\bm{x}}_m.
\end{equation}

Since $y_{m,i}$ is 1 only when the true class is $i$, the expression simplifies naturally into the canonical Softmax gradient form:
\begin{equation}~\label{eq:gradient_Softmax}
    \frac{\partial \mathcal{L}(\bm{\Theta})}{\partial \bm{\theta}_i}
    = \frac{1}{M}\sum_{m=1}^{M}
    \left(
        \hat{y}_{m,i} - y_{m,i}
    \right){\bm{x}}_m.
\end{equation}
Thus, the gradient takes the elegant form of {\em prediction minus label}, multiplied by the input feature. Substituting this gradient into gradient descent (or its variants) yields the solution to the Softmax regression model.

\paratitle{A Unified View of Gradients: ``Prediction Minus Label'' Times Input.}
If we carefully compare linear regression (in Eq.~\eqref{eq:lr:bgd}), logistic regression (in Eq.~\eqref{eq:gradient_Logistic}), and Softmax (multinomial) regression (in Eq.~\eqref{eq:lr:bgd}), we will find that their gradient expressions under gradient descent share a common and very simple structure: the gradient is always of the form
\begin{equation}\label{eq:uni_gradient}
  \text{gradient}
\;=\;
\frac{1}{M}\sum_{m}
\bigl(\text{prediction}_m - \text{label}_m\bigr)\times \text{input}_m.
\end{equation}
Table~\ref{tab:unified_gradient} summarizes their prediction forms.

\begin{table}[t]\small
\centering
\caption{\small Unified Gradient Form of Linear Regression, Logistic Regression, and Softmax Regression.}\label{tab:unified_gradient}
\vspace{+0.2cm}
\begin{tabular}{ccc}
\toprule
\textbf{Model} & \textbf{Prediction} $\hat{y}$ or $\hat{y}_{k}$ & \textbf{Gradient} $\frac{\partial \mathcal{L}}{\partial \theta}$ \\
\midrule

{Linear Regression} &
$\displaystyle \hat{y}_m = \bm{\theta}^{\top}\bm{x}_m$ &
$\displaystyle \frac{1}{M}\sum_{m}(\hat{y}_m - y_m)\bm{x}_m$ \\

\midrule

{Logistic Regression} &
$\displaystyle \hat{y}_m = \frac{1}{1+\exp\left(-\bm{\theta}^{\top}\bm{x}_m\right)}$ &
$\displaystyle \frac{1}{M}\sum_{m}(\hat{y}_m - y_m)\bm{x}_m$ \\

\midrule

{Softmax Regression} &
$\displaystyle
\hat{y}_{m,k}
= \frac{\exp(\bm{\theta}_k^{\top}{\bm{x}}_m)}
{\sum_{j=1}^K \exp(\bm{\theta}_j^{\top}{\bm{x}}_m)}$ &
$\displaystyle
\frac{1}{M}\sum_{m}\left(\hat{y}_{m,k} - y_{m,k}\right){\bm{x}}_m$ \\

\bottomrule
\end{tabular}
\end{table}

\textbf{Intuitive interpretation.} This shared ``prediction minus label times input'' pattern is not accidental; it reflects a common principle underlying gradient descent in regression models:
\begin{itemize}
  \item The term $(\text{prediction}_m - \text{label}_m)$ represents the \emph{error} made on sample $m$.
  \item Multiplying by the input $\text{input}_m$ scales the update for each parameter by the value of its corresponding feature: if a feature is zero, that sample contributes no update to that parameter; if a feature has a large magnitude, the update is amplified because changes to that parameter have a stronger influence on the model's output for that sample.
  \item When the prediction is correct and confident (prediction $\approx$ label), the error term is small, producing a near-zero gradient and only minimal parameter updates.
  \item When the prediction deviates significantly from the label, the error term becomes large, leading to a stronger corrective update. Features with larger values induce proportionally larger adjustments, since their associated parameters play a greater role in generating the incorrect output.
\end{itemize}

This ``error times input'' structure originates from the linear (affine) transformation at the core of these models. Although more complex architectures such as deep neural networks incorporate nonlinear activations, each layer still performs a linear transformation whose weights follow the same gradient pattern. Consequently, the principle observed here extends naturally to deeper models: during backpropagation, every layer updates its parameters according to an error signal multiplied by the input to that layer.

\paratitle{Closed-form Solution.}
From the perspective of minimizing the cross-entropy loss, Softmax regression does not admit a direct closed-form solution. However, logistic regression does have a closed-form solution under a probabilistic generative model (Gaussian class-conditional densities with shared covariance). By extending the same generative assumptions to the multi-class setting, we can derive an analogous form for Softmax regression. The key idea is to work with \emph{log-odds} between classes, which removes the normalization term in the Softmax function.

\medskip
\textbf{Step 1: Express Softmax probabilities in log-odds form.}
For the Softmax regression model,
\[
\mathrm{Pr}(y=k\mid \bm{x})
= \frac{\exp(\bm{\theta}_k^\top {\bm{x}})}
{\sum_{i=1}^K \exp(\bm{\theta}_i^\top {\bm{x}})},
\]
consider the ratio between two class probabilities:
\[
\log\frac{\mathrm{Pr}(y=k\mid\bm{x})}{\mathrm{Pr}(y=j\mid\bm{x})}.
\]
Substituting the Softmax formula eliminates the common normalization term:
\begin{equation}\label{eq:lr:Softmax_wx_b}
\begin{split}
\log\frac{\mathrm{Pr}(y=k\mid\bm{x})}{\mathrm{Pr}(y=j\mid\bm{x})}
&= \log \frac{\exp(\bm{\theta}_k^\top{\bm{x}})}
{\exp(\bm{\theta}_j^\top{\bm{x}})} \\
&= (\bm{\theta}_k - \bm{\theta}_j)^\top {\bm{x}} \\
&= \bm{x}^\top(\bm{w}_k - \bm{w}_j) + (b_k - b_j).
\end{split}
\end{equation}
Thus, Softmax regression always expresses \emph{differences in log-probabilities} as linear functions of $\bm{x}$. This also shows why Softmax parameters are identifiable only up to a constant shift: only the differences $(\bm{\theta}_k - \bm{\theta}_j)$ matter.

\medskip
\textbf{Step 2: Derive the same log-odds expression under a generative model.}
Now assume the same probabilistic model used to derive the closed-form solution for logistic regression: 1) each class follows a Gaussian distribution
  $\bm{x} \mid y=k \sim \mathcal{N}(\bm{\mu}_k,\bm{\Sigma})$; 2) all classes share the same covariance $\bm{\Sigma}$; 3) priors are $\mathrm{Pr}(y=k)$.
Under these assumptions, Bayes' rule gives:
\[
\mathrm{Pr}(y=k\mid\bm{x})
\propto \exp\!\left(
 -\frac{1}{2}(\bm{x}-\bm{\mu}_k)^\top\bm{\Sigma}^{-1}(\bm{x}-\bm{\mu}_k)
\right)\mathrm{Pr}(y=k).
\]
Taking the log-odds between two classes $k$ and $j$ yields:
\begin{equation}\label{eq:lr:Softmax_ana}
\log\frac{\mathrm{Pr}(y=k\mid\bm{x})}{\mathrm{Pr}(y=j\mid\bm{x})}
= \bm{x}^\top\bm{\Sigma}^{-1}(\bm{\mu}_k - \bm{\mu}_j)
+ \frac{1}{2}\left(\bm{\mu}_j^\top\bm{\Sigma}^{-1}\bm{\mu}_j
- \bm{\mu}_k^\top\bm{\Sigma}^{-1}\bm{\mu}_k \right)
+ \log\frac{\mathrm{Pr}(y=k)}{\mathrm{Pr}(y=j)}.
\end{equation}
This expression is linear in $\bm{x}$, matching the structural form obtained from Softmax regression in Eq.~\eqref{eq:lr:Softmax_wx_b}.

\medskip
\textbf{Step 3: Match coefficients to obtain a valid parameter solution.}
By comparing Eq.~\eqref{eq:lr:Softmax_ana} and Eq.~\eqref{eq:lr:Softmax_wx_b}, we can directly identify one set of parameters that satisfies the equality:
\begin{equation}
\begin{cases}
\bm{w}_k = \bm{\Sigma}^{-1} \bm{\mu}_k, \\[4pt]
b_k = \log \mathrm{Pr}(y=k)
      - \dfrac{1}{2}\bm{\mu}_k^\top\bm{\Sigma}^{-1}\bm{\mu}_k.
\end{cases}
\end{equation}

\medskip
\textbf{Interpretation.}
This solution is not unique because Softmax parameters are defined only up to an additive constant (i.e., adding the same vector to all $\bm{\theta}_k$ leaves the model unchanged). Nevertheless, the above set provides a consistent closed-form solution under the Gaussian generative model, directly extending the logistic regression case to the multi-class setting.

\newpage

\section{Nonlinear Regression and Neural Networks}\label{sec:lr:nonlinear}

In the regression methods introduced in the preceding sections (such as linear regression and logistic regression), we implicitly assumed that the relationship between the input variables and the output variable is linear. For example, linear regression models the output as a weighted sum of the inputs, and logistic regression relies on a linear decision boundary. However, linearity is a very strong assumption. In real-world data generation processes, the relationship between variables is often highly nonlinear. When such nonlinear structures exist in the data, linear regression models fail to capture the underlying patterns and therefore perform poorly.

To address this limitation, this chapter introduces a natural extension of linear regression: nonlinear regression via \emph{linear basis function models}. The key idea is to apply nonlinear transformations to the input variables through a set of basis functions, enabling the model to capture complex nonlinear relationships while retaining the simplicity and interpretability of linear parametric forms. This approach also serves as an essential stepping stone toward more expressive models such as neural networks.

\subsection{Polynomial Regression}\label{know:lr:poly_regress}

We begin by introducing one of the simplest forms of nonlinear regression: \emph{polynomial regression}. The key idea is to project the original input variable $x$ into a higher-dimensional polynomial feature space, thereby enabling the model to capture nonlinear relationships between $x$ and $y$.

\paratitle{Regression Function.} Given a dataset $\mathcal{D}=\{(x_1, y_1), (x_2, y_2), \ldots, (x_M, y_M)\}$, where both $x_m \in \mathbb{R}$ and $y_m \in \mathbb{R}$ are real-valued, polynomial regression models the output as a linear combination of polynomial terms of the input. For a sample $(x_m, y_m)$, the regression function is defined as
\begin{equation}\label{eq:lr:k23_1}
    \hat{y}_m(x_m, \boldsymbol{\theta})
    = \theta_0 + \theta_1 x_m + \theta_2 x_m^2 + \cdots + \theta_K x_m^K
    = \sum_{k=0}^{K} \theta_k x_m^k.
\end{equation}
Here, $\boldsymbol{\theta} = (\theta_0, \theta_1, \ldots, \theta_K)$ denotes the parameters of the regression function.

The polynomial regression model can also be extended to classification tasks, leading to the \emph{polynomial logistic regression} model. Specifically, given a dataset $\mathcal{D}=\{(x_1, y_1), (x_2, y_2), \ldots, (x_M, y_M)\}$, where $x_m \in \mathbb{R}$ is real-valued and $y_m \in \{0,1\}$ is a binary classification label, the polynomial logistic regression model is defined as
\begin{equation}\label{eq:lr:k23_1_2}
    \hat{y}_m(x_m, \boldsymbol{\theta})
    = \sigma\!\left( \sum_{k=0}^{K} \theta_k x_m^k \right),
\end{equation}
where $\sigma(\cdot)$ is the sigmoid function and $\boldsymbol{\theta} = (\theta_0, \ldots, \theta_K)$ is the parameter vector associated with the polynomial features.

Similarly, polynomial regression can be further extended to multi-class classification tasks, giving rise to the \emph{polynomial Softmax regression} model, whose formulation follows directly from applying the Softmax function to polynomial features. We do not elaborate on this extension here.

\paratitle{Loss Function.} In polynomial regression, we continue to use the sum of squared residuals as the loss function, namely
\[
\mathcal{L}(\boldsymbol{\theta})
= \sum_{m=1}^{M} \left( y_m - \hat{y}_m \right)^{2}.
\]
The model parameters can be estimated in the same way as in ordinary linear regression, either by applying the least squares method or by using gradient descent. For polynomial logistic regression, the model is trained by minimizing the binary cross-entropy loss. the loss function over the dataset is
\[
\mathcal{L}(\boldsymbol{\theta})
= - \sum_{m=1}^{M} \Big(
      y_m \log \hat{y}_m
    + (1 - y_m) \log (1 - \hat{y}_m)
    \Big).
\]
The parameters $\boldsymbol{\theta}$ can be optimized using gradient descent, as in standard logistic regression.

\paratitle{Parameter Estimation.} For parameter estimation, we may regard the polynomial features as a vector
\[
\boldsymbol{x}_m = (x_m^{0}, x_m^{1}, \ldots, x_m^{K})^{\top},
\]
such that the polynomial regression and polynomial logistic regression models in
Eq.~\eqref{eq:lr:k23_1} and Eq.~\eqref{eq:lr:k23_1_2} can both be viewed as standard linear models applied to the transformed input $\boldsymbol{x}_m$.
Consequently, the model parameters can be estimated using either the traditional least squares method (for polynomial regression) or gradient descent (for both regression and logistic regression). It is straightforward to see that the gradient in gradient-based optimization still satisfies the general form $\text{gradient} = \bigl( \text{prediction}_m - \text{label}_m \bigr) \times \text{input}_m$, that is,
\[
\frac{\partial \mathcal{L}}{\partial \theta_k}
= \left( \sum_{k=0}^{K} \theta_k x_m^k - y_m \right) \, x_m^{k},
\quad k = 0,\ldots,K,
\]
for polynomial linear regression, and
\[
\frac{\partial \mathcal{L}}{\partial \theta_k}
= \left[ \sigma\left(\sum_{k=0}^{K} \theta_k x_m^k\right) - y_m \right] \, x_m^{k},
\quad k = 0,\ldots,K,
\]
for polynomial logistic regression.

\paratitle{Theoretical Foundation of Polynomial Regression.}
Although polynomial regression appears simple, its approximation capability is remarkably powerful due to its solid theoretical foundation. In theory, polynomial regression can approximate any continuous functional relationship $y = f(x)$ to an arbitrary degree of accuracy.

Using the \textbf{Maclaurin series} expansion, any sufficiently smooth function $f(x)$ can be written as
\begin{equation}
    f(x)
    = f(0) + \frac{f'(0)}{1!}x + \frac{f''(0)}{2!}x^{2}
      + \ldots + \frac{f^{(n)}(0)}{n!}x^{n} + \cdots
    = \sum_{n=0}^{\infty} \frac{f^{(n)}(0)}{n!} x^{n}.
\end{equation}
From this perspective, polynomial regression can be interpreted as a data-driven variant of the Maclaurin expansion, where the model attempts to infer parameters
\[
\theta_n \approx \frac{f^{(n)}(0)}{n!}
\]
based on the empirical dataset $\mathcal{D}$. Therefore, polynomial regression is, in principle, capable of approximating any continuous nonlinear function, provided that a sufficiently high polynomial order is used.

The strong approximation ability of polynomial models is further supported by a fundamental result in mathematical analysis known as the {\em Weierstrass Approximation Theorem}. This theorem provides a rigorous guarantee that polynomial functions are sufficiently expressive to approximate a broad class of continuous functions.
\begin{theorem}[Weierstrass Approximation Theorem]
Let $f(x)$ be a continuous function defined on a closed and bounded interval $[a,b]$.
Then, for any $\varepsilon > 0$, there exists a polynomial function $p(x)$ such that
\[
\sup_{x \in [a,b]} | f(x) - p(x) | < \varepsilon.
\]
In other words, polynomial functions are dense in the space of continuous functions on $[a,b]$.
\end{theorem}
This theorem implies that no matter how complex or nonlinear the true functional relationship $y = f(x)$ is, it can always be approximated arbitrarily well by a polynomial of sufficiently high degree.

From the perspective of machine learning, this result provides a theoretical justification for using polynomial regression: polynomials form a universal approximator for continuous functions, and hence, with enough data and a high-enough polynomial degree, polynomial regression has the capacity to approximate any continuous nonlinear relationship.

\subsection{Linear Basis Function Models}

Although polynomial regression has strong approximation capabilities, it is fundamentally limited to modeling mappings from $\mathbb{R} \rightarrow \mathbb{R}$. To model more general mappings from $\mathbb{R}^N \rightarrow \mathbb{R}$, we require more expressive modeling tools.

From the theoretical analysis in polynomial regression, we have seen that polynomial regression works by identifying a set of ``basis functions'' (namely $x^k,\; k=1,\ldots,K$) and expressing the target function as a linear combination of these basis functions. Under this perspective, linear regression, logistic regression, and Softmax regression can all be viewed within a unified framework in which the model computes a linear combination of basis functions and then maps it to a regression output, a binary classification output, or a multi-class output. In mathematics, representing a complex function as a linear combination of a set of simpler basis functions is a well-established idea known as \emph{functional approximation}. A large body of classical results -- such as Fourier series, wavelet bases, and orthogonal polynomial systems (e.g., Legendre polynomials~\cite{legendre1785recherches}, Chebyshev polynomials~\cite{chebyshev1853theorie}) -- provide different families of basis functions with powerful approximation properties. These theoretical foundations offer rich tools for constructing diverse regression models capable of capturing various nonlinear processes.

Following this line of thought, researchers have proposed the general framework of \emph{Linear Basis Function Models}, which extend the idea of polynomial regression by allowing arbitrary basis functions to be used in place of polynomial terms.

\paratitle{Regression Function.} Given a dataset $\mathcal{D} = \{(\boldsymbol{x}_1, y_1), (\boldsymbol{x}_2, y_2), \ldots, (\boldsymbol{x}_M, y_M)\}$, $\bm{x}_m \in \mathbb{R}^{N}$ and $y_m \in \mathbb{R}$, the regression function of a linear basis function model is defined as
\begin{equation}
    \label{eq:lr:k24_1}
    \hat{y}_m
    = \theta_0 + \sum_{k=1}^{N} \theta_k \, \phi_k(\boldsymbol{x}_m)
    = \boldsymbol{\theta}^{\top} \boldsymbol{\phi}(\boldsymbol{x}_m).
\end{equation}
Here, each $\phi_k(\cdot)$ is a nonlinear transformation of the input, referred to as a \textbf{basis function}. Its role is to project the original input $\boldsymbol{x}_m$ into a nonlinear feature space. After this transformation, the original input vector becomes
\[
\boldsymbol{\phi}(\boldsymbol{x}_m)
= \Bigl( \phi_0, \phi_1(\boldsymbol{x}_md), \ldots, \phi_K(\boldsymbol{x}_m) \Bigr)^{\top},
\qquad \text{with } \phi_0 = 1,
\]
and the model output is obtained by forming a linear combination of these basis functions using the parameter vector
\[
\boldsymbol{\theta} = (\theta_0, \theta_1, \ldots, \theta_K)^{\top}.
\]

\begin{figure}[t]
    \centering
    \includegraphics[width=0.7\columnwidth]{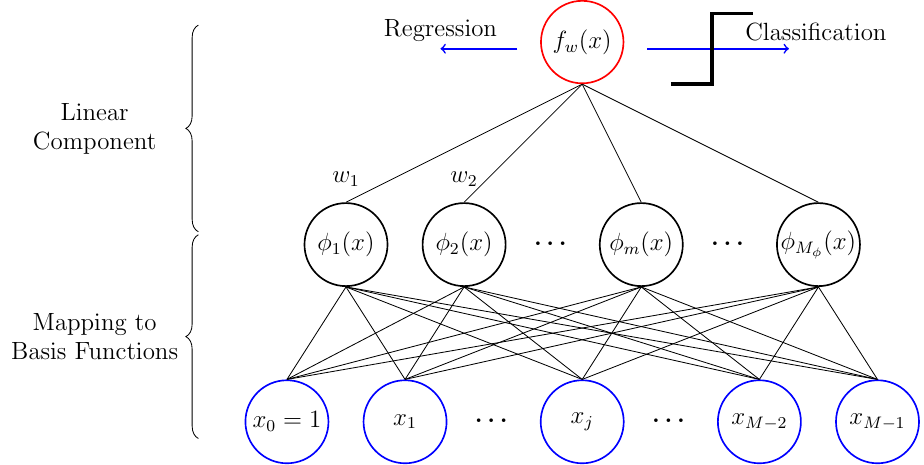}
    \caption{\small Linear basis function model.}
    \label{fig:lr:linear_basis}
\end{figure}

For classification problems, where the labels satisfy $y_m \in \{0,1\}$ in the dataset $\mathcal{D}$, the linear basis function model is combined with the logistic (sigmoid) function to produce probabilistic outputs. In this case, the model becomes
\begin{equation}
    \hat{y}_m = \sigma\!\left( \boldsymbol{\theta}^{\top} \boldsymbol{\phi}(\boldsymbol{x}_m) \right),
\end{equation}
where $\sigma(\cdot)$ denotes the sigmoid function and
$\boldsymbol{\phi}(\boldsymbol{x}_m)$ is the vector of basis functions applied to the input.

Figure~\ref{fig:lr:linear_basis} illustrates how the variables combine within a linear basis function model. Moreover, it is evident that polynomial regression is a special case of the linear basis function model, where the basis functions are given by $x_m \in \mathbb{R}$ $\phi_n({x}) = x^n$.

\paratitle{Loss Function.}
For regression tasks, the linear basis function model uses the sum of squared residuals as its loss function:
\begin{equation}\label{eq:loss_lbfm}
    \mathcal{L}(\boldsymbol{\theta})
    = \sum_{m=1}^{M}
      \left( \boldsymbol{\theta}^{\top}\boldsymbol{\phi}(\boldsymbol{x}_m) - y_m \right)^{2}.
\end{equation}
For classification tasks, the corresponding cross-entropy loss is
\begin{equation}
\mathcal{L}(\boldsymbol{\theta})
= - \sum_{m=1}^{M} \Big(
      y_m \log \sigma\!\left(\boldsymbol{\theta}^{\top}\boldsymbol{\phi}(\boldsymbol{x}_m)\right)
    + (1 - y_m)\log\!\left( 1 - \sigma\!\left(\boldsymbol{\theta}^{\top}\boldsymbol{\phi}(\boldsymbol{x}_m)\right) \right)
    \Big).
\end{equation}
These formulations are completely consistent with the loss functions used in other regression and classification models. They are restated here so that this section remains self-contained and can be understood without consulting previous sections.

\paratitle{Parameter Estimation: Gradient Descent.} The most general approach for estimating the parameters of a linear basis function model is still gradient descent. The procedure is exactly the same as in linear regression or logistic regression. The only modification is to replace the original input vector $\boldsymbol{x}_m$ with the basis-function-expanded input $\boldsymbol{\phi}(\boldsymbol{x}_m)
= \bigl( \phi_0, \phi_1(\boldsymbol{x}_m), \ldots, \phi_K(\boldsymbol{x}_m) \bigr)^{\top}$, For both regression and classification, the gradient of the loss function with respect to each parameter $\theta_k$ takes the general form
\[
\frac{\partial \mathcal{L}}{\partial \theta_k}
= \sum_{m=1}^{M} \bigl( \hat{y}_m - y_m \bigr)\, \phi_k(\boldsymbol{x}_m),
\]
which matches the familiar pattern ``$\text{gradient} = (\text{prediction} - \text{label}) \times \text{input}$''. This unified gradient expression shows that linear basis function models preserve the simplicity of linear models while gaining much stronger expressive power through nonlinear basis functions.

\paratitle{Parameter Estimation: Closed-form Solution.}
In practice, the closed-form solution is typically used only for regression tasks within the linear basis function model framework. Given the loss function in Eq.~\eqref{eq:loss_lbfm}, we minimize it by setting the gradient with respect to $\boldsymbol{\theta}$ equal to zero, \ie
\[
\frac{\partial \mathcal{L}(\boldsymbol{\theta})}{\partial \boldsymbol{\theta}} = 0.
\]
This yields the closed-form solution for the optimal parameter vector (see Eq.~\eqref{eq:LR_close} for the derivation):
\begin{equation}\label{eq:lr:k24_3_2}
    \boldsymbol{\theta}^{*}
    = (\boldsymbol{\Phi}^{\top}\boldsymbol{\Phi})^{-1}
      \boldsymbol{\Phi}^{\top}\boldsymbol{y},
\end{equation}
which is derived using the same steps as those shown in Eq.~\eqref{eq:lr:theta_star}. Here, $\boldsymbol{\Phi}$ is the design matrix obtained by applying the basis functions to all input samples:
\begin{equation*}
\boldsymbol{\Phi}
=
\begin{pmatrix}
\phi_0, \phi_1(\boldsymbol{x}_1), \ldots, \phi_K(\boldsymbol{x}_1) \\
\vdots \\
\phi_0, \phi_1(\boldsymbol{x}_m), \ldots, \phi_K(\boldsymbol{x}_m) \\
\vdots \\
\phi_0, \phi_1(\boldsymbol{x}_M), \ldots, \phi_K(\boldsymbol{x}_M)
\end{pmatrix},
\end{equation*}
and $\boldsymbol{y} = (y_1, y_2, \ldots, y_M)^{\top}$ is the vector of target outputs.

\subsection{Common Basis Functions}

In linear basis function models, different basis functions correspond to different types of nonlinear feature transformations. By choosing appropriate basis functions, we can significantly enhance the model's ability to capture nonlinear patterns in the data. This section introduces several widely used and important classes of basis functions.

\paratitle{Gaussian Radial Basis Functions (RBF).}
Radial basis functions use the probability density function of a Gaussian distribution as their basis function, defined as
\begin{equation}
    \label{eq:lr:k24_3}
    \phi_k(\boldsymbol{x})
    = \exp\!\left(
        -\frac{\|\boldsymbol{x} - \boldsymbol{\mu}_k\|^2}{2\sigma^2}
      \right).
\end{equation}
Here, $\boldsymbol{\mu}_k$ is the center of the $k$-th basis function in the input space, and $\sigma$ is the standard deviation.
The notation $\|\cdot\|_{2}$ denotes the Euclidean (L2) norm.

The strong approximation capability of RBF models is supported by a classical result known as the \emph{RBF universal approximation theorem}~\cite{park1991universal}.
It states that, for any continuous target function $f(\boldsymbol{x})$ defined on a compact domain and any $\varepsilon > 0$, there exists a finite number of Gaussian radial basis functions and corresponding coefficients $\{\theta_n\}$ such that
\[
\bigl| f(\boldsymbol{x}) - \sum_{n=1}^{N} \theta_n \,\phi_n(\boldsymbol{x}) \bigr| < \varepsilon
\quad \text{for all } \boldsymbol{x}.
\]
In other words, linear combinations of Gaussian RBFs are dense in the space of continuous functions.
This guarantees that RBF-based linear models can approximate arbitrarily complex nonlinear relationships when sufficient basis functions are provided.

As illustrated in Fig.~\ref{fig:lr:linear_basis}, the centers $\boldsymbol{\mu}_n$ can be interpreted as coordinate points in a $K$-dimensional space.
A common practical choice is to place these centers uniformly across the input domain, as shown in Fig.~\ref{fig:lr:rbf_distrib_compose}, enabling the model to approximate the target function through linear combinations of these Gaussian ``bumps''.

\begin{figure}[t]
    \centering
    \subfigure{\includegraphics[width=0.35\linewidth]{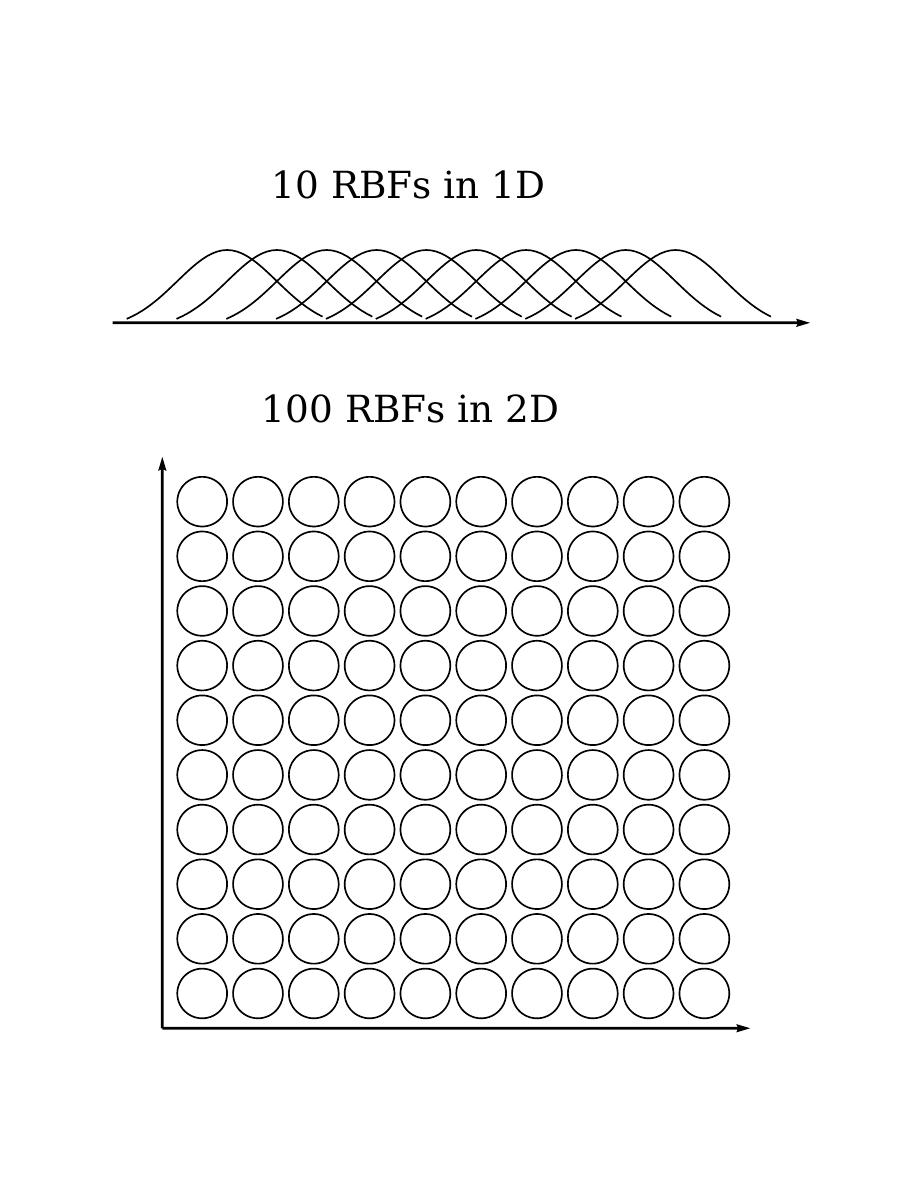}\label{fig:lr:rbf_distribution}}~~
    \subfigure{\includegraphics[width=0.35\linewidth]{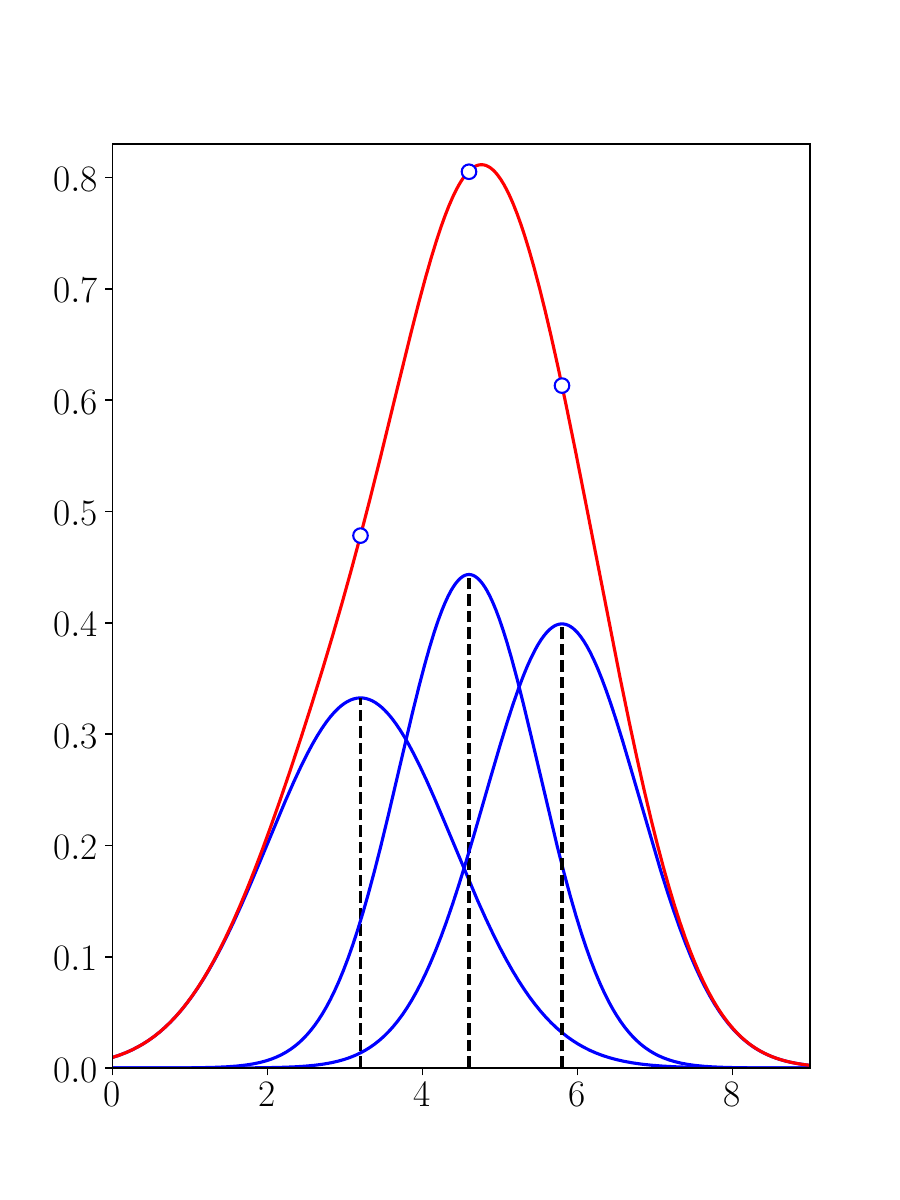}\label{fig:lr:rbf_composition}}
    \caption{\small Distribution of RBF centers and their linear composition.}
    \label{fig:lr:rbf_distrib_compose}
\end{figure}

\paratitle{Sigmoid Basis Functions.}
Another commonly used class of basis functions is the \emph{sigmoid basis function}.
A general sigmoid basis function is defined as
\begin{equation} \label{eq:lr:k24_2_general}
    \phi_n(\boldsymbol{x})
    = \sigma\!\left( \boldsymbol{w}_k^{\top}\boldsymbol{x} + b_k \right),
\end{equation}
where $\sigma(\cdot)$ is the sigmoid activation function,
\[
\sigma(z) = \frac{1}{1 + \exp(-z)}.
\]
At this point, one may notice that each sigmoid basis function is mathematically identical to the output of a logistic regression model applied to the input $\boldsymbol{x}$. The key difference is conceptual: in a linear basis function model, the parameters $\boldsymbol{w}_k$ and $b_k$ are not learned but must be predefined. In practice, the parameters in Eq.~\eqref{eq:lr:k24_2_general} can be determined using several heuristic strategies:
\begin{itemize}
    \item \textbf{Random initialization.}
    The parameters $\boldsymbol{w}_k$ and $b_k$ may be sampled from simple distributions (e.g., Gaussian or uniform).
    When a sufficient number of sigmoid basis functions is included, random features can provide effective approximations.

    \item \textbf{Data-driven initialization.}
    The offsets $b_k$ can be chosen so that the transition points of the sigmoids cover important regions of the input distribution.
    Meanwhile, $\boldsymbol{w}_k$ may be aligned with major directions of variation (e.g., principal components).

    \item \textbf{Clustering-based initialization.}
    By applying $k$-means clustering to the input data, one may place the sigmoid transition points near the cluster centers and adjust $\boldsymbol{w}_k$ to control the steepness of activation in each region.

    \item \textbf{Handcrafted structure.}
    In low-dimensional or pedagogical settings, $\boldsymbol{w}_k$ may be fixed (e.g., all pointing in the same direction), while $b_k$ is arranged to uniformly cover the domain-similar to the layout used for RBF centers.
\end{itemize}

Once $\boldsymbol{w}_k$ and $b_k$ are fixed, a linear basis function model remains linear in its final parameters and can be trained using closed-form regression or gradient descent.
If the parameters $\boldsymbol{w}_k$ and $b_k$ are also learned, the model becomes a neural network with one hidden layer.

\paratitle{Fourier Basis Functions.} Fourier basis functions form another important family of basis functions, especially suited for representing periodic or oscillatory structure in data. A general Fourier basis consists of sinusoidal functions of different frequencies:
\begin{equation}
    \phi_{k}^{\sin}(\boldsymbol{x}) = \sin(\boldsymbol{w}_k^{\top}\boldsymbol{x}),
    \qquad
    \phi_{k}^{\cos}(\boldsymbol{x}) = \cos(\boldsymbol{w}_k^{\top}\boldsymbol{x}),
\end{equation}
where $\boldsymbol{w}_k$ controls the frequency and direction of oscillation in the input space.

The theoretical foundation of Fourier basis functions comes from the classical \emph{Fourier series expansion}. For any continuous periodic function $f(x)$ defined on the interval $[-L, L]$, the Fourier series expresses $f(x)$ as an infinite linear combination of sine and cosine functions:
\begin{equation}~\label{eq:Fourier_series}
    f(x)
    = a_0
    + \sum_{k=1}^{\infty}
      \left(
          a_k \cos\!\left(\frac{k\pi x}{L}\right)
        + b_k \sin\!\left(\frac{k\pi x}{L}\right)
      \right),
\end{equation}
where the coefficients $a_k$ and $b_k$ are determined by the inner products of $f(x)$ with the corresponding basis functions.

Eq.~\eqref{eq:Fourier_series} shows that sine and cosine functions form a complete basis for continuous periodic functions, and finite truncations of the Fourier series can approximate such functions arbitrarily well. When applied in a linear basis function model, the parameters (\ie the coefficients $a_k$ and $b_k$) are simply learned from data, while the frequencies are either predefined or chosen according to certain strategies (\eg uniformly increasing frequencies or data-driven spectral analysis).

Fourier basis functions therefore provide a theoretically sound and practically effective method for modeling periodic or oscillatory patterns, complementing polynomial, Gaussian RBF, and sigmoid basis functions within the framework of linear basis function models.

\paratitle{Polynomial Basis Functions.} Polynomial basis functions are one of the most classical and widely used families of basis functions.
When the input is a vector $\boldsymbol{x} = (x_1, x_2, \ldots, x_N)^{\top}$, a polynomial basis of degree $K$ consists of all monomials of the components of $\boldsymbol{x}$ up to total degree $K$.
More formally, each basis function has the form
\begin{equation}
    \phi_{\boldsymbol{\alpha}}(\boldsymbol{x})
    = x_1^{\alpha_1} x_2^{\alpha_2} \cdots x_N^{\alpha_N},
\end{equation}
where the multi-index $\boldsymbol{\alpha} = (\alpha_1, \ldots, \alpha_N)$ satisfies
\[
\alpha_1 + \alpha_2 + \cdots + \alpha_N \le K,
\qquad
\alpha_i \in \mathbb{N}.
\]
The complete polynomial basis contains all such monomials, including the constant term $\phi_0(\boldsymbol{x}) = 1$. In the special case where the input is one-dimensional ($N=1$), this formulation reduces exactly to the standard polynomial regression model.

The theoretical foundation of polynomial basis functions stems from two fundamental results in approximation theory.
First, the multivariate Taylor series shows that a sufficiently smooth function $f(\boldsymbol{x})$ can be expanded as an infinite sum of multivariate monomials:
\begin{equation}
    f(\boldsymbol{x})
    = \sum_{|\boldsymbol{\alpha}| = 0}^{\infty}
      \frac{1}{\boldsymbol{\alpha}!}
      \left.\frac{\partial^{|\boldsymbol{\alpha}|} f}{\partial \boldsymbol{x}^{\boldsymbol{\alpha}}}\right|_{\boldsymbol{x}=\boldsymbol{0}}
      \boldsymbol{x}^{\boldsymbol{\alpha}}.
\end{equation}
More generally, the \emph{Weierstrass approximation theorem} (and its multivariate extensions) states that polynomial functions are dense in the space of continuous functions on compact domains.
Thus, for any continuous function $f(\boldsymbol{x})$ and any $\varepsilon > 0$, there exists a polynomial $p(\boldsymbol{x})$ such that
\[
\sup_{\boldsymbol{x}} \, | f(\boldsymbol{x}) - p(\boldsymbol{x}) | < \varepsilon.
\]
Polynomial basis functions therefore provide a universal approximation capability for continuous nonlinear mappings. They also form the theoretical foundation of polynomial regression: applying a linear regression model on polynomial basis functions is precisely what yields the classical polynomial regression model.

Polynomial basis functions are simple and expressive, but the number of monomials grows combinatorially with the input dimension $D$ and polynomial degree $K$, which may lead to computational challenges and overfitting.
For high-dimensional or highly nonlinear data, alternative basis functions such as Gaussian RBFs, sigmoids, or Fourier bases are often preferred.

\begin{table}[t]\footnotesize
\centering
\caption{Comparison of Common Basis Functions (Simplified).}
\renewcommand{\arraystretch}{1.25}
\begin{tabular}{lcl}
\hline
\textbf{Basis Function}
& \textbf{Form }
& \textbf{Key Property} \\
\hline

Polynomial
& $\phi_{\boldsymbol{\alpha}}(\boldsymbol{x})
   = x_1^{\alpha_1}\cdots x_N^{\alpha_N}$
& Global approximation; smooth trends \\

RBF (Gaussian)
& $\phi_k(\boldsymbol{x})
   = \exp\!\left(-\frac{\|\boldsymbol{x}-\boldsymbol{\mu}_k\|^2}
                        {2\sigma^{2}}\right)$
& Localized response; flexible shape modeling \\

Sigmoid
& $\phi_k(\boldsymbol{x})
   = \sigma(\boldsymbol{w}_k^{\top}\boldsymbol{x} + b_k)$
& Directional nonlinearity; NN-like activation \\

Fourier
& $\phi^{\sin}_k=\sin(\boldsymbol{w}_k^{\top}\boldsymbol{x}),\;
   \phi^{\cos}_k=\cos(\boldsymbol{w}_k^{\top}\boldsymbol{x})$
& Excellent for periodic or oscillatory patterns \\

\hline
\end{tabular}
\end{table}

\paratitle{Other Common Basis Functions.} In addition to polynomial, radial basis, sigmoid, and Fourier basis functions, many other families of basis functions have been developed for modeling diverse forms of nonlinear relationships.
\begin{itemize}
  \item \textbf{Spline basis functions}, such as B-splines and cubic splines, provide flexible piecewise-polynomial representations with excellent smoothness and numerical stability, and are widely used in curve fitting and smoothing applications.
  \item \textbf{Wavelet basis functions} (e.g., Haar, Daubechies, and Morlet wavelets) offer simultaneous localization in both time and frequency domains, making them particularly effective for representing signals with local discontinuities or multi-scale structure.
  \item \textbf{Orthogonal polynomial families} -- including Chebyshev, Legendre, Laguerre, and Hermite polynomials -- serve as numerically stable alternatives to standard monomials and enjoy desirable properties such as orthogonality and completeness.
  \item \textbf{Kernel-induced basis functions}, arising from kernel methods such as Gaussian, Laplacian, or Mat\'{e}rn kernels, implicitly generate high- or even infinite-dimensional feature spaces that enable powerful nonlinear modeling capabilities.
  \item \textbf{Piecewise and indicator-type basis functions}, including step functions, histogram bins, and partition-based features, offer simple yet effective representations for nonparametric regression and decision tree methods.
\end{itemize}
These diverse basis function families greatly expand the expressive power of linear basis function models and allow practitioners to tailor the model to the structure and characteristics of the data at hand.

\subsection{A First Look at Kernel Methods}

The closed-form solution and prediction formula of the linear basis function model naturally involve a number of inner-product computations over basis function evaluations. This makes the model closely related to the idea of kernel methods. To prepare the reader for later chapters, we provide a brief introduction to kernels here.

\paratitle{Kernel Function.} For regression tasks using a linear basis function model, the optimal parameters are given by the closed-form solution (see Eq.~\eqref{eq:lr:k24_3_2}):
\[
\boldsymbol{\theta}^{*}
= (\boldsymbol{\Phi}^{\top}\boldsymbol{\Phi})^{-1}
  \boldsymbol{\Phi}^{\top}\boldsymbol{y}.
\]
For a new input sample $\boldsymbol{x}_i$, substituting this solution into the regression function in Eq.~\eqref{eq:lr:k24_1} yields the prediction
\begin{equation}
    \label{eq:lr:k24_4}
    \hat{y}_i
    = \boldsymbol{\theta}^{* \top} \boldsymbol{\phi}(\boldsymbol{x}_i)
    = \boldsymbol{y}^{\top}
      (\boldsymbol{\Phi}\boldsymbol{\Phi}^{\top})^{-1}
      \boldsymbol{\Phi}\,
      \boldsymbol{\phi}(\boldsymbol{x}_i).
\end{equation}
From Eq.~\eqref{eq:lr:k24_4}, we observe that the prediction relies on a series of inner products involving basis function evaluations, such as
$\boldsymbol{\phi}(\boldsymbol{x}_m)^{\top}\boldsymbol{\phi}(\boldsymbol{x}_i)$
and the matrix product $\boldsymbol{\Phi}\boldsymbol{\Phi}^{\top}$.
When the number of basis functions is large--or when each basis function maps the input into a high-dimensional feature space--these inner products become computationally expensive or even infeasible to compute explicitly.

To address this issue, we introduce the concept of \emph{kernel functions}, which provide an implicit way to compute inner products in a feature space without explicitly constructing the feature mapping.

\begin{mydef}[Kernel Function]
Let $\phi(\boldsymbol{x})$ be a mapping from the input space to a (possibly high-dimensional) feature space. A function
\[
\kappa : \mathcal{X} \times \mathcal{X} \rightarrow \mathbb{R}
\]
is called a {\em kernel function} if, for any two samples
$\boldsymbol{x}_i$ and $\boldsymbol{x}_j$, it satisfies
\begin{equation}
    \label{eq:lr:k24_5}
    \kappa(\boldsymbol{x}_i, \boldsymbol{x}_j)
    = \left\langle \phi(\boldsymbol{x}_i),\, \phi(\boldsymbol{x}_j) \right\rangle,
\end{equation}
where $\langle \cdot, \cdot \rangle$ denotes the inner product in the feature space induced by $\phi(\cdot)$. In this sense, a kernel function computes the inner product of two feature vectors implicitly, without explicitly forming $\phi(\boldsymbol{x})$.
\end{mydef}

\begin{table}[t]\footnotesize
    \centering
    \caption{Common Kernel Functions Corresponding to Typical Basis Functions.}
    \label{tab:lr:kernal_function}
    \begin{tabular}{lll}
    \toprule
        \textbf{Kernel Name} & \textbf{Form} & \textbf{Parameters} \\
    \midrule
        Linear Kernel
        & $\kappa(\boldsymbol{x}_i,\boldsymbol{x}_j)
           =\boldsymbol{x}_i^{\top}\boldsymbol{x}_j$
        & --- \\

        Polynomial Kernel
        & $\kappa(\boldsymbol{x}_i,\boldsymbol{x}_j)
           = (\boldsymbol{x}_i^{\top}\boldsymbol{x}_j + c)^d$
        & $d \ge 1$ degree,\; $c\ge 0$ bias term \\

        Gaussian RBF Kernel
        & $\kappa(\boldsymbol{x}_i,\boldsymbol{x}_j)
           = \exp\!\left(-\frac{\|\boldsymbol{x}_i-\boldsymbol{x}_j\|^2}{2\sigma^2}\right)$
        & $\sigma>0$ bandwidth \\

        Laplacian RBF Kernel
        & $\kappa(\boldsymbol{x}_i,\boldsymbol{x}_j)
           = \exp\!\left(-\frac{\|\boldsymbol{x}_i-\boldsymbol{x}_j\|}{\sigma}\right)$
        & $\sigma>0$ \\

        Sigmoid Kernel
        & $\kappa(\boldsymbol{x}_i,\boldsymbol{x}_j)
           = \tanh(\beta\,\boldsymbol{x}_i^{\top}\boldsymbol{x}_j + \theta)$
        & $\beta>0,\ \theta<0$ \\

        Fourier Kernel
        & $\kappa(\boldsymbol{x}_i,\boldsymbol{x}_j)
           = \cos(\boldsymbol{w}^{\top}(\boldsymbol{x}_i-\boldsymbol{x}_j))$
        & frequency vector $\boldsymbol{w}$ \\

    \bottomrule
    \end{tabular}
\end{table}

All basis functions introduced earlier --- including polynomial basis functions,
Gaussian radial basis functions, sigmoid basis functions, and Fourier basis
functions -- admit corresponding kernel functions that implicitly compute the
inner products in their associated feature spaces. Table~\ref{tab:lr:kernal_function} summarizes several commonly used kernels and
their connections to these basis functions.

\paratitle{Kernel Trick.} Using such kernel functions allows us to replace explicit inner-product computations in the prediction formula of Eq.~\eqref{eq:lr:k24_4}, thereby improving computational efficiency. Specifically, every inner product of the form $\langle \phi(\boldsymbol{x}_m),\, \phi(\boldsymbol{x}_i) \rangle$ can be replaced by the kernel evaluation $\kappa(\boldsymbol{x}_m, \boldsymbol{x}_i)$. This yields the kernelized prediction function
\begin{equation}\label{eq:lr:k24_5_2}
    \hat{y}_i
    = \boldsymbol{y}^{\top}
      \boldsymbol{K}^{-1}\,
      \boldsymbol{\kappa}(\boldsymbol{X}, \boldsymbol{x}_i).
\end{equation}
Here, $\boldsymbol{K}$ is the \emph{kernel matrix}, whose $(m_1,m_2)$-th entry
is $\kappa(\boldsymbol{x}_{m_1}, \boldsymbol{x}_{m_2})$, and
$\boldsymbol{\kappa}(\boldsymbol{X}, \boldsymbol{x}_i)$ is the kernel vector
defined by
\begin{equation}\label{eq:kernel_matrix}
    \boldsymbol{K}=
    \begin{pmatrix}
    \kappa(\boldsymbol{x}_1,\boldsymbol{x}_1) & \cdots & \kappa(\boldsymbol{x}_1,\boldsymbol{x}_M) \\
    \vdots & \ddots & \vdots \\
    \kappa(\boldsymbol{x}_M,\boldsymbol{x}_1) & \cdots & \kappa(\boldsymbol{x}_M,\boldsymbol{x}_M)
    \end{pmatrix},
    \qquad
    \boldsymbol{\kappa}(\boldsymbol{X},\boldsymbol{x}_i)
    =
    \begin{pmatrix}
    \kappa(\boldsymbol{x}_1,\boldsymbol{x}_i) \\
    \vdots \\
    \kappa(\boldsymbol{x}_M,\boldsymbol{x}_i)
    \end{pmatrix}.
\end{equation}
The kernel-based prediction function in Eq.~\eqref{eq:lr:k24_5_2} is algebraically equivalent to the original prediction function in Eq.~\eqref{eq:lr:k24_4} of linear basis function models but avoids explicitly constructing the potentially high-dimensional feature mappings $\phi(\boldsymbol{x})$.

The rewriting in Eq.~\eqref{eq:lr:k24_5_2} is known as the \textbf{kernel trick}. The key idea is simple: instead of explicitly mapping inputs to a high-dimensional feature space and computing inner products of the form $\langle \phi(\boldsymbol{x}), \phi(\boldsymbol{x}') \rangle$, we use a kernel function to obtain the same inner product directly. This allows us to work with very high-dimensional feature spaces efficiently and without ever computing the feature vectors themselves.

\paratitle{Discussion: A Non-parametric Perspective on Linear Basis Function Models.} The kernelized prediction formula in Eq.~\eqref{eq:lr:k24_5_2},
\[
\hat{y}_i
= \boldsymbol{y}^{\top}
  \boldsymbol{K}^{-1}\,
  \boldsymbol{\kappa}(\boldsymbol{X}, \boldsymbol{x}_i),
\]
offers an alternative interpretation of linear regression from a
\emph{non-parametric} viewpoint. To see this, observe that the kernel vector
$\boldsymbol{\kappa}(\boldsymbol{X}, \boldsymbol{x}_i)$---that is,
\[
\boldsymbol{\kappa}(\boldsymbol{X}, \boldsymbol{x}_i)
= \bigl(
    \kappa(\boldsymbol{x}_1,\boldsymbol{x}_i),\,
    \kappa(\boldsymbol{x}_2,\boldsymbol{x}_i),\,
    \ldots,\,
    \kappa(\boldsymbol{x}_M,\boldsymbol{x}_i)
  \bigr)^{\top},
\]
measures the similarity (or equivalently, inverse distance) between the new
input $\boldsymbol{x}_i$ and every training sample
$\boldsymbol{x}_1,\boldsymbol{x}_2,\ldots,\boldsymbol{x}_M$. Meanwhile, the matrix $\boldsymbol{K}^{-1}$ acts as a normalization or
correction factor determined entirely by the training data.
Under this interpretation, the prediction for $\boldsymbol{x}_i$ may be viewed
as a weighted average of all training labels:
\[
\hat{y}_i
= \sum_{m=1}^{M} \alpha_m(\boldsymbol{x}_i)\, y_m,
\]
where the weight $\alpha_m(\boldsymbol{x}_i)$ depends on the kernel
similarity between $\boldsymbol{x}_i$ and $\boldsymbol{x}_m$.

This formulation closely resembles classical \emph{non-parametric} regression
methods such as kernel regression and $k$-nearest neighbors (kNN).
In those approaches, the prediction at $\boldsymbol{x}_i$ is obtained by
computing a weighted average of the labels of nearby data points, where the
weights are determined by a kernel function applied to the distance between
samples. The kernelized linear regression prediction in Eq.~\eqref{eq:lr:k24_5_2} follows exactly the same principle: training samples that are more similar to $\boldsymbol{x}_i$ (as quantified by the kernel\footnote{Ordinary linear regression can be interpreted as using the linear kernel.}) contribute more strongly to its predicted value.
\emph{This perspective highlights an important conceptual connection:
although linear regression is typically regarded as a \emph{parametric} model,
its kernelized form behaves like a non-parametric estimator whose predictions
depend directly on the entire training set.
In this sense, kernelization transforms a classical parametric model into a
flexible non-parametric method capable of capturing complex nonlinear patterns.}

From the discussion above, we can see that this non-parametric formulation is
highly inefficient in large-scale data settings. For a linear basis function model, kernelizing the closed-form solution requires computing the kernel similarity between the input sample and \emph{every} training sample. When the training set is large, this results in a prohibitive computational cost, making the closed-form--plus--kernel approach significantly less efficient than
gradient-based optimization methods. In contrast, kernel methods become extremely effective in a different class of
models---most notably, the Support Vector Machine (SVM). The key reason is
that, unlike kernelized linear regression, SVMs identify a small subset of
training samples, known as \emph{support vectors}, that alone determine the
final decision function. All other training samples receive zero weights and do
not contribute to the prediction. As a result, kernel computations are needed
only between the input sample and a much smaller set of support vectors,
greatly reducing the computational burden and enabling kernel methods to scale
to large datasets. We will introduce the full theoretical foundations of kernel methods in the dedicated SVM course.

A detailed treatment of kernel functions and kernelized learning will be provided in the SVM chapter. Here, our goal is simply to illustrate that the prediction formula of the linear basis function model already contains the structural form of a kernel method, and that kernelization is a powerful technique for simplifying and accelerating these implicit inner-product computations.

\subsection{(Deep) Neural Networks}

In our earlier discussion of sigmoid basis functions, we noted that a linear
basis function model equipped with learnable sigmoid parameters is precisely a
single-hidden-layer neural network. Given the central role of neural networks not only in modern machine learning but also in the broader development of artificial intelligence, it is essential for a regression analysis textbook to devote a dedicated section to this topic. In this subsection, we introduce neural networks from the perspective of regression modeling.

\paratitle{Regression Function of Neural Networks with a Single Hidden Layer.}
Given a dataset $\mathcal{D} = \{(\boldsymbol{x}_1, y_1),$ $(\boldsymbol{x}_2, y_2),$ $\ldots,$ $(\boldsymbol{x}_M, y_M)\}$, where $\boldsymbol{x}_m \in \mathbb{R}^{N}$ and $y_m \in \mathbb{R}$, a neural network constructs a regression function by stacking multiple layers of computation, typically organized into a sequence of \emph{hidden layers} followed by an \emph{output layer}. Each hidden layer transforms its input through a learnable affine mapping and a nonlinear activation function, while the output layer aggregates the final hidden representation to produce the model's prediction.

We begin with a neural network containing only one hidden layer as an illustrative example. A neural network consists of three key elements: \emph{neurons}, \emph{hidden layers}, and the \emph{output layer}.
\begin{itemize}
  \item \textbf{Neurons:} Given an input sample $\boldsymbol{x}_m$, a neuron applies a learnable affine transformation followed by a nonlinear activation function to extract a nonlinear representation from the input. Formally, the $k$-th neuron in the hidden layer computes
      \begin{equation}\label{eq:nn_neuron_single}
        h_k(\boldsymbol{x}_m)
          = \sigma\!\left( \boldsymbol{\theta}_{k}^{(1)\top} \boldsymbol{x}_m \right),
      \end{equation}
      where $\boldsymbol{\theta}_{k}^{(1)}$ is the parameter vector of the neuron and   $\sigma(\cdot)$ is an activation function such as sigmoid, tanh, or ReLU.

\item \textbf{Hidden layers:} A hidden layer is formed by stacking multiple neurons, each producing a scalar output. By concatenating these output, the layer outputs a vector known as the \emph{hidden representation vector} for the input sample $\boldsymbol{x}_m$. For a hidden layer with $K$ neurons, the combined output is
    \begin{equation}\label{eq:nn_hidden_single}
      \boldsymbol{h}^{(1)}(\boldsymbol{x}_m)
        = \bigl(h_1(\boldsymbol{x}_m),\, h_2(\boldsymbol{x}_m),\, \ldots,\, h_K(\boldsymbol{x}_m)\bigr)^{\top}.
    \end{equation}
    Hidden layers capture increasingly complex patterns by transforming the input into higher-level and more expressive representations.

  \item \textbf{Output layer:} The output layer takes the hidden representation as its input and converts it into the final prediction of the network. For regression problems, the output   layer applies a linear regression function:
      \begin{equation}\label{eq:nn_output_single}
        \hat{y}_m
          = \theta^{(2)}_0
          + \sum_{k=1}^{K} \theta^{(2)}_k \, h_k(\boldsymbol{x}_m),
      \end{equation}
      where $\theta^{(2)}_0$ and $\theta^{(2)}_k$ are the learnable parameters of the output layer. For binary or multi-class classification, the linear output may instead be passed through a logistic or Softmax regression function.
\end{itemize}

A neural network with a single hidden layer can be viewed as a data-dependent
extension of a linear basis function model. In this correspondence, each neuron
in the hidden layer plays the role of a basis function. More precisely, a linear
basis function model represents the input through fixed functions
$\{\phi_k(\boldsymbol{x})\}_{k=1}^{K}$, whereas a neural network replaces them with
learnable nonlinear transformations
\[
    \phi_k(\boldsymbol{x})
    \;\longleftrightarrow\;
    h_k(\boldsymbol{x})
    = \sigma\!\left( \boldsymbol{\theta}_k^{(1)\top} \boldsymbol{x} \right),
\]
where the parameters $\boldsymbol{\theta}_k^{(1)}$ of each neuron are optimized
from data. The output layer then forms a linear combination of these learned
``basis functions,'' just as in a traditional linear basis function model, but with
the crucial advantage that the basis functions are automatically adapted to the
regression problem.

\paratitle{Regression Function of Multiple-Hidden-Layer (Deep) Neural Networks.} A deep neural network extends the single-hidden-layer architecture by stacking
multiple hidden layers sequentially. Let $\boldsymbol{h}^{(0)} = \boldsymbol{x}_m$ denote the network input. For a network with $L$ hidden layers, the components are defined as follows:
\begin{itemize}
    \item \textbf{Neurons:} Each neuron in hidden layer $l$ ($l = 1,\ldots,L$) applies a learnable affine transformation to the output of the previous layer, followed by a nonlinear activation. The $k_l$-th neuron in layer $l$ computes
        \[
            h^{(l)}_{k_l}(\boldsymbol{x}_m)
            =
            \sigma\!\left(
                \boldsymbol{\theta}^{(l)}_{k_l}{}^{\top}
                \boldsymbol{h}^{(l-1)}(\boldsymbol{x}_m)
            \right),
        \]
        where $\boldsymbol{\theta}^{(l)}_{k_l}$ is the parameter vector of this neuron and $\sigma(\cdot)$ denotes an activation function such as the sigmoid, tanh, or ReLU.

    \item \textbf{Hidden layers:}
    Hidden layer $l$ consists of $K_l$ neurons. The outputs of all neurons in the layer
    form the hidden representation vector
    \[
        \boldsymbol{h}^{(l)}(\boldsymbol{x}_m)
        =
        \bigl(
            h^{(l)}_{1}(\boldsymbol{x}_m),\,
            h^{(l)}_{2}(\boldsymbol{x}_m),\,
            \ldots,\,
            h^{(l)}_{K_l}(\boldsymbol{x}_m)
        \bigr)^{\top}.
    \]
    As depth increases, the hidden representations become progressively more expressive,
    enabling the network to capture hierarchical nonlinear patterns.

    \item \textbf{Output layer:}
    The output layer (indexed as layer $L+1$) maps the final hidden representation
    $\boldsymbol{h}^{(L)}(\boldsymbol{x}_m)$ to the network's prediction. For a
    regression task, it applies a linear regression function:
    \[
        \hat{y}_m
        =
        \theta^{(L+1)}_0
        +
        \sum_{k_L=1}^{K_L}
        \theta^{(L+1)}_{k_L}\,
        h^{(L)}_{k_L}(\boldsymbol{x}_m),
    \]
    where $\theta^{(L+1)}_0$ and $\theta^{(L+1)}_{k_L}$ are the learnable parameters of
    the output layer.

\end{itemize}

This multilayer composition enables deep neural networks to construct rich
hierarchical representations, allowing them to approximate highly complex
regression functions.

\paratitle{Universal Approximation Theorem of Neural Networks.} The expressive power of neural networks is supported by a fundamental result known as the universal approximation theorem. It states that even shallow neural networks with a single hidden layer are capable of approximating a broad class of functions.
\begin{theorem}[Universal Approximation Theorem]
Let $f: \mathbb{R}^{N} \rightarrow \mathbb{R}$ be a continuous function defined on a compact subset of $\mathbb{R}^{N}$. Let $\sigma(\cdot)$ be a non-constant, bounded, and continuous activation function. Then, for any $\varepsilon > 0$, there exist an integer $K$ and parameters $\{\boldsymbol{\theta}^{(1)}_{k}\}_{k=1}^{K}$ and $\{\theta^{(2)}_{k}\}_{k=1}^{K}$ such that the neural network
\[
    \hat{f}(\boldsymbol{x})
    =
    \theta^{(2)}_0
    +
    \sum_{k=1}^{K}
    \theta^{(2)}_k\,
    \sigma\!\left( \boldsymbol{\theta}^{(1)\top}_{k} \boldsymbol{x} \right)
\]
satisfies
\[
    \lvert f(\boldsymbol{x}) - \hat{f}(\boldsymbol{x}) \rvert < \varepsilon,
    \qquad
    \text{for all } \boldsymbol{x} \text{ in the compact domain}.
\]
\end{theorem}
This theorem provides a theoretical justification for using neural networks as flexible nonlinear regression models. Although the network parameters are learned from data, the architecture itself is expressive enough to approximate a wide range of input-output relationships. Moreover, deeper architectures often achieve the same approximation accuracy more efficiently, using fewer neurons per layer and enabling hierarchical feature extraction.

\paratitle{Loss Function of Neural Networks.} Consistent with the regression models introduced earlier, neural networks optimize the same class of loss functions, with the distinction that the parameter set now includes \emph{all} learnable weights across both the hidden layers and the output layer. For a neural network with $L$ hidden layers, let
\begin{equation*}\label{}
  \boldsymbol{\Theta}^{(1)},\, \boldsymbol{\Theta}^{(2)},\, \ldots,\, \boldsymbol{\Theta}^{(L)}
\end{equation*}
denote the collections of neuron-level parameters in layers $1$ through $L$, where each \(\boldsymbol{\Theta}^{(l)}\) consists of the parameter vectors ${\small
    \boldsymbol{\Theta}^{(l)}
    =
    \bigl(
        \boldsymbol{\theta}^{(l)}_{1},\,
        \boldsymbol{\theta}^{(l)}_{2},\,
        \ldots,\,
        \boldsymbol{\theta}^{(l)}_{K_l}
    \bigr)}$.
The output layer is parameterized by {\small $
    \boldsymbol{\theta}^{(L+1)}
    =
    \bigl(
        \theta^{(L+1)}_0,\,
        \theta^{(L+1)}_1,\,
        \ldots,\,
        \theta^{(L+1)}_{K_L}
    \bigr)$}.
For regression problems with real-valued outputs, the standard loss function is the mean squared error (MSE):
\[
    \mathcal{L}
    \left(
        \boldsymbol{\Theta}^{(1)}, \ldots, \boldsymbol{\Theta}^{(L)},\,
        \boldsymbol{\theta}^{(L+1)}
    \right)
    =
    \frac{1}{M}
    \sum_{m=1}^{M}
        \Bigl( \hat{y}_m - y_m \Bigr)^2 .
\]
For binary classification tasks, the network output is interpreted as a probability
$\hat{p}_m$, and the cross-entropy loss is used:
\[
    \mathcal{L}
    \left(
        \boldsymbol{\Theta}^{(1)}, \ldots, \boldsymbol{\Theta}^{(L)},\,
        \boldsymbol{\theta}^{(L+1)}
    \right)
    =
    -\frac{1}{M}
    \sum_{m=1}^{M}
        \Bigl(
            y_m \log \hat{y}_m
            + (1 - y_m)\log(1 - \hat{y}_m)
        \Bigr).
\]
Thus, while the loss functions mirror those in classical regression and classification, the key distinction of neural networks lies in their learnable nonlinear transformations, which endow them with a far more expressive hypothesis space.

\subsection{Backpropagation  (BP) Algorithm}

Training a neural network is fundamentally different from fitting a linear regression model. Because the network contains nonlinear activation functions and parameters that interact across layers, the resulting prediction function is \emph{nonlinear} in all of its parameters. As a consequence, the optimal parameters cannot be obtained in closed form, and iterative gradient-based optimization methods -- most commonly stochastic gradient descent and its variants -- must be employed. A central challenge then arises: neural networks often contain thousands or even millions of parameters arranged across multiple layers, and directly computing the gradient of the loss with respect to each parameter via naive application of the chain rule would be computationally prohibitive. The breakthrough that makes large-scale neural network training feasible is the \emph{backpropagation} (BP) algorithm.

The BP algorithm provides an efficient, systematic, and scalable way to compute all parameter gradients by exploiting the layered structure of the network. Rather than recomputing intermediate derivatives repeatedly, BP propagates error signals backward from the output layer to the earlier layers, reusing computations along the way. This mechanism enables the gradient of every parameter to be obtained with a computational cost that is only a small constant multiple of the cost of evaluating the network itself.

In the remainder of this subsection, we derive the BP algorithm step by step using a layer-by-layer formulation. Beginning with a single-hidden-layer network to illustrate the essential ideas, we then generalize the derivation to networks with multiple hidden layers. This development reveals the recurring structure behind gradient computations and highlights why backpropagation is both elegant and indispensable in modern deep learning.

\paratitle{BP Algorithm for Single-hidden-layer Neural Networks.} Consider the single-hidden-layer neural network defined earlier:
\[
    \hat{y}_m
    =
    \theta^{(2)}_0
    +
    \sum_{k_1=1}^{K_1}
        \theta^{(2)}_{k_1}\,
        \sigma\!\left(
            \boldsymbol{\theta}^{(1)\top}_{k_1}\boldsymbol{x}_m
        \right),
\]
and let $\mathcal{L}(\boldsymbol{\Theta}^{(1)},\boldsymbol{\theta}^{(2)})$ denote the loss function. We rewrite the network in a layer-by-layer form. The output layer is
\begin{equation}\label{eq:nn_output}
    \hat{y}_m
    =
    \boldsymbol{\theta}^{(2)\top}\,
    \boldsymbol{h}^{(1)}(\boldsymbol{x}_m),
\end{equation}
and the hidden layer is
\begin{equation}\label{eq:nn_hidden}
    h_{k_1}^{(1)}(\boldsymbol{x}_m)
    =
    \sigma\!\left(
        \boldsymbol{\theta}^{(1)\top}_{k_1}\boldsymbol{x}_m
    \right),
    \qquad
    k_1=1,\ldots,K_1.
\end{equation}
We derive the backpropagation (BP) algorithm using the familiar pattern
\[
    \text{gradient}
    =
    (\text{prediction} - \text{label}) \times \text{input} = \text{error} \times \text{input} .
\]

$\bullet$ {\bf Gradients for the Output Layer.} For the output layer, we treat $\boldsymbol{h}^{(1)}(\boldsymbol{x}_m)$ as the input. Equation~\eqref{eq:nn_output} then becomes a linear regression model. Following the pattern ``$\text{gradient} = \text{error} \times \text{input}$'', we can directly write the corresponding derivative of the parameter  $\theta^{(2)}_{k_1}$ for the $m$-th training sample as
\[
    \frac{\partial \mathcal{L}_m}{\partial \theta^{(2)}_{k_1}}
    =
    (\hat{y}_m - y_m)\cdot
    h^{(1)}_{k_1}(\boldsymbol{x}_m).
\]
Here, the {\em error} is $(\hat{y}_m - y_m)$, and the {\em input} is
$h^{(1)}_{k_1}(\boldsymbol{x}_m)$ from the hidden layer.

$\bullet$ {\bf Gradients for the Hidden Layer.} Now consider the parameters $\boldsymbol{\theta}^{(1)}_{k_1}$ of the $k$-th hidden neuron for the first hidden layer. Applying the chain rule to \eqref{eq:nn_output} and \eqref{eq:nn_hidden}, we obtain
\[
\frac{\partial \mathcal{L}_m}{\partial \boldsymbol{\theta}^{(1)}_{k_1}}
=
\frac{\partial \mathcal{L}_m}{\partial \hat{y}_m}
\cdot
\frac{\partial \hat{y}_m}{\partial h^{(1)}_{k_1}}
\cdot
\frac{\partial h^{(1)}_{k_1}}{\partial (\boldsymbol{\theta}^{(1)\top}_{k_1}\boldsymbol{x}_m)}
\cdot
\frac{\partial (\boldsymbol{\theta}^{(1)\top}_{k_1}\boldsymbol{x}_m)}
     {\partial \boldsymbol{\theta}^{(1)}_{k_1}}.
\]
Each term is
\begin{equation}\footnotesize
  \frac{\partial \mathcal{L}_m}{\partial \hat{y}_m} = (\hat{y}_m - y_m),
\quad
\frac{\partial \hat{y}_m}{\partial h^{(1)}_{k_1}} = \theta^{(2)}_{k_1},
\quad
\frac{\partial h^{(1)}_{k_1}}{\partial (\boldsymbol{\theta}^{(1)\top}_{k_1}\boldsymbol{x}_m)}
=
\sigma'\!\left(
    \boldsymbol{\theta}^{(1)\top}_{k_1}\boldsymbol{x}_m
\right),
\quad
\frac{\partial (\boldsymbol{\theta}^{(1)\top}_{k_1}\boldsymbol{x}_m)}
     {\partial \boldsymbol{\theta}^{(1)}_{k_1}}
=
\boldsymbol{x}_m.
\end{equation}
Thus,
\[
\frac{\partial \mathcal{L}_m}{\partial \boldsymbol{\theta}^{(1)}_{k_1}}
=
(\hat{y}_m - y_m)\,
\theta^{(2)}_{k_1}\,
\sigma'\!\left(
    \boldsymbol{\theta}^{(1)\top}_{k_1}\boldsymbol{x}_m
\right)
\boldsymbol{x}_m.
\]

$\bullet$ {\bf Backpropagation of Error Signals.} We define the hidden layer as layer \(1\) and the output layer as layer \(2\). The error at
layer \(2\) is
\begin{equation}\label{eq:error_2}
    \delta^{(2)} = \hat{y}_m - y_m.
\end{equation}
To compute the gradient for the hidden layer, we first determine how much of the
output-layer error should be assigned to each hidden neuron. This quantity is the {\em error signal} for neuron \(k_1\) in layer \(1\), denoted as $\delta^{(1)}_{k_1}$. It is defined by propagating $\delta^{(2)}$ backward through the weight connecting neuron \(k_1\) to the output, \ie
\begin{equation}\label{eq:error_1_new_rewrite}
    \delta^{(1)}_{k_1}
    =
    \delta^{(2)}\,
    \theta^{(2)}_{k_1}\,
    \sigma'\!\left(
        \boldsymbol{\theta}^{(1)\top}_{k_1}\boldsymbol{x}_m
    \right).
\end{equation}
Once the error signal is computed, the gradient with respect to the hidden-layer
parameters follows immediately:
\begin{equation}\label{eq:pattern}
    \frac{\partial \mathcal{L}_m}{\partial \boldsymbol{\theta}^{(1)}_{k_1}}
    =
    \delta^{(1)}_{k_1}\,\boldsymbol{x}_m.
\end{equation}

This expression reflects the fundamental pattern ``$\text{gradient} = \text{error} \times \text{input}$''. Here, the {\em error} term is the error signal $\delta^{(1)}_{k_1}$, which measures the influence of neuron \(k_1\) on the final prediction. The {\em input} term is the vector $\boldsymbol{x}_m$, indicating how strongly the neuron was activated by the sample. Together, they determine how the parameters of the hidden neuron should be updated (\ie the gradient ${\partial \mathcal{L}_m}/{\partial \boldsymbol{\theta}^{(1)}_{k_1}}$).

\paratitle{Extension to Multiple Hidden Layers.} For a neural network with $L$ hidden layers, we index the hidden layers by $l = 1,\ldots,L$ and the output layer by $l = L+1$. The backpropagation algorithm computes an \emph{error signal} for each neuron, layer by layer, starting from the output layer and moving backward. For layer $l$, the error signal of neuron $k_l$ is obtained by propagating backward the error signals from layer $l+1$:
\[
    \delta^{(l)}_{k_l}
    =
    \sum_{k_{l+1}=1}^{K_{l+1}}
        \delta^{(l+1)}_{k_{l+1}}\,
        \theta^{(l+1)}_{k_{l+1},k_l}\,
        \sigma'\!\left(
            z^{(l)}_{k_l}
        \right),
\]
where
\begin{itemize}
    \item $\delta^{(l)}_{k_l}$ is the error signal of neuron $k_l$ in layer $l$,
    \item $\theta^{(l+1)}_{k_{l+1},k_l}$ is the weight connecting neuron $k_l$ in layer $l$
          to neuron $k_{l+1}$ in layer $l+1$,
    \item $\delta^{(l+1)}_{k_{l+1}}$ is the error signal propagated from layer $l+1$,
    \item $z^{(l)}_{k_l}$ is the pre-activation input of neuron $k_l$ in layer $l$,
    \item $\sigma'(\cdot)$ is the derivative of the activation function.
\end{itemize}

The formula provides an intuitive interpretation of how error signals propagate
backward through the network. The error signal of a neuron in layer $l$ is obtained by
\emph{collecting all error signals coming from the next layer}, weighting each by the
strength of its connection, and then scaling the result by the neuron's activation sensitivity. More concretely:
\begin{itemize}
    \item Each neuron in layer $l$ sends its output to all neurons in layer $l+1$.
    \item When the error signals $\delta^{(l+1)}_{k_{l+1}}$ are computed for layer $l+1$,
          they flow backward through their respective weights
          $\theta^{(l+1)}_{k_{l+1},k_l}$.
    \item A larger weight indicates a stronger forward influence of neuron $k_l$ on
          neuron $k_{l+1}$, so a larger portion of the error should be assigned to
          neuron $k_l$.
    \item Summing over all neurons in layer $l+1$ aggregates these contributions.
    \item Multiplying by $\sigma'\!\left(z^{(l)}_{k_l}\right)$ adjusts the total by the
          local slope of the activation function, capturing how sensitive the neuron is
          at its current activation level.
\end{itemize}
In summary,
\[
    \delta^{(l)}_{k_l}
    =
    \text{(sum of error signals from the next layer)}
    \times
    \text{(local activation sensitivity)}.
\]
This recursion is a direct consequence of the chain rule: a neuron receives a nonzero error signal only to the extent that it influenced the next layer and was active enough for that influence to matter.

\paratitle{Algorithm and Efficiency Analysis.} Algorithm~\ref{alg:bp} summarizes the full backpropagation procedure for a feedforward neural network with $L$ hidden layers. The algorithm consists of three sequential stages: a forward pass that computes all intermediate activations, a backward pass that propagates error signals from the output layer back through the hidden layers, and a final parameter update step that applies the gradient-based learning rule. This layer-by-layer organization makes it possible to compute all partial derivatives
efficiently and systematically.

\begin{algorithm}[t]\footnotesize
\caption{Backpropagation Algorithm}
\label{alg:bp}
\begin{algorithmic}[1]
\Require training sample $(\boldsymbol{x}_m, y_m)$;
         network parameters $\{\boldsymbol{\Theta}^{(l)}\}_{l=1}^{L}$ and $\boldsymbol{\theta}^{(L+1)}$;
         learning rate $\eta$
\Statex

\State \textbf{Forward pass:}
\State $\boldsymbol{h}^{(0)} \gets \boldsymbol{x}_m$
\For{$l = 1, \ldots, L$}
    \For{$k_l = 1, \ldots, K_l$}
        \State $z^{(l)}_{k_l} \gets \boldsymbol{\theta}^{(l)\top}_{k_l}\,\boldsymbol{h}^{(l-1)}$
        \State $h^{(l)}_{k_l} \gets \sigma\!\bigl(z^{(l)}_{k_l}\bigr)$
    \EndFor
\EndFor
\State $\hat{y}_m \gets \theta^{(L+1)}_0 + \displaystyle\sum_{k_L=1}^{K_L} \theta^{(L+1)}_{k_L}\,h^{(L)}_{k_L}$
\State compute loss $\mathcal{L}_m$ and set
       $\delta^{(L+1)} \gets \dfrac{\partial \mathcal{L}_m}{\partial \hat{y}_m}$
       \Comment{e.g., $\delta^{(L+1)} = \hat{y}_m - y_m$ for MSE}
\Statex

\State \textbf{Backward pass (error signals):}
\For{$k_L = 1, \ldots, K_L$}
    \State $\delta^{(L)}_{k_L}
        \gets
        \delta^{(L+1)}\,\theta^{(L+1)}_{k_L}\,
        \sigma'\!\bigl(z^{(L)}_{k_L}\bigr)$
\EndFor
\For{$l = L-1, \ldots, 1$}
    \For{$k_l = 1, \ldots, K_l$}
        \State $\delta^{(l)}_{k_l}
            \gets
            \sigma'\!\bigl(z^{(l)}_{k_l}\bigr)
            \displaystyle\sum_{k_{l+1}=1}^{K_{l+1}}
                \theta^{(l+1)}_{k_{l+1},k_l}\,
                \delta^{(l+1)}_{k_{l+1}}$
    \EndFor
\EndFor
\Statex

\State \textbf{Gradient updates:}
\State $\theta^{(L+1)}_0 \gets \theta^{(L+1)}_0 - \eta\,\delta^{(L+1)}$
\For{$k_L = 1, \ldots, K_L$}
    \State $\theta^{(L+1)}_{k_L}
        \gets
        \theta^{(L+1)}_{k_L}
        - \eta\,\delta^{(L+1)}\,h^{(L)}_{k_L}$
\EndFor
\For{$l = 1, \ldots, L$}
    \For{$k_l = 1, \ldots, K_l$}
        \State $\boldsymbol{\theta}^{(l)}_{k_l}
            \gets
            \boldsymbol{\theta}^{(l)}_{k_l}
            - \eta\,\delta^{(l)}_{k_l}\,\boldsymbol{h}^{(l-1)}$
    \EndFor
\EndFor

\end{algorithmic}
\end{algorithm}

As shown in Algorithm~\ref{alg:bp}, the backpropagation procedure is highly efficient because it exploits the layered structure of the network and reuses intermediate computations from the forward pass. During the backward sweep, each layer computes its error signals using only local information: the error signals from the next layer, the corresponding weights, and the stored pre-activation values. This avoids redundant computations and prevents the exponential blow-up that would occur if derivatives were computed independently for each parameter.

From a computational perspective, the total running time of backpropagation scales linearly with the number of parameters in the network. Let $W$ denote the total number of weights (including biases). Both the forward pass and the backward pass require processing each weight exactly once, through either a multiplication or an elementwise activation evaluation. Therefore, the overall computational cost of one full forward-backward cycle satisfies
\[
    T_{\text{BP}} = O(W),
\]
up to a small constant factor. This linear-time complexity is the key reason why
backpropagation remains practical even for modern deep neural networks containing millions or billions of parameters.

\paratitle{Limitations of the Backpropagation Algorithm.} Despite its central role in modern deep learning, the backpropagation algorithm has a fundamental limitation: the problem of \emph{vanishing} and \emph{exploding} gradients. Because BP repeatedly applies the chain rule across many layers, gradients are multiplied by weights and activation derivatives at every layer. In deep networks, these multiplicative factors may be consistently smaller than~$1$
or larger than~$1$, causing gradients to either decay toward zero or grow
uncontrollably. As a result, early layers may receive almost no learning signal
(vanishing gradients), or the training process may become numerically unstable
(exploding gradients).  To see these limitations more concretely, we can examine the recursive formula for the error signals. For a network with $L$ hidden layers, the error signal in layer~$l$ is given by
\[
    \delta^{(l)}_{k_l}
    =
    \sigma'\!\left(z^{(l)}_{k_l}\right)
    \sum_{k_{l+1}=1}^{K_{l+1}}
        \theta^{(l+1)}_{k_{l+1},k_l}\,
        \delta^{(l+1)}_{k_{l+1}}.
\]
Along a single path of neurons
\[
    k_1 \to k_2 \to \cdots \to k_L \to \text{output},
\]
ignoring the summation over branches, repeated substitution of the recursion yields
an approximate expression for the error signal in layer~1:
\[
    \delta^{(1)}_{k_1}
    \approx
    \delta^{(L+1)}
    \prod_{l=1}^{L}
        \theta^{(l+1)}_{k_{l+1},k_l}\,
        \sigma'\!\left(z^{(l)}_{k_l}\right).
\]
This product form clearly reveals the source of vanishing and exploding gradients. If the magnitudes of the weights $\theta^{(l+1)}_{k_{l+1},k_l}$ and activation derivatives $\sigma'\!\left(z^{(l)}_{k_l}\right)$ are typically less than~$1$, their product shrinks rapidly with depth, causing the gradient to vanish. If these terms are often greater than~$1$, the product grows exponentially, leading to exploding gradients. Thus, the recursive structure of backpropagation directly explains why training deep networks can be slow or unstable unless additional mechanisms are introduced to stabilize gradient flow.

A variety of techniques have been developed to mitigate vanishing and exploding
gradients. Common remedies include using alternative activation functions
(\eg ReLU and its variants), careful weight initialization strategies,
normalization layers such as batch normalization, and architectural designs
that improve gradient flow, most notably residual connections in deep
residual networks. These methods substantially stabilize training in very
deep models. We will discuss these techniques in detail in the course on
deep learning.

\newpage

\section{Regularization}

In the preceding sections, we have shown that regression models equipped with flexible basis functions -- such as neural networks -- can, in principle, approximate arbitrarily complex data distributions. At first glance, this appears to provide a universal modeling tool capable of capturing any underlying relationship in the data. However, with such expressive power comes a new challenge: the risk of overfitting. When a model becomes excessively strong, it may not only learn the general patterns present in the data but also memorize noise, idiosyncrasies, or sample-specific artifacts that do not generalize beyond the training set. Consequently, although highly expressive models can achieve near-perfect performance on training data, they may exhibit poor predictive accuracy on unseen samples.

In many practical scenarios, a moderately complex model is often preferable, as it captures the essential structure or ``shared regularities'' of the data while ignoring incidental fluctuations. In contrast, an overly flexible model may fit every detail of the training samples, including those arising from randomness, thereby harming its generalization ability. To address this issue, one must adopt techniques that explicitly control model complexity. This section introduces one of the most important approaches for achieving this goal --- regularization.

\subsection{Underfitting and Overfitting}

The linear basis function regression introduced in the previous section, together with the regularization techniques to be discussed in this chapter, plays a crucial role in addressing the problems of underfitting and overfitting. In this subsection, we focus on introducing these two fundamental concepts.

\paratitle{Concepts of Underfitting and Overfitting.} Underfitting and overfitting are two closely related yet fundamentally opposite phenomena that arise from the mismatch between model capacity and the underlying data-generating process.
\begin{itemize}
    \item \textbf{Underfitting} occurs when the regression model (together with its estimated parameters) has insufficient capacity to capture the true underlying data-generating process. As a result, it fails to fit not only unseen data but also the training data adequately, leading to high error on both.
    \item \textbf{Overfitting} occurs when the regression model (and its estimated parameters) is overly flexible and fits not only the true underlying structure but also the noise and sample-specific irregularities in the training data. Consequently, although it achieves very low training error, its performance deteriorates significantly on unseen data drawn from the same population.
\end{itemize}
{\em Underfitting typically arises when the model is too simple to capture the underlying structure of the data, whereas overfitting usually occurs when the model is excessively complex relative to the amount of available data.}

\begin{figure}[t]
    \centering
    \includegraphics[width=0.7\columnwidth]{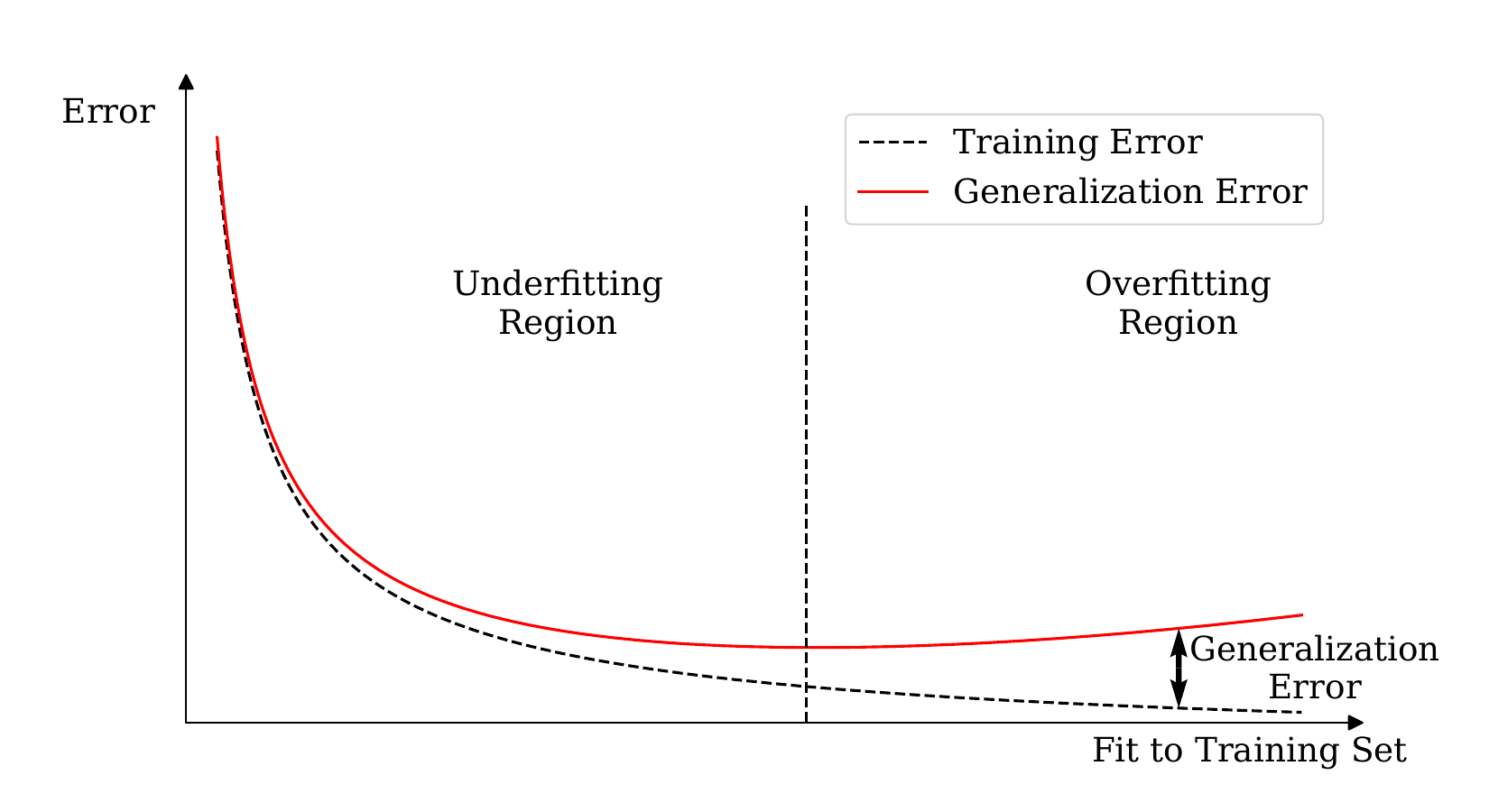}
    \caption{\small Relationship between underfitting, overfitting, training error, and generalization error.}
    \label{fig:lr:k28_2}
\end{figure}

We refer to the model's prediction error on the training dataset $\mathcal{D}_{\text{train}}$ as the \textbf{training error}, and the prediction error on previously unseen data (such as the test set $\mathcal{D}_{\text{test}}$) as the \textbf{generalization error}. The relationship between these two quantities under underfitting and overfitting is illustrated in Figure~\ref{fig:lr:k28_2}. When the model is too simple, both the training error and generalization error remain high, reflecting underfitting. As the model becomes more complex, the training error continues to decrease; however, beyond a certain level of complexity, the generalization error begins to rise due to overfitting. The discrepancy between the training error and the generalization error is known as the {generalization gap}. As shown in Figure~\ref{fig:lr:k28_2}, this gap widens as the degree of overfitting increases. A model that achieves a low generalization error is said to possess strong {generalization ability} or simply good {generalization}.

\paratitle{An Illustrative Example of Underfitting and Overfitting.} These concepts may seem abstract. To provide a more tangible understanding of underfitting and overfitting, we illustrate them with a concrete example.\footnote{This example is adapted from the book \emph{Pattern Recognition and Machine Learning}~\cite{bishop2006pattern}, with slight modifications.} Consider the dataset
\[
    \mathcal{D} = \{(x_1, y_1), (x_2, y_2), \ldots, (x_M, y_M)\},
\]
where both $x_m \in \mathbb{R}$ and $y_m \in \mathbb{R}$ are real-valued. The data are generated according to
\begin{equation}\label{eq:sin_data}
    y = \sin(2\pi x) + \epsilon,
\end{equation}
where $y$ is obtained by adding Gaussian noise $\epsilon$ with zero mean to the sinusoidal function of $x$. From the generative model~\eqref{eq:sin_data}, we sample $M = 10$ data points, resulting in the dataset $\mathcal{D}$ shown in Figure~\ref{fig:lr:sine_basis}.

\begin{figure}[t]
    \centering
    \includegraphics[width=0.4\columnwidth]{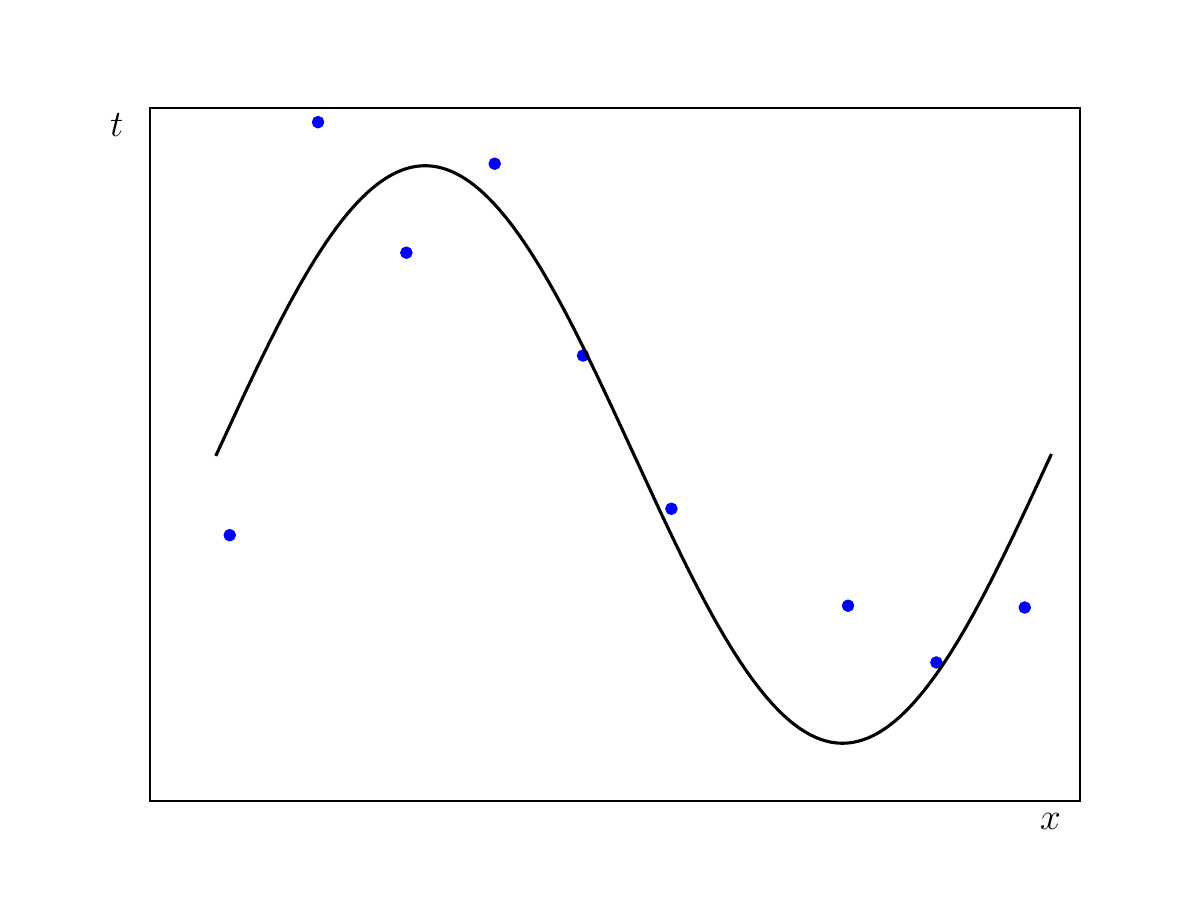}
    \caption{\small Visualization of the sampled dataset $\mathcal{D}$ generated from the noisy sinusoidal model $y=\sin(2\pi x)+\epsilon$, where $\epsilon$ is zero-mean Gaussian noise. The ten data points illustrate how random fluctuations around the underlying periodic pattern can make the true data-generating process difficult to observe directly.}
    \label{fig:lr:sine_basis}
\end{figure}

In real-world data analysis, the true data-generating process (such as $y = \sin(2\pi x) + \epsilon$) is unknown. Therefore, we approximate it using a polynomial regression model of degree $K$:
\begin{equation}\label{eq:lr:k29_1}
    \hat{y}_m(x_m, \bm{\theta})
    = \theta_0 + \theta_1 x_m + \theta_2 x_m^2 + \cdots + \theta_K x_m^K
    = \sum_{k=0}^K \theta_k x_m^k.
\end{equation}
When adopting the regression function in~\eqref{eq:lr:k29_1}, we must determine the degree $K$ of the polynomial. This choice is made during the design of the regression function, rather than through optimization of the objective function. Such parameters -- defined prior to model training -- are referred to as \textbf{hyperparameters}.

The value of the hyperparameter $K$ determines the complexity of the polynomial regression model. Figure~\ref{fig:lr:sine_mpoly} illustrates how the fitted curves obtained via least squares change as the model complexity increases. The phenomenon can be summarized as follows:
\begin{itemize}
    \item \textbf{Underfitting (e.g., $K=1$).}
    When $K=1$, the model reduces to simple linear regression. As shown in Figure~\ref{fig:lr:sine_mpoly_a}, such a low-capacity model cannot capture the nonlinear structure of the sinusoidal data. The fitted line exhibits a large systematic deviation from the underlying pattern, resulting in high error on both the training points and the true data-generating process.

    \item \textbf{Appropriate fitting (moderate $K$).}
    For intermediate polynomial degrees (e.g., $K=3$ or $K=4$), the model becomes expressive enough to represent the smooth curvature of the sinusoidal function while still avoiding unnecessary fluctuations. In this regime, the fitted curve follows the overall trend of the data without chasing noise, achieving a desirable balance between bias and variance.

    \item \textbf{Overfitting (e.g., $K=9$).}
    When $K=9$, the model becomes overly flexible relative to the very small dataset ($M=10$). A ninth-degree polynomial can interpolate all data points exactly, producing a curve that passes through every sample. However, as shown in Figure~\ref{fig:lr:sine_mpoly_c}, the resulting function oscillates wildly and deviates severely from the true underlying distribution. This behavior reflects classic overfitting: extremely low training error but poor generalization.
\end{itemize}

\begin{figure}[t]
    \centering
    \subfigure[Underfitting]{\includegraphics[width=0.3\linewidth]{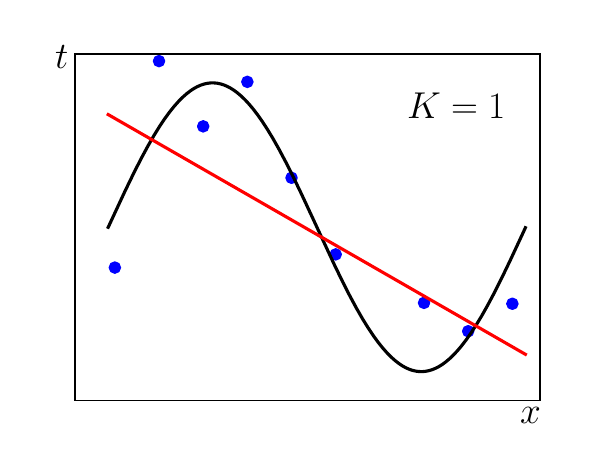}\label{fig:lr:sine_mpoly_a}}~~
    \subfigure[Appropriate fitting]{\includegraphics[width=0.3\linewidth]{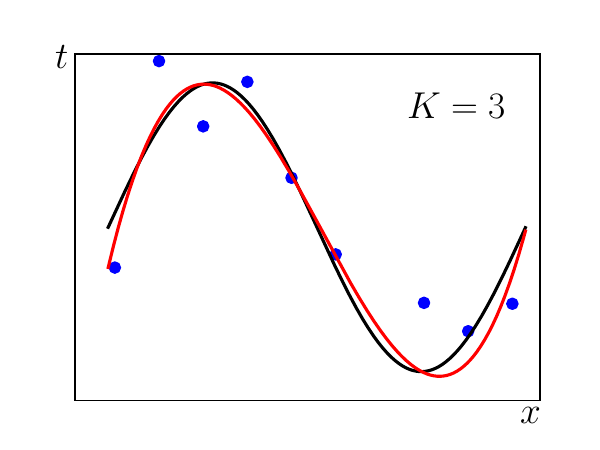}\label{fig:lr:sine_mpoly_b}}
    \subfigure[Overfitting]{\includegraphics[width=0.3\linewidth]{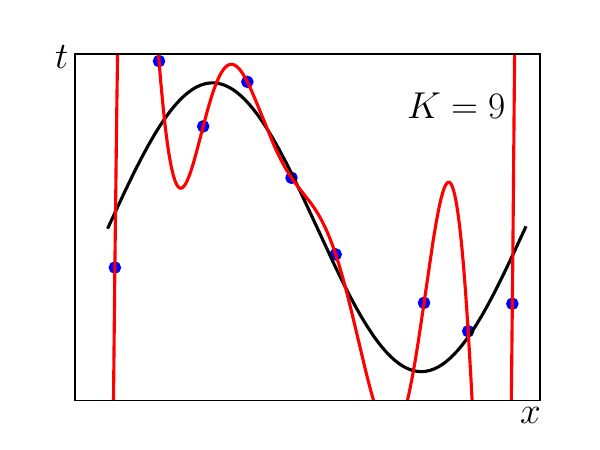}\label{fig:lr:sine_mpoly_c}}
    \caption{\small As the model complexity increases, the fitted polynomial model transitions from underfitting to overfitting.}
    \label{fig:lr:sine_mpoly}
\end{figure}

Figure~\ref{fig:lr:sine_residual} shows the performance of the polynomial regression model on both the training set (training error) and the test set (generalization error) for different choices of $K$. When $K < 3$, the generalization error decreases as the training error decreases, indicating that the model is in the underfitting regime. When $K > 8$, although the training error drops to zero, the generalization error rises sharply. This sudden deterioration reflects overfitting caused by an overly complex model. The overall trend in Figure~\ref{fig:lr:sine_residual} is fully consistent with the conceptual illustration in Figure~\ref{fig:lr:k28_2}.

\begin{figure}[t]
    \centering
    \includegraphics[width=0.5\columnwidth]{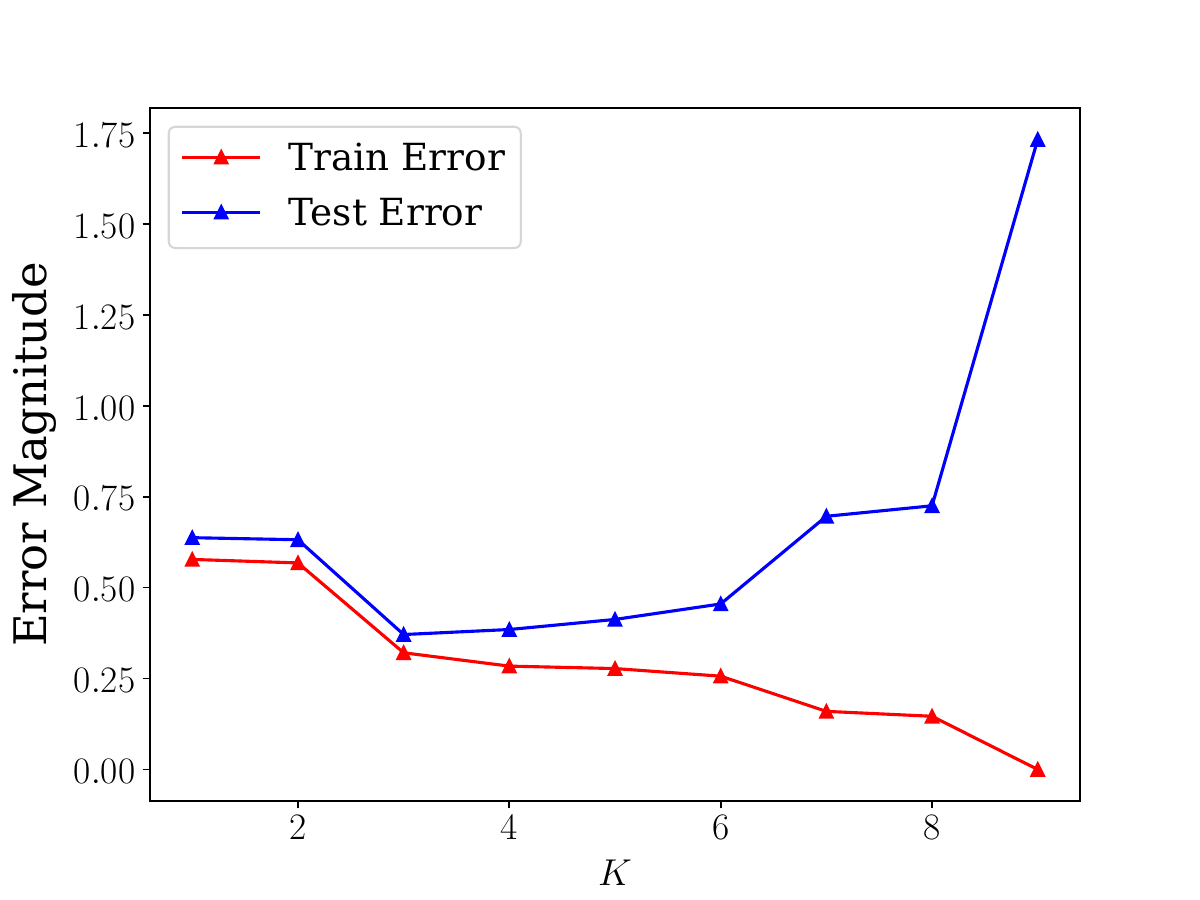}
    \caption{\small Behavior of training and test errors under different model complexities, where the polynomial degree $K$ controls the capacity of the model.}
    \label{fig:lr:sine_residual}
\end{figure}

\paratitle{Approaches to Address Underfitting and Overfitting.} As illustrated in the previous examples, underfitting is relatively straightforward to remedy. The key is to increase the complexity of the model so that it has sufficient capacity to capture the underlying structure of the data. The linear basis function models introduced earlier in this chapter achieve this by expanding the input variable into a higher-dimensional feature space, thereby increasing the expressive power of the regression model and mitigating underfitting.

In contrast, resolving overfitting is more challenging. Here, we introduce three commonly used approaches.

\textbf{Method 1: Increasing the amount of training data.} Increasing the size of the training dataset provides the model with a fuller and more stable picture of the underlying data-generating process. This reduces the chance that the model will latch onto noise or accidental patterns that arise when only a small number of samples are available. Figure~\ref{fig:lr:sine_npoly} illustrates the behavior of a polynomial regression model with $K=9$ when the number of training samples is increased from $M=15$ to $M=100$. When $M=15$, the model already avoids the severe overfitting observed when only 10 samples are used. When $M=100$, the model fits the underlying sinusoidal function remarkably well, indicating that sufficient data can effectively prevent overfitting even for high-capacity models.

\begin{figure}[t]
    \centering
    \subfigure{\includegraphics[width=0.4\linewidth]{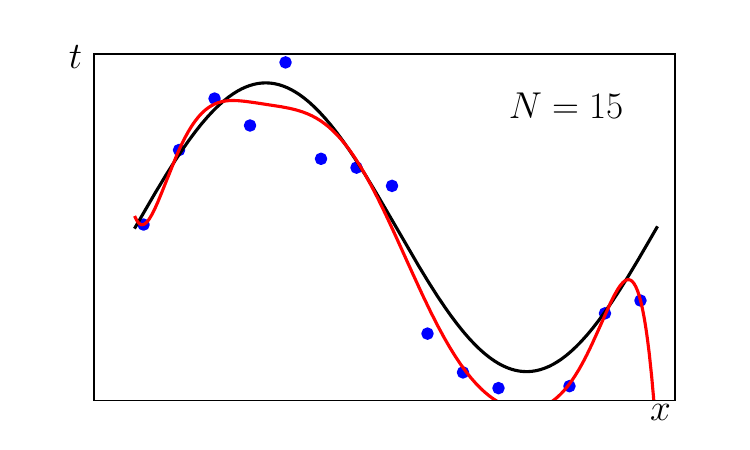}\label{fig:lr:sine_npoly_a}}~~
    \subfigure{\includegraphics[width=0.4\linewidth]{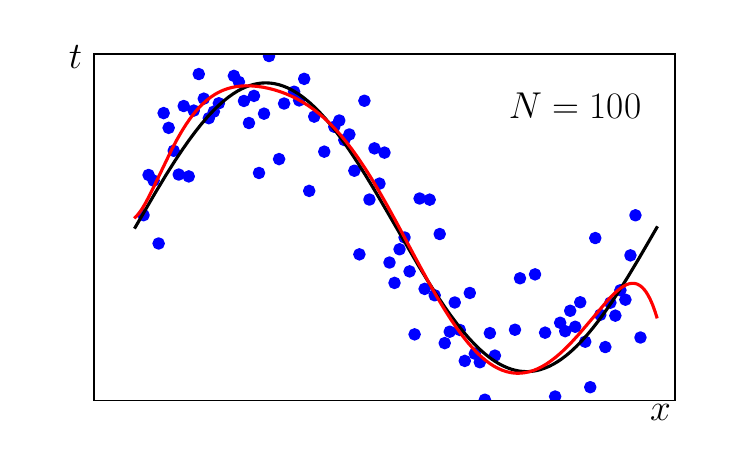}\label{fig:lr:sine_npoly_b}}
    \caption{\small Preventing overfitting by increasing the amount of training data.}
    \label{fig:lr:sine_npoly}
\end{figure}

\textbf{Method 2: Cross-validation.} In many practical scenarios, increasing the amount of training data is difficult or even infeasible. Cross-validation therefore provides an effective alternative for tuning the hyperparameters of a regression model -- such as the polynomial degree $K$ -- so as to strike a balance between model complexity and generalization performance. In {cross-validation}, the original training dataset is further partitioned into two subsets: one subset is used for training the model parameters~$\boldsymbol{\theta}$, and the other subset serves as a \textbf{validation set} for selecting hyperparameters. The procedure typically involves the following steps:
\begin{enumerate}
  \item Specify several candidate values for the hyperparameter (say, $N$ choices).
  \item Train $N$ models using the training subset, each corresponding to a different hyperparameter value.
  \item Evaluate their generalization errors on the validation set.
  \item Plot the relationship between the hyperparameter and the generalization error (as in Figures~\ref{fig:lr:sine_mpoly} and \ref{fig:lr:sine_residual}) and select the hyperparameter value that yields the smallest validation error.
\end{enumerate}

This procedure is known as \textbf{simple cross-validation}. More advanced variants include \textbf{leave-one-out cross-validation (LOOCV)}~\cite{lachenbruch1968estimation} and \textbf{$k$-fold cross-validation}~\cite{geisser1975predictive}, which provide more stable or more computationally efficient estimates of the generalization error.

In addition, regularization is another important technique for preventing overfitting, and we will discuss it in detail in the next subsection.

\subsection{Regularization Terms and Ridge/LASSO Regression}

In statistics and machine learning, the term \textbf{regularization} refers to the process of enforcing ``simpler'' solutions through certain techniques. Regularization is commonly used either to obtain stable solutions to ill-posed problems or to prevent model overfitting. The theoretical foundation behind regularization is the principle of \textbf{Occam's Razor}, which states that ``entities should not be multiplied beyond necessity,'' or in essence, simpler models are preferable when possible. Regularization can be implemented in two primary ways:
\begin{itemize}
    \item \textbf{Explicit regularization}. This refers to approaches that incorporate a {\em regularization term} directly into the optimization objective. Such terms are often interpreted as priors, penalties, or constraints.
    \item \textbf{Implicit regularization}. This includes all forms of regularization not explicitly added to the loss function, such as early stopping in gradient-based optimization or discarding outliers. Implicit regularization techniques are widely used in training deep neural networks.
\end{itemize}
In this section, we focus primarily on explicit regularization methods in regression models, namely the introduction of {\bf Regularization Terms}.

\paratitle{Regularization Terms and Structural Risk Minimization.}
Introducing a regularization term is one of the most fundamental ways to apply explicit regularization in regression models. The core idea is to modify the original loss function by adding a term that penalizes overly complex models. The regularized formulation is expressed as
\begin{equation}
    \label{eq:lr:k25_1}
    \mathcal{J}(\boldsymbol{\theta})
    = \mathcal{L}(\boldsymbol{\theta}; \boldsymbol{X}, \boldsymbol{y})
    + \lambda \, \Omega(\boldsymbol{\theta}).
\end{equation}
Here, $\mathcal{L}(\boldsymbol{\theta}; \boldsymbol{X}, \boldsymbol{y})$ is the loss function, measuring the discrepancy between the model's predictions and the observed data. The term $\Omega(\boldsymbol{\theta})$ is a complexity function that evaluates how ``complicated'' the model is, typically by measuring the magnitude or structure of the parameter vector~$\boldsymbol{\theta}$. Larger values of $\Omega(\boldsymbol{\theta})$ correspond to more complex models. The coefficient $\lambda$ controls the trade-off between fitting the data and suppressing unnecessary complexity, and is usually treated as a hyperparameter to be tuned.

Once the regularization term is added, parameter estimation becomes the problem of minimizing the new objective:
\begin{equation}\label{eq:objective_function}
    \boldsymbol{\theta}^*
    = \underset{\boldsymbol{\theta}}{\arg \min}\, \mathcal{J}(\boldsymbol{\theta}).
\end{equation}
Because this objective now incorporates both data fidelity and model complexity, $\mathcal{J}(\boldsymbol{\theta})$ is commonly referred to as the \textbf{objective function}. In contrast, the original loss function $\mathcal{L}(\boldsymbol{\theta}; \boldsymbol{X}, \boldsymbol{y})$ is often called the \textbf{empirical loss} or \textbf{empirical risk}. The objective function minimization in Eq.~\eqref{eq:objective_function} operationalizes Occam's Razor by quantitatively encoding a preference for simplicity through the complexity penalty $\Omega(\boldsymbol{\theta})$. The central premise is that, among all models that explain the data adequately, simpler models should be favored.

From the perspective of learning theory, regularization alters the fundamental principle guiding model selection. When no regularization is applied, learning follows the paradigm of \textbf{Empirical Risk Minimization (ERM)}, which seeks a model that minimizes the predictive error on the training dataset. While ERM works well when data are abundant and noise-free, it easily leads to overfitting in high-dimensional or noisy settings. With regularization, the learning objective shifts to \textbf{Structural Risk Minimization (SRM)}. Under SRM, the learner considers not only how well the model fits the training data but also how complex the model is. The term ``structure'' refers to the set of hypotheses ordered by their complexity. By minimizing both empirical risk and the complexity penalty, SRM provides a principled and theoretically grounded framework for balancing bias and variance, leading to models with improved generalization performance on unseen data.

\paratitle{Ridge Regression and LASSO Regression (L2 and L1 Regularization).}
In regression models, the \textbf{norm} of the parameter vector is one of the most commonly used regularization terms for quantifying model complexity. Given a parameter vector $\boldsymbol{\theta} = (\theta_1, \ldots, \theta_n)^\top$, its $p$-norm, denoted $L_p$, is defined as
\begin{equation}
    \label{eq:lr:1.59}
    \Vert \boldsymbol{\theta} \Vert_p = \left( \sum_i |\theta_i|^p \right)^{\frac{1}{p}},
\end{equation}
where $p \in \mathbb{R}$ and $p \ge 1$. $L_2$ and $L_1$ norms are the two most widely used norms in regularization. Their corresponding regression methods can be summarized as follows:
\begin{itemize}
    \item \textbf{Ridge Regression ($L_2$ Regularization).}
    When the $L_2$ norm of the parameter vector is used as the regularization term, the resulting method is known as \emph{Ridge Regression}. The Ridge objective function is given by
    \begin{equation}
        \label{eq:lr:k25_2}
        \mathcal{J}(\boldsymbol{\theta})
        = \frac{1}{M} \sum_{m=1}^M \Big(\hat{y}_m(\boldsymbol{\theta}) - {y}_m\Big)^2
        + \lambda \, \Vert \boldsymbol{\theta} \Vert_2^2.
    \end{equation}
    Ridge Regression is one of the earliest and most influential regularization techniques in statistics. It was first introduced by Hoerl and Kennard in the 1970s as a remedy for multicollinearity in linear regression~\cite{hoerl1970ridge}. When the input variables are highly correlated, the ordinary least squares (OLS) estimator becomes unstable, and small perturbations in the data can lead to large fluctuations in the estimated parameters. Ridge Regression tackles this issue by penalizing the $L_2$ norm of the parameter vector, shrinking the coefficients toward zero and thereby stabilizing the solution.

    \item \textbf{LASSO Regression ($L_1$ Regularization).}
    When the $L_1$ norm of the parameter vector is used as the regularization term, the resulting method is known as \emph{LASSO} (Least Absolute Shrinkage and Selection Operator). The LASSO objective function is
    \begin{equation}
        \label{eq:lr:k25_3}
        \mathcal{J}(\boldsymbol{\theta})
        = \frac{1}{M} \sum_{m=1}^M \Big(\hat{y}_m(\boldsymbol{\theta}) - {y}_m\Big)^2
        + \lambda \, \Vert \boldsymbol{\theta} \Vert_1.
    \end{equation}
    LASSO was proposed much later by Tibshirani in 1996~\cite{tibshirani1996regression}, motivated by the need for a method that not only stabilizes estimation but also performs automatic feature selection. By penalizing the $L_1$ norm of the parameter vector, LASSO induces sparsity: it drives many coefficients exactly to zero, yielding simpler and more interpretable models. This sparsity property distinguishes LASSO from Ridge Regression and makes it especially suitable for high-dimensional settings where the number of features may exceed the number of samples.
\end{itemize}

\paratitle{Interpreting the Role of Regularization Terms.} The $p$-norm can be interpreted as the ``distance'' from the point $\boldsymbol{\theta}$ to the origin in parameter space. This interpretation provides an intuitive understanding of why $\Vert \boldsymbol{\theta} \Vert_p$ serves as a measure of model complexity. For the linear basis function model in Eq.~\eqref{eq:lr:k23_1}, each parameter $\theta_k$ determines the contribution of a corresponding basis function. If the value of $\Vert \boldsymbol{\theta} \Vert_p$ is large, it indicates that many coefficients have large magnitudes, meaning that the model places strong emphasis on multiple basis functions. Such models tend to exhibit complex behaviors -- sharp bends, large oscillations, or highly flexible functional forms -- making them more prone to overfitting. Conversely, when $\Vert \boldsymbol{\theta} \Vert_p$ is small, the coefficients are either collectively small in magnitude or many are exactly zero (especially under $L_1$ regularization). In this situation, the resulting model behaves more smoothly and has a simpler functional structure. A small-norm solution therefore corresponds to a model that captures the essential patterns in the data while avoiding unnecessary complexity. This explains why controlling $\Vert \boldsymbol{\theta} \Vert_p$ is a principled and effective way to regularize regression models.

We continue using the polynomial regression example with sinusoidal data to illustrate the effect of regularization (Figure~\ref{fig:lr:sine_mpoly}). Table~\ref{tab:lr:k29_1} lists the estimated coefficients for different degrees $K$. When $K=9$, the magnitudes of the parameters become extremely large. Intuitively, if we constrain the size of such parameters -- for example, by introducing a LASSO or Ridge regularization term -- we can effectively control the complexity of the model and mitigate overfitting.

\begin{table}[t]\small
\caption{\small Estimated parameters of polynomial regression models with different degrees $K$.}
\label{tab:lr:k29_1}
\centering
\begin{tabular}{r|rrrr}
& $K=0$ & $K=1$ & $K=3$ & $K=9$ \\
\hline
$w_0^*$ & -0.04 & 0.78 & -0.28 & -6.18 \\
$w_1^*$ &        & -0.26 & 1.95 & 78.45 \\
$w_2^*$ &        &        & -0.87 & -221.32 \\
$w_3^*$ &        &        & 0.09 & 285.53 \\
$w_4^*$ &        &        &      & -203.44 \\
$w_5^*$ &        &        &      & 86.65 \\
$w_6^*$ &        &        &      & -22.61 \\
$w_7^*$ &        &        &      & 3.54 \\
$w_8^*$ &        &        &      & -0.30 \\
$w_9^*$ &        &        &      & 0.01
\end{tabular}
\end{table}

To see the effect of regularization more clearly, we add an $L_1$ regularization term (LASSO) to the polynomial regression model in Eq.~\eqref{eq:lr:k29_1} and vary the value of the regularization coefficient~$\lambda$ (see Eq.~\eqref{eq:lr:k25_2}). The resulting fitted curves are shown in Figure~\ref{fig:lr:sine_aploy}. As seen from Figure~\ref{fig:lr:sine_apoly_a}, when $\lambda$ is chosen appropriately, even a highly flexible model with $K=9$ can produce a smooth curve that generalizes well to the underlying sinusoidal function. However, Figure~\ref{fig:lr:sine_aploy_b} also demonstrates that if $\lambda$ is set too large, the model becomes overly restricted, leading to underfitting. In this sense, $\lambda$ itself acts as a hyperparameter, and selecting an appropriate value typically requires cross-validation.

\begin{figure}[t]\small
    \centering
    \subfigure[Moderate $\lambda$]{
    \includegraphics[width=0.4\linewidth]{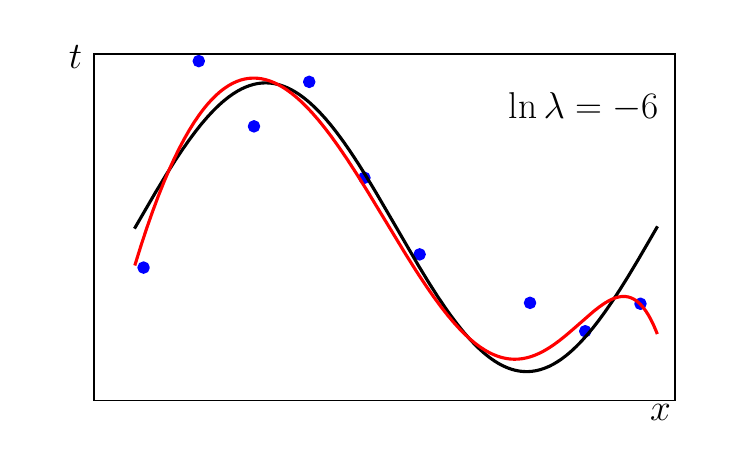}
    \label{fig:lr:sine_apoly_a}
    }
    \subfigure[Overly large $\lambda$]{
    \includegraphics[width=0.4\linewidth]{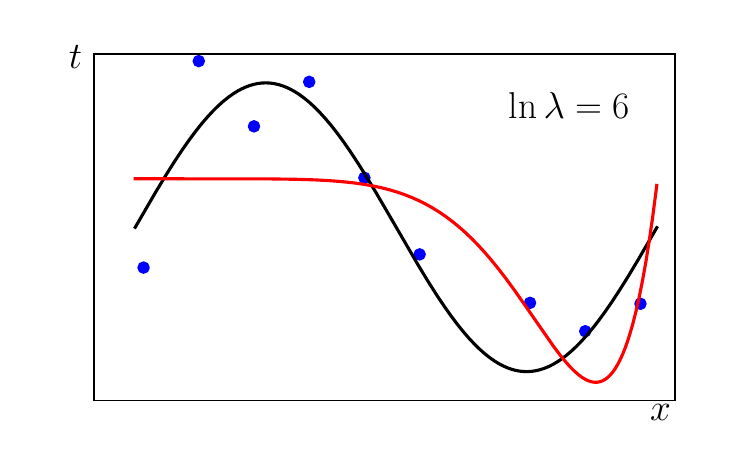}
    \label{fig:lr:sine_aploy_b}
    }
    \caption{\small Preventing overfitting through regularization.}
    \label{fig:lr:sine_aploy}
\end{figure}

Table~\ref{tab:lr:k29_2} further illustrates how the regression coefficients change under different levels of regularization. When no regularization is applied ($\ln \lambda = -\infty$), the coefficients are extremely large, reflecting the severe overfitting observed earlier. As the regularization strength increases ($\ln \lambda = -6$), most coefficients shrink substantially. When $\lambda$ becomes very large ($\ln \lambda = 6$), nearly all coefficients are driven close to zero. These results show that introducing a regularization term dramatically reduces parameter magnitudes, effectively limiting model complexity and preventing overfitting.

\begin{table}[t] \small
\caption{\small Changes in polynomial regression coefficients under different regularization strengths.}
\label{tab:lr:k29_2}
\centering
\begin{tabular}{r|rrr}
& $\ln \lambda=-\infty$ & $\ln \lambda=-6$ & $\ln \lambda=6$ \\
\hline
$w_0^*$ & -6.18 & -0.25 & 0.33 \\
$w_1^*$ & 78.45 & 1.90 & 0.00 \\
$w_2^*$ & -221.32 & -0.75 & 0.00 \\
$w_3^*$ & 285.53 & 0.02 & 0.00 \\
$w_4^*$ & -203.44 & 0.00 & 0.00 \\
$w_5^*$ & 86.65 & 0.00 & 0.00 \\
$w_6^*$ & -22.61 & 0.00 & 0.00 \\
$w_7^*$ & 3.54 & 0.00 & -4.09 $\times 10^{-5}$ \\
$w_8^*$ & -0.30 & 0.00 & 5.52 $\times 10^{-7}$ \\
$w_9^*$ & 0.01 & 0.00 & 9.46 $\times 10^{-7}$
\end{tabular}
\end{table}

\subsection{Parameter Estimation for Ridge/LASSO Regression}

When Ridge Regression and LASSO Regression were first introduced, both methods were formulated by augmenting the linear regression's ordinary least squares objective with $L_2$ and $L_1$ regularization terms, respectively~\cite{hoerl1970ridge, tibshirani1996regression}. These regularization schemes were later generalized to a wide range of regression models. In principle, any regression-type model -- such as the linear regression, logistic regression, Softmax regression, and linear basis function models discussed earlier, as well as neural network models trained with squared-error or cross-entropy loss -- can incorporate $L_2$ or $L_1$ penalties to achieve structural risk minimization. For this reason, $L_2$ and $L_1$ penalties are often referred to more broadly as \emph{Ridge} and \emph{LASSO} regularization. In this subsection, we focus on the case of linear regression and examine the properties of the closed-form solutions for Ridge and LASSO Regression. We then extend the discussion to the gradient descent framework and analyze how the two regularization terms influence the parameter update dynamics. This analysis provides insight into the different behaviors induced by $L_2$ and $L_1$ penalties.

\paratitle{Closed-form Solution of Ridge Regression.} For Ridge Regression based on the linear regression model, the objective function can be written in matrix form as
\begin{equation}
    \mathcal{J}(\boldsymbol{\theta})
    = \sum_{m=1}^{M} \big(\boldsymbol{\theta}^{\top}{\boldsymbol{x}}_m - y_m \big)^2
      + \lambda \Vert \boldsymbol{\theta} \Vert_2^2
    = ({\boldsymbol{X}}\boldsymbol{\theta} - \boldsymbol{y})^{\top}
      ({\boldsymbol{X}}\boldsymbol{\theta} - \boldsymbol{y})
      + \lambda \boldsymbol{\theta}^{\top}\boldsymbol{\theta}.
\end{equation}
For convenience, the coefficient \( \frac{1}{M} \) in the loss function is omitted.
Taking the gradient with respect to \( \boldsymbol{\theta} \) gives
\begin{equation}
    \frac{\partial \mathcal{J}(\boldsymbol{\theta})}{\partial \boldsymbol{\theta}}
    = 2{\boldsymbol{X}}^\top {\boldsymbol{X}}\boldsymbol{\theta}
      - 2{\boldsymbol{X}}^{\top}\boldsymbol{y}
      + 2\lambda \boldsymbol{\theta}.
\end{equation}
Setting the gradient to zero yields
\begin{equation}
    \label{eq:lr:1.68}
    \boldsymbol{\theta}^*
    = \left({\boldsymbol{X}}^\top {\boldsymbol{X}} + \lambda \boldsymbol{I}\right)^{-1}
      {\boldsymbol{X}}^{\top}\boldsymbol{y},
\end{equation}
which is the closed-form solution of Ridge Regression, where \( \boldsymbol{I} \) denotes the identity matrix.
Since the regularization coefficient \( \lambda \) is strictly positive, the matrix
\( {\boldsymbol{X}}^\top {\boldsymbol{X}} + \lambda \boldsymbol{I} \)
is guaranteed to be invertible.

Comparing Eq.~\eqref{eq:lr:1.68}, the closed-form solution of Ridge Regression, with the ordinary least squares solution in Eq.~\eqref{eq:lr:theta_star}, we see that Ridge Regression introduces an additional term \( \lambda\boldsymbol{I} \) to the denominator matrix. As a result, for the same dataset, Ridge Regression produces coefficient estimates with smaller magnitudes than ordinary least squares, thereby achieving effective control over model complexity.

\paratitle{A Formal Analytic Expression for LASSO Regression.} In LASSO regression, the presence of the $L_1$ norm makes the objective function non-differentiable at $\boldsymbol{\theta}=\boldsymbol{0}$. A rigorous treatment therefore relies on subgradient or KKT conditions, and in general LASSO does not admit a closed-form solution. However, to gain intuition about the shrinkage effect of the $L_1$ penalty, we can examine a local analytic form of the solution under a fixed sign pattern.

Specifically, we consider the case where $\boldsymbol{\theta}\neq\boldsymbol{0}$ and all coordinates stay away from zero so that their signs are well-defined and constant. Under this assumption, the LASSO objective can be written as
\begin{equation}
    \label{eq:lr:k25_4}
    \begin{aligned}
    \mathcal{J}(\boldsymbol{\theta})
    &= \sum_{m=1}^M\big(\boldsymbol{\theta}^{\top} {\boldsymbol{x}}_m - y_m\big)^2
       + \lambda \sum_{n=1}^N \operatorname{sign}(\theta_n)\,\theta_n \\
    &= \Big({\boldsymbol{X}}\boldsymbol{\theta} - \boldsymbol{y}\Big)^{\top}
       \Big({\boldsymbol{X}}\boldsymbol{\theta} - \boldsymbol{y}\Big)
       + \lambda \sum_{n=1}^N \operatorname{sign}(\theta_n)\,\theta_n,
    \end{aligned}
\end{equation}
where $\operatorname{sign}(\theta_n)\in\{+1,-1\}$ denotes the sign of $\theta_n$, and
$\operatorname{sign}(\boldsymbol{\theta})$ is the vector of signs of all components. Taking the gradient with respect to $\boldsymbol{\theta}$ (under the fixed-sign assumption) yields
\begin{equation}
    \label{eq:lr:k25_5}
    \frac{\partial \mathcal{J}(\boldsymbol{\theta})}{\partial \boldsymbol{\theta}}
    = 2 {\boldsymbol{X}}^{\top} {\boldsymbol{X}} \boldsymbol{\theta}
      - 2 {\boldsymbol{X}}^{\top} \boldsymbol{y}
      + \lambda \,\operatorname{sign}(\boldsymbol{\theta}) .
\end{equation}
Setting this gradient to zero leads to the following \emph{formal} expression:
\begin{equation}
    \label{eq:lr:k25_6}
    \boldsymbol{\theta}^*
    = \left({\boldsymbol{X}}^{\top} {\boldsymbol{X}}\right)^{-1}
      \left({\boldsymbol{X}}^{\top} \boldsymbol{y}
      - \frac{1}{2} \lambda \,\operatorname{sign}(\boldsymbol{\theta})\right).
\end{equation}

It is important to emphasize that Eq.~\eqref{eq:lr:k25_6} is \emph{not} a genuine closed-form solution for LASSO, because the right-hand side still depends on $\operatorname{sign}(\boldsymbol{\theta})$, which itself is determined by $\boldsymbol{\theta}$. Thus, Eq.~\eqref{eq:lr:k25_6} should be viewed as an implicit, local analytic form valid under a fixed sign pattern, rather than a globally applicable closed-form estimator. In practice, LASSO solutions are typically obtained via numerical optimization methods such as coordinate descent. Nevertheless, comparing Eq.~\eqref{eq:lr:k25_6} with the ordinary least squares solution in Eq.~\eqref{eq:lr:theta_star} still reveals an important qualitative property. The additional term
\[
- \frac{1}{2} \lambda \,\operatorname{sign}(\boldsymbol{\theta})
\]
acts componentwise: when $\theta_n>0$, we have $-\operatorname{sign}(\theta_n)<0$, and when $\theta_n<0$, we have $-\operatorname{sign}(\theta_n)>0$. In both cases, each parameter is ``pulled'' toward zero, so the LASSO estimates tend to have smaller magnitudes than the ordinary least squares estimates. This shrinkage toward zero is a key mechanism by which LASSO controls model complexity and, in the full subgradient/KKT treatment, also leads to sparse solutions with many coefficients exactly equal to zero.

\paratitle{Gradient Descent Solutions for Ridge and LASSO.} We now analyze the behavior of Ridge and LASSO regression from the perspective of gradient descent, which provides an intuitive view of how the two regularization terms influence parameter updates. Since the properties discussed here apply broadly to different regression models, we express the two objective functions using the following general forms:
\begin{equation*}
    \mathcal{J}_{L_2}(\boldsymbol{\theta})
    = \mathcal{L}(\boldsymbol{\theta}; \boldsymbol{X}, \boldsymbol{y})
      + \lambda \Vert \boldsymbol{\theta} \Vert_2^2,
    \qquad
    \mathcal{J}_{L_1}(\boldsymbol{\theta})
    =\mathcal{L}(\boldsymbol{\theta}; \boldsymbol{X}, \boldsymbol{y})
      + \lambda \Vert \boldsymbol{\theta} \Vert_1.
\end{equation*}
Correspondingly, the coordinate-wise update rule for each parameter $\theta_n$ during gradient descent takes the form:
\begin{equation}
    \label{eq:lr:k25_7}
    \begin{aligned}
        &\text{NoRegular: }
        && \theta_n^{\text{(new)}} \leftarrow \theta_n - \alpha\,\Delta, \\[4pt]
        &\text{Ridge: }
        && \theta_n^{\text{(new)}} \leftarrow \theta_n - \alpha\,\Delta - \alpha\lambda\,\theta_n, \\[4pt]
        &\text{LASSO: }
        && \theta_n^{\text{(new)}} \leftarrow \theta_n - \alpha\,\Delta - \alpha\lambda\,\operatorname{sign}(\theta_n),
    \end{aligned}
\end{equation}
where
\[
\Delta = \frac{\partial \mathcal{L}(\boldsymbol{\theta}; \boldsymbol{X}, \boldsymbol{y})}{\partial \theta_n},
\]
and $\alpha$ is the learning rate.
For comparison, we include the update rule for the unregularized model (NoRegular).

From Eq.~\eqref{eq:lr:k25_7}, we observe that both Ridge and LASSO introduce an additional ``penalty'' term -- $\lambda\theta_n$ for Ridge and $\lambda\operatorname{sign}(\theta_n)$ for LASSO -- into each gradient update. As a result, given the same number of iterations, the parameters obtained by Ridge or LASSO will have smaller magnitudes than those obtained without regularization. This observation is fully consistent with the closed-form analyses discussed earlier, but importantly, it does not depend on the specific form of the regression function and loss function $\mathcal{L}(\boldsymbol{\theta}; \boldsymbol{X}, \boldsymbol{y})$, making the conclusion applicable to all kinds of regression models. In general, both Ridge and LASSO regularization shrink parameter magnitudes.

A closer examination of the update rules in Eq.~\eqref{eq:lr:k25_7} reveals distinct behaviors induced by the two regularization terms:
\begin{itemize}
    \item \textbf{Ridge Regularization.}
    The penalty term $\lambda\theta_n$ is proportional to the magnitude of the parameter. Larger parameters receive larger penalties and shrink more aggressively, while smaller parameters are penalized less. This proportional shrinkage encourages the parameters to approach similar magnitudes. Ridge therefore tends to produce \emph{smooth, non-sparse} solutions.

    \item \textbf{LASSO Regularization.}
    The penalty term $\lambda\,\operatorname{sign}(\theta_n)$ has a constant magnitude, regardless of the size of $\theta_n$. This means that every parameter receives the same force pushing it toward zero. Under such dynamics, small parameters may be shrunk all the way to zero, while larger parameters shrink more moderately. This behavior leads to \emph{sparsity}, with many coefficients becoming exactly zero.
\end{itemize}

The sparsity-inducing property of LASSO makes it particularly useful for feature selection. We will discuss this property in greater detail in the next subsection.

\subsection{Feature Selection Effect of LASSO}

LASSO regression and Ridge regression are both effective for reducing model complexity and minimizing structural risk. However, a key distinction lies in the fact that \textbf{LASSO can perform feature selection}, whereas Ridge regression cannot. This is because the $L_2$ norm encourages coefficient values to be spread out more evenly, resulting in dense solutions in which most coefficients remain nonzero. In contrast, the $L_1$ norm encourages sparsity by driving many coefficients exactly to zero, thereby yielding a model that relies only on a subset of the original features. We have already gained some intuition about this behavior from the gradient-descent updates of Ridge and LASSO. In this subsection, we provide a deeper and more geometric explanation of why LASSO produces sparse solutions.

\paratitle{Geometric Principle Behind LASSO Feature Selection.}
Recall the formulation of the regularized objective function in Eq.~\eqref{eq:lr:k25_1}:
\begin{equation}
    \label{eq:lr:k25_8}
    \mathcal{J}(\boldsymbol{\theta})
    = \mathcal{L}(\boldsymbol{\theta}; \boldsymbol{X}, \boldsymbol{y})
      + \lambda\,\Omega(\boldsymbol{\theta}).
\end{equation}
Any optimal solution (including local optima in non-convex cases) must satisfy the first-order optimality condition
\[
\frac{\partial \mathcal{J}(\boldsymbol{\theta})}{\partial \boldsymbol{\theta}} = 0,
\]
which leads to
\begin{equation}
    \label{eq:lr:k25_9}
    \frac{\partial \mathcal{L}(\boldsymbol{\theta}; \boldsymbol{X}, \boldsymbol{y})}{\partial \boldsymbol{\theta}}
    = -\lambda\,\frac{\partial \Omega(\boldsymbol{\theta})}{\partial \boldsymbol{\theta}}.
\end{equation}
Geometrically, Eq.~\eqref{eq:lr:k25_9} states that at an optimum, the level sets of the loss function $\mathcal{L}$ and the regularizer $\Omega$ \emph{must be tangent}. This is because the gradients of $\mathcal{L}$ and $\Omega$ point in exactly opposite directions at the optimum, while the level sets of a differentiable function are orthogonal to its gradient.\footnote{At the optimum, the gradients of $\mathcal{L}$ and $\Omega$ cancel each other. Since the gradient is perpendicular to the contour line, their contour lines must be tangent at the point of contact.}

\begin{figure}[t]
    \centering
    \includegraphics[width=0.65\columnwidth]{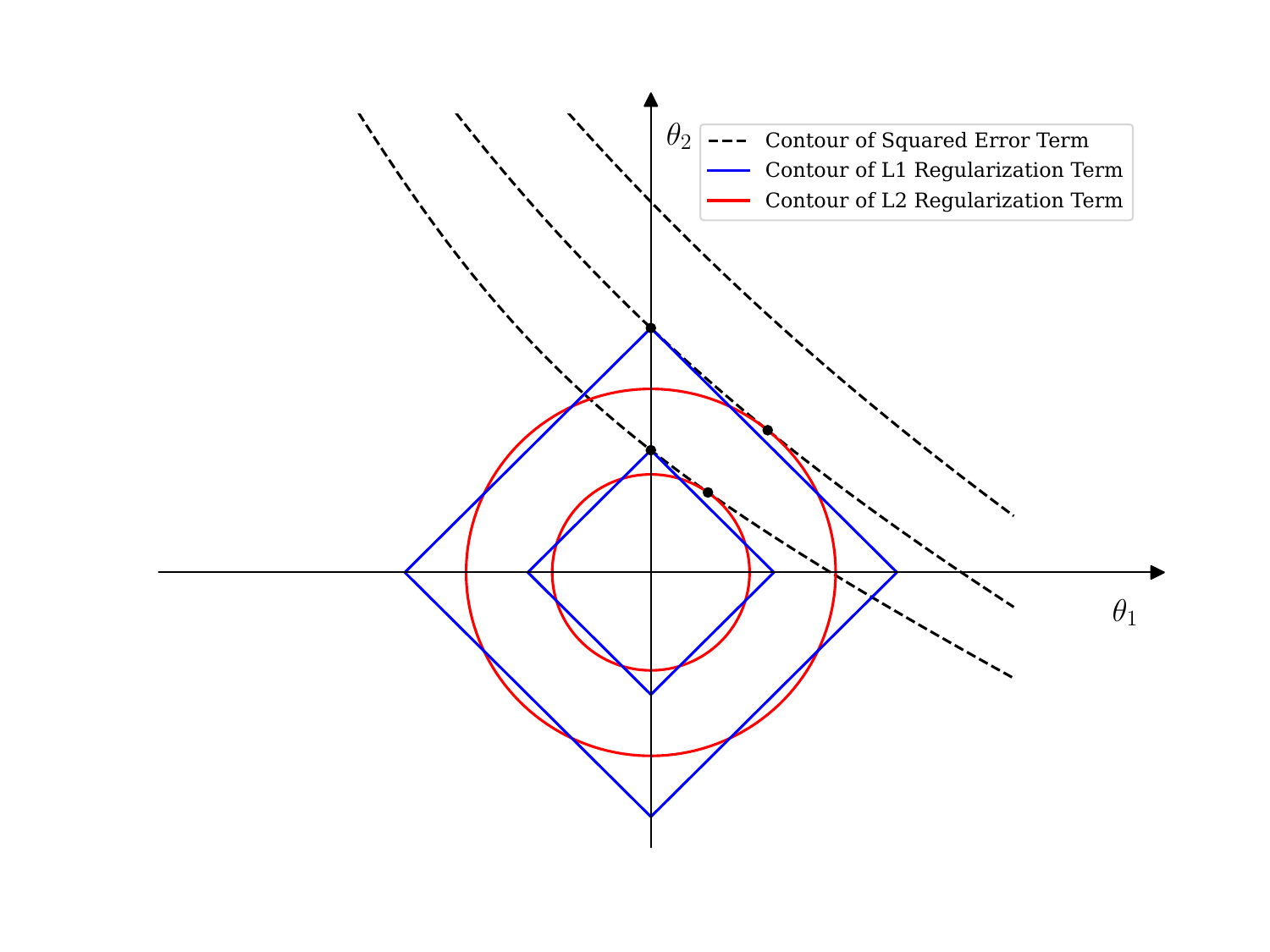}
    \caption{\small Illustration of why $L_1$ regularization yields sparse solutions more frequently than $L_2$ regularization.}
    \label{fig:lr:feat_sel}
\end{figure}

Since Ridge and LASSO use the $L_2$ and $L_1$ norms respectively as their regularization functions, the shapes of their level sets differ drastically, which leads to very different behaviors of the optimal solution. Figure~\ref{fig:lr:feat_sel} provides an intuitive illustration in a two-dimensional parameter space. As shown in Figure~\ref{fig:lr:feat_sel}, the level sets of the $L_2$ norm are circles (or hyperspheres in higher dimensions). Consequently, the tangency point between the loss function and the $L_2$ constraint can occur at any point on the circle. In contrast, the level sets of the $L_1$ norm form diamonds (or cross-polytopes in higher dimensions). Because of the sharp corners of the diamond, the tangency between the loss function and the $L_1$ constraint is highly likely to occur at one of these corners -- that is, along one of the coordinate axes. A tangency point lying on a coordinate axis corresponds to an optimal solution in which at least one component of $\boldsymbol{\theta}$ is exactly zero. Geometrically, this explains why LASSO produces sparse solutions.

In other words, under an $L_1$ regularization constraint, the optimal solution $\boldsymbol{\theta}^*$ naturally tends to have only a small number of nonzero entries corresponding to the most important features. In many regression models, such as linear regression, logistic regression and Softmax regression, the parameters $\boldsymbol{\theta}$ determine how the features are linearly combined. When a coefficient $\theta_n$ becomes zero, the corresponding feature $x_n$ has no effect on predicting the target variable $y$ and can be removed. The remaining nonzero coefficients indicate the selected ``useful features.'' LASSO achieves feature selection precisely through this sparsity-inducing property of the $L_1$ norm.

Indeed, the full name of LASSO is \emph{Least Absolute Shrinkage and Selection Operator}, where the word ``Selection'' explicitly refers to this unique ability to perform feature selection. Strictly speaking, $L_p$ penalties with $0 < p < 1$ also promote sparsity and can even induce stronger sparsity than the $L_1$ norm. However, for $0<p<1$, the resulting optimization problem becomes non-convex and generally difficult to solve reliably, often containing many local minima. For this reason, such fractional penalties are rarely used in practice. Among all sparsity-inducing penalties, the $L_1$ norm is the only one that yields a convex objective, making LASSO the most widely adopted method for feature selection.

\paratitle{An Example of LASSO Feature Selection.} To illustrate the feature selection capability of LASSO more concretely, we consider a real-data example described in~\cite{hastie2009elements}\footnote{This example is based on the LASSO case study discussed in \emph{The Elements of Statistical Learning} (Hastie, Tibshirani, and Friedman, 2009).}. The dataset contains pathological measurements from 67 prostate cancer patients undergoing radiation therapy. The variables and their interpretations are listed in Table~\ref{tab:lr:dataset_var}.

\begin{table}[t]\small
    \centering
    \caption{Descriptions of variables in the prostate cancer dataset (Hastie et al., 2009).}\label{tab:lr:dataset_var}
    \renewcommand{\arraystretch}{1.3}
    \begin{tabular}{lll}
    \toprule
         Variable & Full Name & Description \\
    \midrule
lpsa    & $\log$ PSA score                        & Log prostate-specific antigen level \\
lcavol  & $\log$ cancer volume                    & Log cancer volume \\
lweight & $\log$ prostate weight                  & Log prostate weight \\
age     & Patient age                             & Patient age (in years) \\
\multirow{2}*{lbph}
        & Log of benign prostatic                 & \multirow{2}*{Log amount of benign prostatic hyperplasia}\\
        & hyperplasia amount                      & ~\\
svi     & Seminal vesicle invasion                & Indicator of seminal vesicle invasion (0/1) \\
lcp     & $\log$ capsular penetration             & Log capsular penetration \\
gleason & Gleason score                           & Gleason score (tumor grade) \\
pgg45   & Percent of Gleason grade 4 or 5         & Percentage of tissue of Gleason grade 4 or 5 \\
    \bottomrule
    \end{tabular}
\end{table}

\begin{figure}[t]
    \centering
    \includegraphics[width=0.7\columnwidth]{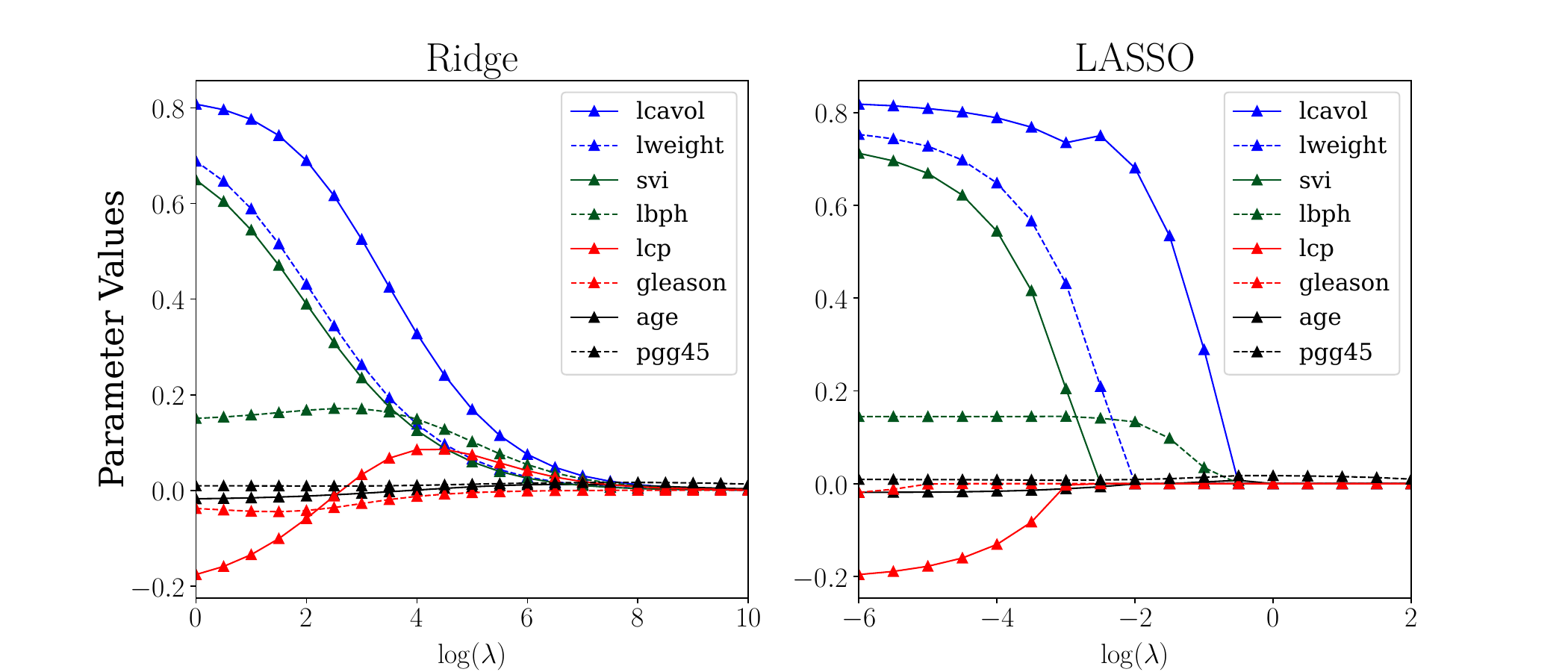}
    \caption{\small Coefficient paths of Ridge and LASSO as the regularization parameter $\lambda$ varies.}
    \label{fig:LASSO_Ridge_lambda}
\end{figure}

In the experiment, \texttt{lpsa} is treated as the response variable, and all other variables serve as predictors. Both Ridge Regression and LASSO Regression are fitted to the data. The estimated coefficient paths as functions of the regularization parameter $\lambda$ are shown in Figure~\ref{fig:LASSO_Ridge_lambda}. For ease of interpretation, the horizontal axis is plotted on the $\log(\lambda)$ scale. As illustrated in the figure, as $\lambda$ increases (\ie as the strength of regularization grows), the coefficients in both models shrink toward zero. However, the two methods exhibit distinctly different shrinkage behaviors. In particular:
\begin{itemize}
    \item \textbf{Ridge Regression:} The coefficients shrink smoothly and simultaneously toward zero. All coefficients decrease at comparable rates, reflecting the fact that Ridge encourages all parameters to be small but rarely drives any coefficient exactly to zero.
    \item \textbf{LASSO Regression:} The coefficients shrink in a \emph{sequential and selective} manner. Some coefficients drop to zero very quickly as $\lambda$ increases, whereas others diminish much more slowly. This corresponds to a feature selection mechanism: predictors with weaker influence on the response are eliminated early, leaving only the most informative variables as $\lambda$ becomes large.
\end{itemize}
This example clearly demonstrates LASSO's ability to perform automatic feature selection, a property not shared by Ridge Regression.

\newpage

\section{Conclusion}

Throughout this document, we have seen that the central purpose of regression analysis is to provide a unified framework for understanding, modeling, and predicting the relationship between input variables $x$ and output variables $y$. As a core component of intelligent computing, regression analysis connects classical statistical thinking with modern machine-learning practice, allowing us both to make predictions and to interpret the mechanisms that may generate the data.

Linear regression, logistic regression, and softmax regression together form a basic toolkit for handling continuous outcomes, binary classification, and multi-class problems. With their associated loss functions -- such as mean squared error and cross-entropy -- these models create a general modeling structure that applies to many types of data. Building on this foundation, Linear Basis Function Models and kernel methods introduce powerful function-approximation techniques, enabling regression models to capture nonlinear relationships that simple linear models cannot express. Even deep neural networks, which drive many of today's most successful AI systems, can be seen as highly flexible regression functions within this same framework. Their strong approximation ability is supported by solid theoretical results, yet this strength also leads to the long-standing challenge of overfitting. To address this, we introduced regularization methods, including Ridge and LASSO, which help control model complexity and, in the case of LASSO, perform effective feature selection -- especially valuable in high-dimensional settings.

In terms of parameter estimation, the closed-form ordinary least squares solution for linear regression offers a clear and intuitive example of how regression models work. However, most nonlinear models do not allow closed-form solutions, so gradient descent becomes the main tool for learning model parameters. Interestingly, all regression models share the same structure for their gradients: \emph{gradient = error $\times$ input}. This simple rule is also the foundation of the backpropagation algorithm that makes it possible to train deep neural networks with millions of parameters. In the part on kernel methods, we also saw how regression analysis connects to another important family of algorithms -- support vector machines (SVM) -- revealing the deep mathematical ties running across different machine-learning approaches.

These lecture notes grew out of my teaching experience in the Intelligent Computing course cluster at the School of Computer Science and Engineering, Beihang University, especially during my teaching of the course \emph{Introduction to Data Mining}. As I continue teaching and learning, I plan to enrich this document further. Future additions include: a Bayesian view of regression analysis, techniques for assessing feature importance, and several pieces of background knowledge -- such as basic supervised-learning concepts, dataset splitting, and model evaluation -- that will help make the notes even more self-contained.

Producing high-quality teaching material is not easy. It requires time, patience, and careful refinement of ideas. The effort is often greater than writing a research paper of similar length, yet the formal academic recognition from the university -- typically centered on papers, grants, and measurable outputs -- is usually much smaller. Even so, creating clear and accessible teaching materials is deeply valuable. A well-written textbook or set of notes does more than support a single course: it helps shape how students understand a field, lowers the barrier to entry for newcomers, and plays a quiet but important role in the long-term development and dissemination of knowledge. In this sense, teaching and textbook writing contribute to the growth of a discipline in ways that may not be immediately reflected in academic evaluation, but whose impact accumulates over time.

I have long been inspired by several excellent textbooks in our field: Christopher M. Bishop's \emph{Pattern Recognition and Machine Learning}, Trevor Hastie, Robert Tibshirani, and Jerome Friedman's \emph{The Elements of Statistical Learning}, Zhou Zhi-Hua's \emph{Machine Learning}, and Li Hang's \emph{Statistical Learning Methods}. Compared with these classics, this tutorial is still modest. But I hope it can offer something of its own: that readers with no background in regression -- equipped only with basic university mathematics -- can complete the entire document without external references and gain a solid and connected understanding of regression analysis. I also hope that readers already familiar with the topic may still find new insights here, perhaps by seeing familiar ideas presented in a clearer and more unified structure. This wish is, in many ways, the reason I decided to write these notes, even though many excellent materials already exist.

Inevitably, there are places where this document can be improved. I sincerely welcome feedback and suggestions from readers.

\begin{flushright}
--- \mywritedate, by the Kunyu River in Beijing

Jingyuan Wang
\end{flushright}
\newpage
\section*{Acknowledgements}

I would like to thank my students, Zimeng Li, Dayan Pan, and Yichuan Zhang (listed in alphabetical order), for their contributions to these notes.

\bibliographystyle{plain}
\bibliography{bibfile4}

\end{document}